\documentclass[journal,12pt,onecolumn,draftclsnofoot]{IEEEtran}
\usepackage{float}
\usepackage{amsmath}
\usepackage{amsfonts}
\usepackage{mathtools}
\usepackage{amssymb}
\usepackage{float}
\usepackage{xcolor}
\usepackage{enumerate}
\usepackage{isomath}
\usepackage{bm}
\usepackage{cite}
\usepackage{caption}
\usepackage{subfig}
\usepackage{graphicx}
\usepackage{algorithm,algcompatible,amsmath}
\algnewcommand\INPUT{\item[\textbf{Input:}]}%
\algnewcommand\OUTPUT{\item[\textbf{Output:}]}%
\restylefloat{table}
\usepackage{tikz} 
\usepackage{lipsum}
\newcommand\blfootnote[1]{%
  \begingroup
  \renewcommand\thefootnote{}\footnote{#1}%
  \addtocounter{footnote}{-1}%
  \endgroup
}
\usetikzlibrary{arrows, positioning, automata}
\usetikzlibrary{shapes.geometric, arrows}
\def\xfoo#1^#2\relax#3\valign{%
\mathbf{#1}\ifx\valign#2\valign\else^{\mathbf{#2}}\fi}

\def\Ber{{\rm Bern}}
\def \Beta{{\rm Beta}}

\def\dP{P}
\def\dQ{Q}

\newcommand{\etali}{\emph{et al.}\@\xspace}

\def\mset{{Z^M_{1:N}}}

\def\Escr{\mathscr{E}}

\def \Wscr{\mathscr{W}}

\usepackage{amsmath, amssymb, bbm, xspace}
\usepackage{epsfig}
\usepackage{longtable}
\usepackage{color}
\usepackage{mathrsfs}
\usepackage{comment}

\usepackage{courier}

\newtheorem{theorem}{Theorem}[section]
\newtheorem{assumption}{Assumption}[section]

\newtheorem{lemma}{Lemma}[section]

\newtheorem{corollary}[theorem]{Corollary}

\newtheorem{definition}{Definition}[section]

\def\bkE{{\rm I\kern-.17em E}}
\def\bk1{{\rm 1\kern-.17em l}}
\def\bkD{{\rm I\kern-.17em D}}
\def\bkR{{\rm I\kern-.17em R}}
\def\bkP{{\rm I\kern-.17em P}}

\def\bkZ{{\bf{Z}}}

\def\bkE{{\rm I\kern-.17em E}}
\def\bk1{{\rm 1\kern-.17em l}}
\def\bkD{{\rm I\kern-.17em D}}
\def\bkR{{\rm I\kern-.17em R}}
\def\bkP{{\rm I\kern-.17em P}}

\makeatletter
\newcommand{\pushright}[1]{\ifmeasuring@#1\else\omit\hfill$\displaystyle#1$\fi\ignorespaces}
\newcommand{\pushleft}[1]{\ifmeasuring@#1\else\omit$\displaystyle#1$\hfill\fi\ignorespaces}
\makeatother

\def\bkZ{{\bf{Z}}}
\def\b12{(\beta_1,\beta_2)}

\newcounter{example}
\renewcommand{\theexample}{\thesection.\arabic{example}}

\newcounter{remark}
\renewcommand{\theremark}{\thesection.\arabic{remark}}
\newenvironment{remarkc}[1][]{\refstepcounter{remark}
\noindent{\itshape Remark~\theremark. #1} \rmfamily}{\hspace*{\fill}~$\square$\vspace{0pt}}

\def\Hscr{\mathscr{H}}

\def\Xscr{\mathcal{X}}

\def\Ebb{\mathbb{E}}
\newlength{\noteWidth}
\long\def\notes#1{\ifinner
{\tiny #1}
\else
\marginpar{\parbox[t]{\noteWidth}{\raggedright\tiny #1}}
\fi\typeout{#1}}

 \def\notes#1{\typeout{read notes: #1}} 

\newcommand{\ie}{i.e.\@\xspace} 
\newcommand{\eg}{e.g.\@\xspace} 

\newcommand{\Real}{\ensuremath{\mathbb{R}}}

\def\Ebb{\mathbb{E}}

\def\Pbb{{\mathbb{P}}}

\def\Ibb{{\mathbb{I}}}

\def\exp{\mathop{\hbox{\rm exp}}}

\def\spose#1{\hbox to 0pt{#1\hss}}

\def\text #1{\hbox{\quad#1\quad}}

\def\Escr{\mathcal{E}}

\def\nthinsp{\mskip -2   mu}

\def\superstar{^{\raise 0.5pt\hbox{$\nthinsp *$}}}
\def\SUPERSTAR{^{\raise 0.5pt\hbox{$*$}}}

\def\lamstarT {\lambda^{\raise 0.5pt\hbox{$\nthinsp *$}T}}

\def\Ascr{{\cal A}}

\def\Lscr{{\cal L}}

\def\Pscr{{\cal P}}

\def\Uscr{{\cal U}}

\def\Wscr{{\cal W}}

\def\Zscr{{\cal Z}}
\def\Xscr{{\cal X}}

\def\supp{{\rm supp}}

\def\non{\nonumber}

\let\forallnew\forall
\renewcommand{\forall}{\forallnew\ }
\let\forall\forallnew

		\def\bkE{{\rm I\kern-.17em E}}
		\def\bk1{{\rm 1\kern-.17em l}}
		\def\bkD{{\rm I\kern-.17em D}}
		\def\bkR{{\rm I\kern-.17em R}}
		\def\bkP{{\rm I\kern-.17em P}}
		\def\bkY{{\bf \kern-.17em Y}}
		\def\bkZ{{\bf \kern-.17em Z}}
		\def\bkC{{\bf  \kern-.17em C}}

{\begin{list}{}%
         {\setlength{\leftmargin}{#1}}%
         \item[]%
}
{\end{list}}

		\def\bsp{\begin{split}}
		\def\beq{\begin{eqnarray}}
		\def\bal{\begin{align*}}
		\def\bc{\begin{center}}
		\def\be{\begin{enumerate}}
		\def\bi{\begin{itemize}}
		\def\bs{\begin{small}}
		\def\bS{\begin{slide}}
		\def\ec{\end{center}}
		\def\ee{\end{enumerate}}
		\def\ei{\end{itemize}}
		\def\es{\end{small}}
		\def\eS{\end{slide}}
		\def\eeq{\end{eqnarray}}
		\def\eal{\end{align*}}
		\def\esp{\end{split}}
		\def\qed{ \vrule height7.5pt width7.5pt depth0pt}  

	\def\cp2problem#1#2#3#4{\fbox
		 {\begin{tabular*}{0.9\textwidth}
			{@{}l@{\extracolsep{\fill}}l@{\extracolsep{6pt}}l@{\extracolsep{\fill}}c@{}}
				#1 & & $#4 $ 
			\end{tabular*}}}

		\def\bkE{{\rm I\kern-.17em E}}
		\def\bk1{{\rm 1\kern-.17em l}}
		\def\bkD{{\rm I\kern-.17em D}}
		\def\bkR{{\rm I\kern-.17em R}}
		\def\bkP{{\rm I\kern-.17em P}}
		
		\def\bkZ{{\bf{Z}}}

\newcommand {\beeq}[1]{\begin{equation}\label{#1}}
\newcommand {\eeeq}{\end{equation}}
\newcommand {\bea}{\begin{eqnarray}}
\newcommand {\eea}{\end{eqnarray}}

\def\texitem#1{\par\smallskip\noindent\hangindent 25pt
               \hbox to 25pt {\hss #1 ~}\ignorespaces}

\def\bsp{\begin{split}}
		\def\beq{\begin{eqnarray}}
		\def\bal{\begin{align*}}
		\def\bc{\begin{center}}
		\def\be{\begin{enumerate}}
		\def\bi{\begin{itemize}}
		\def\bs{\begin{small}}
		\def\bS{\begin{slide}}
		\def\ec{\end{center}}
		\def\ee{\end{enumerate}}
		\def\ei{\end{itemize}}
		\def\es{\end{small}}
		\def\eS{\end{slide}}
		\def\eeq{\end{eqnarray}}
		\def\eal{\end{align*}}
		\def\esp{\end{split}}
		\def\qed{ \vrule height7.5pt width7.5pt depth0pt}  

\usepackage{amsmath, amssymb, xspace}
\usepackage{epsfig}
\usepackage{longtable}
\usepackage{color}
\usepackage{mathrsfs}
\usepackage{subfig}
\newenvironment{proof}[1][]{{\noindent \textit{ Proof}: }}{\hfill \qed \vspace{3pt}\\ }

\ifCLASSINFOpdf
  \else
 
\fi

\author{\IEEEauthorblockN{Sharu Theresa Jose, Osvaldo Simeone, Giuseppe Durisi}}
\title{Transfer Meta-Learning: Information-Theoretic Bounds and Information Meta-Risk Minimization}
\begin{document}
\maketitle
\begin{abstract}
\textit{Meta-learning} automatically infers an \textit{inductive bias} by observing data from a number of related tasks. The inductive bias is encoded by hyperparameters that determine aspects of the model class or training algorithm, such as initialization or learning rate. Meta-learning assumes that the learning tasks belong to a \textit{task environment}, and that tasks are drawn from the same task environment both during meta-training and meta-testing. This, however, may not hold true in practice. In this paper, we introduce the problem of \textit{transfer meta-learning}, in which tasks are drawn from a \textit{target task environment} during meta-testing that may differ from the \textit{source task environment} observed during meta-training. Novel information-theoretic upper bounds are obtained on the \textit{transfer meta-generalization gap}, which measures the difference between the meta-training loss, available at the meta-learner, and the average loss on meta-test data from a new, randomly selected, task in the target task environment.
  The first bound, on the average transfer meta-generalization gap, captures the \textit{meta-environment shift} between source and target task environments via the KL divergence between source and target data distributions. The second, PAC-Bayesian bound, and the third, single-draw bound, account for this shift via the log-likelihood ratio between source and target task distributions. Furthermore, two transfer meta-learning solutions are introduced. For the first, termed  \textit{Empirical Meta-Risk Minimization} (EMRM), we derive bounds on the average optimality gap. The second, referred to as \textit{Information Meta-Risk Minimization} (IMRM), is obtained by minimizing the PAC-Bayesian bound. IMRM is shown via experiments to potentially outperform EMRM.
\end{abstract}
\blfootnote{S. T. Jose and O. Simeone are with King's Communications, Learning, and Information Processing (KCLIP) lab at the Department of Engineering of King’s College London, UK (emails: sharu.jose@kcl.ac.uk, osvaldo.simeone@kcl.ac.uk). 
They have received funding from the European Research Council
(ERC) under the European Union’s Horizon 2020 Research and Innovation
Programme (Grant Agreement No. 725731). G. Durisi is with the Department of Electrical Engineering, Chalmers Institute of Technology, Sweden (email: durisi@chalmers.se).}
\section{Introduction}
Any machine learning algorithm makes assumptions on the task of interest, which are collectively referred to as the inductive bias. In parametric machine learning, the inductive bias is encoded in the choice of a model class and of a training algorithm used to identify a model parameter vector based on training data. The inductive bias is fixed a priori, ideally with the help of domain expertise, and it can be refined via validation. As a typical example, an inductive bias may consist of a class of neural networks parameterized by synaptic weights and of an optimization procedure such as stochastic gradient descent (SGD). Hyperparameters including number of layers and SGD learning rate schedule can be selected by optimizing the validation error on an held-out data set.

 Meta-learning or \textit{learning to learn} aims to automatically infer some aspects of the inductive bias based on the observation of data from related tasks \cite{schmidhuber1987evolutionary,thrun1996learning,thrun1998learning}. For example, the choice of an inductive bias---model class and training algorithm---for the problem of classifying images of animals may be based on labelled images of vehicles or faces.  As formalized in \cite{baxter2000model}, meta-learning assumes the presence of a 
\textit{task environment} consisting of related learning tasks. A task environment is defined by a distribution on the set of tasks and by per-task data distributions. A meta-learner observes data sets from a finite number of tasks drawn from the task environment to infer the inductive bias, while its performance is evaluated on  a new, previously unseen, task drawn from the same task environment. 

As discussed, a key assumption in the standard formulation of meta-learning is that the tasks encountered during meta-learning are from the same task environment that generates the new ``meta-test" task on which the performance of the hyperparameter is evaluated. This assumption may not be realistic in some applications \cite{collins2020taskrobust}. For example, a personalized health application may be meta-trained by using data from a population of users that is not fully representative of the distribution of the health profiles expected in a different population on which the application is deployed and meta-tested. 
 In this paper, we introduce the problem of \textit{transfer meta-learning}, wherein the performance of a meta-learner that uses data sets drawn from a \textit{source task environment} is tested on a new task drawn from a generally different \textit{target task environment}. In the proposed formulation, highly  popular, or more frequently observed, tasks during meta-training may have a small probability in the target task environment, while other tasks may have a higher chance of being encountered. 
 
 \begin{figure}[h!]
\centering
\includegraphics[scale=0.5,clip=true, trim = 0in  0in 0in 0in]{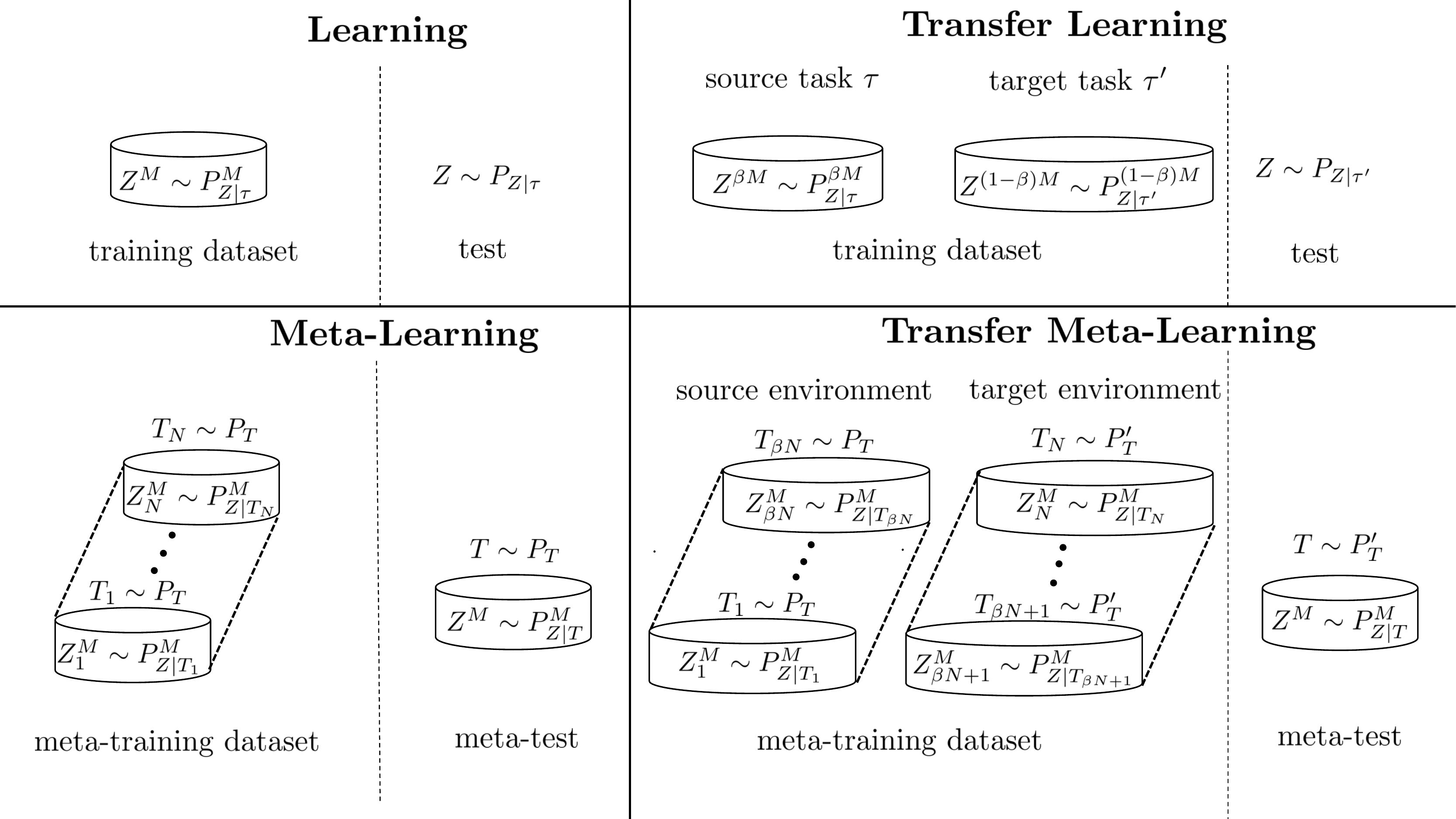} 
\caption{ Illustration of conventional learning, transfer learning, conventional meta-learning and transfer meta-learning with $P_{Z|\tau}$ denoting the distribution $P_{Z|T=\tau}$.}\label{fig:BN_transfer_metalearning}
\vspace{-0.5cm}
\end{figure}
As illustrated in Figure~\ref{fig:BN_transfer_metalearning}, we consider a general formulation of transfer meta-learning where the meta-learner observes a meta-training set of $N$ datasets $Z^M_1,\hdots,Z^M_N$, each of $M$ samples, of which $\beta N$, with $\beta \in (0,1]$, datasets correspond to tasks drawn from the \textit{source task environment} and $(1-\beta)N$ datasets correspond to tasks from the \textit{target task environment}. Under source and target task environments, tasks are drawn according to distinct distributions $P_T$ and $P'_T$, respectively. Based on the meta-training set $\mset=(Z^M_1,\hdots,Z^M_N)$, the meta-learner infers the vector of hyper-parameters $u \in \Uscr$. The hyperparameters $u$ determine the base learning algorithm through a conditional distribution $P_{W|Z^M,U=u}$, that maps a training set $Z^M$ to a model parameter $W$ given $u$. The performance of the inferred hyperparameter $u$ is evaluated in terms of the \textit{transfer meta-generalization loss} $\Lscr'_g(u)$, which is the expected loss over a data set $Z^M \sim P^M_{Z|T}$ sampled from a task $T$ randomly selected from the {target task} distribution $P'_T$. The subscript $g$ of $\Lscr_g'(u)$ indicates that the considered loss is the generalization loss and the superscript $'$ indicates that the generalization loss is evaluated with respect to the target task distribution $P'_T$. This objective function is not available at the meta-learner since the target task distribution $P'_T$ and the per-task distributions $P_{Z|T=\tau}$ for every task $\tau$ are unknown. Instead, the meta-learner can evaluate the empirical performance of the inferred hyperparameter on the meta-training set $\mset$ in terms of the \textit{meta-training loss} $\Lscr_t(u|\mset)$. The subscript $t$ of $\Lscr_t(u|\mset)$ indicates that the loss considered is training loss.

The difference between the transfer meta-generalization loss and the meta-training loss, referred to as the \textit{transfer meta-generalization gap} $\Delta \Lscr'(u|\mset)$, is a key metric to evaluate the generalization performance of the meta-learner. If the transfer meta-generalization gap is small, on average or with high probability, the meta-learner can take the performance on the meta-training set as a reliable measure of accuracy of the inferred hyperparameter in terms of the transfer meta-generalization loss.
In this paper, we first study information-theoretic upper bounds on the transfer meta-generalization gap of three different flavours -- bounds on the average transfer meta-generalization gap, high-probability probably-approximately-correct (PAC)-Bayesian bounds, and  high-probability single-draw bounds-- and, we introduce two transfer meta-learning algorithms based on \textit{Empirical Meta-Risk Minimization} (EMRM) and \textit{Information Meta-Risk Minimization} (IMRM).

The transfer meta-learning setting considered in this paper generalizes conventional transfer learning \cite{ben2007analysis,bickel2007discriminative,blitzer2006domain}, as well as meta-learning (see Figure~\ref{fig:BN_transfer_metalearning}). Specifically, when the source and target task distributions are delta functions centered at source domain task $\tau$ and target domain task $\tau'$ respectively, with $\tau \not =\tau'$, and the hyperparameter $u$ to be inferred coincides with the model parameter, the transfer meta-learning setting reduces to  transfer learning.
 While there exists a rich literature on generalization bounds and algorithms for transfer learning, this work is, to the best of our knowledge, the first one to extend the notion of transfer to meta-learning, to derive information-theoretic upper bounds on the transfer meta-generalization gap, and to propose transfer meta-learning design criteria.
\subsection{Related Work}
Three distinct kinds of bounds on generalization gap, \ie, the difference between training and generalization losses,  have been studied in literature for conventional learning---bounds on average generalization gap, high-probability PAC-Bayesian bounds and high-probability single-draw bounds \cite{hellstrom2020generalization}. For learning algorithms described as a stochastic mapping from the input training set to the model parameter, the average generalization gap evaluates the average difference between the training and generalization losses over the learning algorithm and its input training set. Information-theoretic upper bounds on the average generalization gap have been studied first by  Russo \emph{et al.} \cite{russo2016controlling} and Xu \etali \cite{xu2017information}, and variants of the bounds have been investigated in \cite{bu2019tightening,negrea2019information,steinke2020reasoning}. 
Of particular relevance to our work is the individual sample mutual information (ISMI) based bound \cite{bu2019tightening}, which captures the sensitivity of the learning algorithm to the input training set, and thus the generalization ability, via the mutual information (MI) between the model parameter output of the algorithm and individual data sample of the input training set. These bounds have the distinction that they depend explicitly on the data distribution, the learning algorithm, and  the loss function. Moreover, for deterministic algorithms, the ISMI approach yields a finite upper bound as compared to the MI bounds in \cite{xu2017information}.  The ISMI bound has been extended to obtain bounds on generalization gap for transfer learning in \cite{wu2020information} and for meta-learning in \cite{jose2020information}, where, in the latter, the MI between the hyperparameter and per-task data of the meta-training set captures the sensitivity of the meta-learner to the meta-training data set. The results in this paper can be seen as a natural extension of these lines of work to transfer meta-learning.  

Apart from bounds on average generalization gap, PAC bound on the generalization gap which holds with high probability over the training set have been studied in the literature. 
 Classical PAC bounds for conventional learning assume deterministic learners and employ measures of complexity of the model class like Vapnik-Chervonenkis (VC) dimension \cite{vapnik2015uniform} or Radmacher complexity\cite{koltchinskii2000rademacher} to characterize the generalization gap.  For stochastic learning algorithms, McAllester \cite{mcallester1999pac} developed a PAC-Bayesian upper bound on the average of the generalization gap over the learning algorithm, which holds with probability at least $1-\delta$, with $\delta \in (0,1)$, over the input training set. These bounds employ a reference data-independent `prior' distribution on the model parameter space, and the sensitivity of the learning algorithm to the training set is captured by the Kullback-Leibler (KL) divergence between   the posterior distribution of the learning algorithm and the prior. As such, the PAC-Bayesian bounds are independent of data distributions. We note that the recent line of work in \cite{dziugaite2020role} suggests tightening the PAC-Bayesian bounds by choosing a data-dependent prior distribution evaluated on an heldout data set, which is not part of the training data.
 
   Various refinements of PAC-Bayesian bounds have been studied for conventional learning \cite{seeger2002pac,mcallester2003pac,maurer2004note,alquier2016properties}, and for meta-learning \cite{pentina2014pac,amit2018meta,rothfuss2020pacoh} where for the latter PAC-Bayesian bounds employ a hyper-prior distribution on the space of hyperparameters in addition to the prior. A PAC-Bayesian approach to domain adaptation specialized to linear classifiers has been considered in \cite{germain2017pac}. Furthermore, PAC-Bayesian bounds can be employed to design learning algorithms that ensure generalization through the principle of \textit{Information Risk Minimization} (IRM)\cite{zhang2006information}. For conventional learning, the IRM principle finds a randomized learning algorithm that minimizes the PAC-Bayesian upper bound on the generalization loss, which is given by the empirical training loss regularized by the KL divergence between the posterior learning algorithm and the prior. In Section~\ref{sec:PAC-Bayesian bound}, we resort to the IRM principle and propose a novel learning algorithm  for transfer meta-learning.

 PAC-Bayesian bounds apply to the scenario when a model parameter is drawn every time the learning algorithm is used, and the performance of the learner is evaluated with respect to the average of the generalization gap over these draws. In contrast, high-probability \textit{single-draw bounds} are relevant in scenarios when a model parameter is drawn only once from the stochastic learner, and the goal is to evaluate the generalization performance with respect to this  parameter. Precisely, single-draw probability bounds yield upper bounds on the generalization gap which holds with probability at least $1-\delta,$ with $\delta \in (0,1)$, over the training set and the model parameter. For conventional learning, MI-based single-draw bounds have been obtained in \cite{raginsky2016information,bassily2016algorithmic}, while information-theoretic quantities like R{\'e}nyi divergence, $\alpha$-mutual information, and information leakage have been used in \cite{esposito2019generalization}. To the best of our knowledge, single-draw bounds have not been studied in the context of meta-learning or transfer meta-learning before.
 
In comparison to the generalization bounds for conventional learning, the generalization bounds for transfer learning have to account for the \textit{domain shift} between source domain and target domain. For conventional transfer learning, upper bound on the generalization loss on target domain is obtained in terms of generalization loss on the source domain, together with a divergence measure that captures the domain shift \cite{ben2007analysis, blitzer2008learning,mansour2009domain}. Various distance and divergence measures have been explored in the literature to quantify the domain shift. These measures have the advantage that they can be empirically estimated from finite data sets from source and target domains. For example,
 \cite{ben2007analysis} studies transfer learning for classification tasks and obtains high-probability upper bounds on the target domain generalization loss based on the $\Hscr$-divergence, or $d_{\Ascr}$-distance, in terms of VC dimensions or Radmacher complexity \cite{blitzer2008learning}. The $d_{\Ascr}$ distance has been generalized to the discrepancy distance so as to account for loss functions beyond the detection loss in \cite{mansour2009domain}, and to integral probability metric in \cite{zhang2012generalization}. 
Estimates of these distance measures yield generalization bounds in terms of Radmacher complexity. The $\Hscr$-divergence has been further extended to define the $\Hscr \Delta \Hscr$ divergence in \cite{ben2010theory}.
 While these distance measures are tailored to given loss functions and model class, general statistical divergence measures, such as R{\'e}nyi divergence  and Wasserstein distance  have been considered in \cite{mansour2012multiple,germain2013pac,hoffman2018algorithms} and \cite{redko2017theoretical} respectively. The information-theoretic generalization bound in \cite{wu2020information} captures the domain shift in terms of the KL divergence between source and target domain.
 Our work draws inspiration from this line of research.
 \subsection{Main Contributions}
 Building on the lines of work on transfer learning outlined above, we introduce the problem of transfer meta-learning, in which data from both source and target task environments are available for meta-training. Extending the methods in \cite{ben2010theory,blitzer2008learning,zhang2012generalization} for transfer learning, we measure the 
 meta-training loss as a weighted average of the training losses on source and target task environment data sets. This weighted average includes as special cases methods that use only data from source or target task environments. We refer to the resulting design criterion as EMRM.
We derive information-theoretic upper bounds on the average transfer meta-generalization gap, \ie, on the average difference between transfer meta-generalization loss $\Lscr'_g(u)$ and meta-training loss $\Lscr_t(u|\mset)$. The bounds generalize prior works on
transfer learning \cite{wu2020information} and meta-learning \cite{jose2020information}. We also present novel PAC-Bayesian and single-draw probability bounds. Central to the derivation of these generalization bounds is the information-density based exponential inequality approach of \cite{hellstrom2020generalization}.
We detail the main contributions as follows.
\begin{enumerate}
\item We extend the individual task mutual information (ITMI) based approach of\cite{jose2020information} for meta-learning to obtain novel upper bounds on the average transfer meta-generalization gap that holds for any meta-learner. The resulting bound captures the \textit{meta-environment shift} from source to target task distributions via the KL divergence between source environment data distribution and  target environment data distribution.
\item  We specialize the obtained generalization bound on the average transfer meta-generalization gap to study the performance of the EMRM algorithm that minimizes the empirical average meta-training loss, and obtain a novel upper bound on the average transfer excess meta-risk for EMRM. The average transfer excess meta-risk is the optimality gap between the average transfer meta-generalization loss of EMRM and the optimal transfer meta-generalization loss.
\item We derive novel PAC-Bayesian bounds for transfer meta-learning that quantify the impact of the meta-environment shift through the log-likelihood ratio of the source and target task distributions. We use these bounds to introduce a novel meta-training algorithm, termed IMRM, based on the principle of information risk minimization \cite{zhang2006information}.
\item We obtain new single-draw probability bounds for transfer meta-learning in terms of information densities and a log-likelihood ratio between source and target task distribution. Single-draw bounds captures the performance under a single realization of the hyperparameter drawn by a stochastic meta-learner. Furthermore, the resulting bounds can be specialized to obtain novel single-draw bounds for conventional meta-learning.
\item Finally, we compare the performance of EMRM and IMRM algorithms on a transfer meta-learning example, and show that IMRM can outperform EMRM in terms of transfer meta-generalization loss for sufficiently small number of tasks and per-task data samples. As the number of tasks and per-task data samples grow, IMRM reduces to EMRM.
\end{enumerate}
\subsection{Notation}
Throughout this paper, we use upper case letters, \eg $X$, to denote random variables and lower case letters, \eg $x$ to represent their realizations. We use $\Pscr(\cdot)$ to denote the set of all probability distributions on the argument set or vector space. For a discrete or continuous random variable $X$ taking
values in a set or vector space $\Xscr$, $P_X \in \Pscr(\Xscr)$ denotes its probability distribution, with $P_X(x)$
being the probability mass or density value at $X=x$. We denote as $P^N_X$ the $N$-fold product
distribution induced by $P_X$. The conditional distribution of a random variable $X$ given random
variable $Y$ is similarly defined as $P_{X|Y}$, with $P_{X|Y} (x|y)$ representing the probability mass or
density at $X = x$ conditioned on the event $Y = y$.
We define the Kronecker delta $\delta(x-x_0)=1$ if $x=x_0$ and $\delta(x-x_0)=0$ otherwise, and use $\Ibb_{E}$ to denote the indicator function which equals one when the event $E$ is true and equals zero otherwise.
\section{Problem Formulation}
\subsection{Conventional Transfer Learning}\label{sec:transferlearning}
We review first the conventional transfer learning problem \cite{ben2010theory,blitzer2008learning,zhang2012generalization} in order to define important notation and provide the necessary background for the introduction of transfer meta-learning. We refer to Figure~\ref{fig:BN_transfer_metalearning} for an  illustration comparing conventional learning and transfer learning.  In transfer learning, we are given a data set that consists of: (i) data points from a \textit{source task} $\tau$
drawn from an underlying \textit{unknown} data distribution, $P_{Z|T=\tau} \in \Pscr(\Zscr)$, defined in a subset or vector space $\Zscr$; as well as (ii) data from a target task $\tau'$, drawn from a generally different distribution $P_{Z|T=\tau'} \in \Pscr(\Zscr)$. The goal is to infer a machine learning model that generalizes well on the data from the \textit{target task} $\tau'$. For notational convenience, in the following, we use $P_{Z|\tau}$ to denote source data distribution $P_{Z|T=\tau}$, and $P_{Z|\tau'}$ to denote the target data distribution $P_{Z|T=\tau'}$.

 The learner has  access to a training data set $Z^M=(Z_1,Z_2, \hdots,Z_M)$, which consists of $\beta M$, for some fixed $\beta \in (0,1]$,  independent and identically distributed (i.i.d.) samples $(Z_1,\hdots, Z_{\beta M}) \sim P^{\beta M}_{Z|\tau}$ drawn from the source data distribution $P_{Z|\tau}$, and $(1-\beta)M$ i.i.d. samples $(Z_{\beta M+1},\hdots Z_M)\sim P^{(1-\beta)M}_{Z|\tau'}$ from the target data distribution $P_{Z|\tau'}$. Since the learner instead does not know the distributions $P_{Z|\tau}$ and $P_{Z|\tau'}$,
it uses the data set $Z^M$ to choose a model, or hypothesis, $W$ from the model class $\Wscr$ by using a \textit{randomized} training procedure defined by a conditional distribution $P_{W|Z^M} \in \Pscr(\Wscr)$ as $W \sim P_{W|Z^M}$.
 The conditional distribution $P_{W|Z^M}$ defines a stochastic mapping from the training data set $Z^M$ to the model class $\Wscr$. 

 The performance of a model parameter vector $w \in \Wscr$ on a data sample $z \in \Zscr$ is measured by a loss function $l(w,z)$ where $l:\Wscr \times \Zscr \rightarrow \Real_{+}$. 
The \textit{generalization loss}, or \textit{population loss}, for a model parameter vector $w \in \Wscr$ is evaluated on the target task $\tau'$, and is defined as
\begin{align}
&L_{g}(w|\tau')=\Ebb_{P_{Z|\tau'}}[l(w,Z)], \label{eq:genloss}
\end{align} where the average is taken over a test example $Z$ drawn independently  of $Z^M$ from the target task data distribution $P_{Z|\tau'}$.

 The generalization loss cannot be computed by the learner, given that the data distribution $P_{Z|\tau'}$ is unknown. A typical solution is for the learner to minimize instead the \textit{weighted average training loss} on the data set $Z^M$, which is defined as the empirical average
\begin{align}
 L_{t}(w|Z^M)=\frac{\alpha}{\beta M}\sum_{i=1}^{\beta M}l(w,Z_i)+\frac{1-\alpha}{(1-\beta) M}\sum_{i=\beta M+1}^{M}l(w,Z_i), \label{eq:trainingloss}
\end{align}
where $\alpha\in [0,1]$ is a hyperparameter \cite{ben2010theory}. Note that this formulation assumes that the learner knows which training data comes from the source task and which are from the target task. We distinguish the generalization loss and the training loss via the subscripts $g$ and $t$ of $L_g(w|\tau')$ and $L_t(w|Z^M)$ respectively.
The difference between generalization loss \eqref{eq:genloss} and training loss \eqref{eq:trainingloss}, known as \textit{transfer generalization gap}, is a key metric that relates to the performance of the learner. This is because a small transfer generalization gap ensures that the training loss \eqref{eq:trainingloss} is a reliable estimate of the generalization loss \eqref{eq:genloss}. An information theoretic study of the transfer generalization gap and of the excess risk gap of a learner that minimizes \eqref{eq:trainingloss} was presented in \cite{wu2020information}.
\subsection{ Meta-Learning}\label{sec:metalearning}
We now review the meta-learning setting \cite{finn2017model}. To start, let us fix a class of \textit{within-task base learners} $P_{W|Z^M,U=u}$ mapping a data set $Z^M$ to a model parameter vector $W$, where each base learner is identified by a hyperparameter $u \in \Uscr$.  Meta-learning aims to automatically infer the hyperparameter $u$ using data from related tasks, thereby ``learning to learn''. Towards this goal, a \textit{meta-learner} observes data from tasks drawn from a \textit{task environment}. A task environment is defined by a \textit{ task distribution} $P_T$ supported over the set of tasks $\mathcal{T}$, as well as by a  per-task data distribution $P_{Z|T=\tau}$ for each $\tau \in \mathcal{T}$. Using the meta-training data drawn from a randomly selected subset of tasks, the meta-learner infers the hyperparameter $u \in \Uscr$ with the goal of ensuring that the base learner $P_{W|Z^M,u}$ generalize well on a new, previously unobserved \textit{meta-test task} $T \sim P_T$ drawn independently from the same task environment.


To elaborate, as seen in Figure~\ref{fig:BN_transfer_metalearning}, the meta-training data set consists of $N$ data sets $\mset=(Z^M_1,\hdots,Z^M_N)$, where each $i$th sub-data set $Z^M_i$ is generated independently by first drawing a task $T_i \sim P_T$ and then generating a task specific data set $Z^M_i \sim P^M_{Z|T=T_i}$. The meta-learner does not know the distributions $P_T$ and $\{P_{Z|T=\tau}\}_{\tau \in \mathcal{T}}$. We consider a \textit{randomized} meta-learner \cite{jose2020information} \begin{align}
U \sim P_{U|\mset}, \label{eq:meta_learner}
\end{align}
where $P_{U|\mset}$ is a stochastic mapping from the meta-training set $\mset$ to the space $\Uscr$ of hyperparameters.
As discussed, for a given hyperparameter $U=u$ and given a data set $Z^M$, the \textit{within-task base learner} $P_{W|Z^M,u} \in \Pscr(\Wscr)$ maps the per-task training subset $Z^M$ to random model parameter $W \sim P_{W|Z^M,u}$. The average per-task test loss for a given task $T$ is obtained as
\begin{align}
L_g(u|T,Z^M)=\Ebb_{P_{W|Z^M,u}}[L_g(W|T)],\label{eq:avgtestloss}
\end{align} 
where the per-task generalization loss $L_g(w|T)$ is defined in \eqref{eq:genloss}.
The goal of meta-learning is to minimize the \textit{meta-generalization loss} defined as
\begin{align}
\Lscr_g(u)=\Ebb_{P_{T}P^M_{Z|T}}[ L_g(u|T,Z^M)].\label{eq:meta_testloss}
\end{align} 
 The meta-generalization loss is averaged over new, meta-test tasks $T \sim P_T$ drawn from the task environment $P_T$ and on the corresponding 
 training data $Z^M$ drawn i.i.d from the data distribution $P^M_{Z|T}$.

The meta-generalization loss cannot be computed by the meta-learner, given that the task distribution $P_T$ and per-task data distribution $P_{Z|T}$ are unknown. The meta-learner relies instead on the \textit{empirical meta-training loss} 
\begin{align}
\Lscr_t(u|\mset)=\frac{1}{N} \sum_{i=1}^N L_t(u|Z^M_i), \label{eq:metatrainingloss}
\end{align}
where $L_t(u|Z^M_i)$ is the average per-task training loss,
\begin{align}
L_t(u|Z^M_i)=\Ebb_{P_{W|Z^M_i,u}}[L_t(W|Z^M_i)],
\end{align} with $L_t(w|Z^M)$ defined  in \eqref{eq:trainingloss} (with $\alpha=\beta=1$).
The difference between the meta-generalization loss \eqref{eq:meta_testloss} and meta-training loss \eqref{eq:metatrainingloss} is known as the \textit{meta-generalization gap}, and is a measure of performance of the meta-learner. 
\subsection{Transfer Meta-Learning}
\begin{figure}[h!]
\centering
\includegraphics[scale=0.35,clip=true, trim = 0in  0.3in 0in 0.2in]{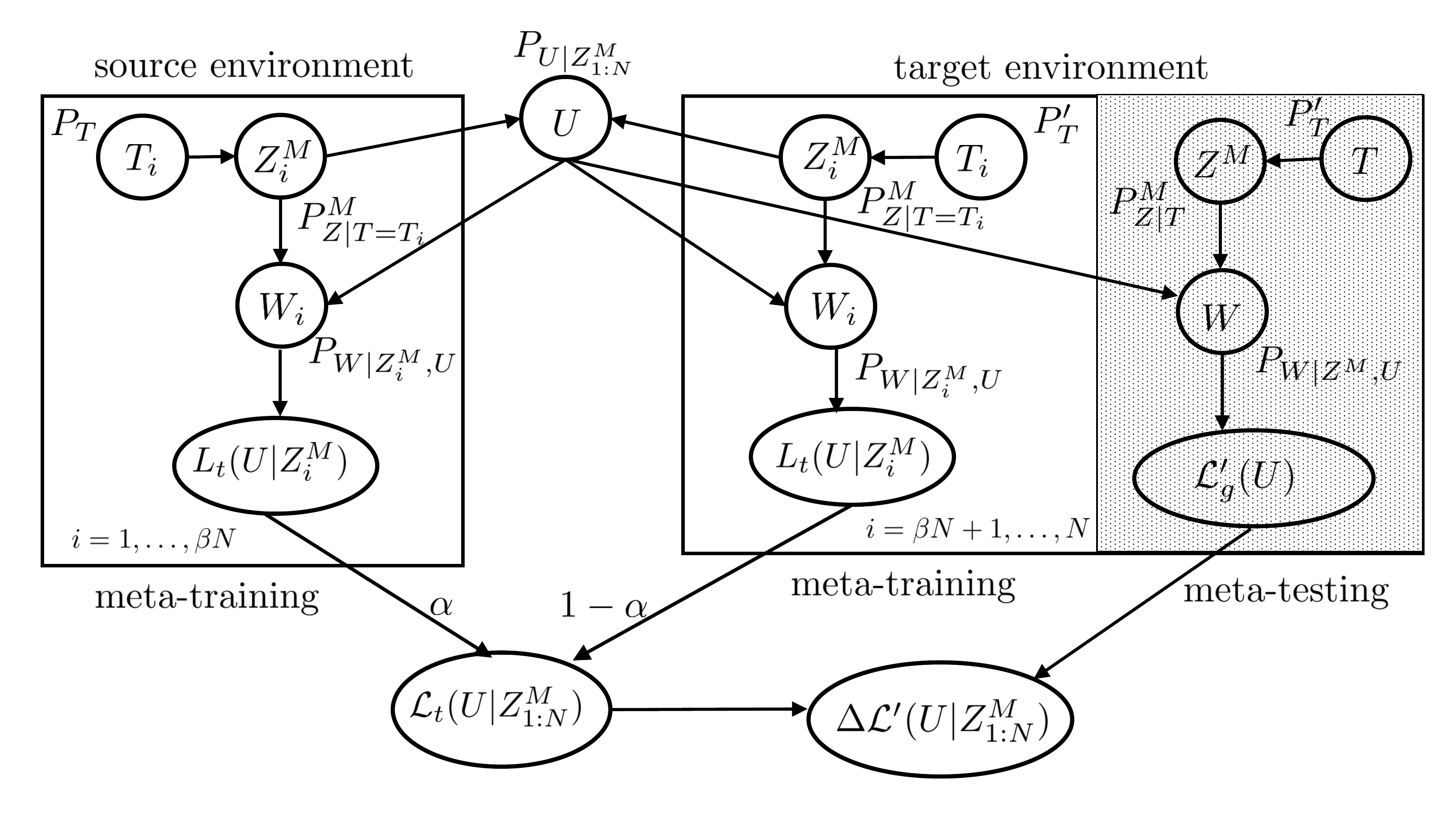} 
\caption{A Bayesian network representation of the variables involved in the definition of transfer meta-learning. }\label{fig:BN_transfer_metalearning_true}
\vspace{-0.5cm}
\end{figure}
In this section, we introduce the problem of transfer meta-learning. As we will explain, it generalizes both transfer learning and meta-learning. In transfer meta-learning, as seen in Figure~\ref{fig:BN_transfer_metalearning_true}, a meta-learner observes meta-training data from two different environments: $(i)$ a \textit{source task environment} which is defined by a \textit{source task distribution} $P_T \in \Pscr(\mathcal{T})$ and a per-task data distribution $P_{Z|T}$; and $(ii)$ a \textit{target task environment} which is defined by \textit{target task distribution} $P'_T \in \Pscr(\mathcal{T})$ and per-task data distribution $P_{Z|T}$. For a given family of per-task base learner $P_{W|Z^M,u}$, 
the goal of transfer meta-learning is to infer a hyperparameter $u \in \Uscr$ from the meta-training data such that the base learner $P_{W|Z^M,u}$ generalize well to a new task $T \sim P'_T$ drawn independently from the target task distribution $P'_T$.

The source and target task distributions $P_T$ and $P'_T$ model the likelihood of observing a given set of tasks during meta-training and meta-testing, respectively.  
 Highly ``popular'', or more frequently observed, tasks in the source task environment may have a smaller chance of being observed, or they may even not appear, in the target task environment, while new tasks may only be encountered during meta-testing. For example, a personalized health application may be meta-trained by using data from a population of users that is not fully representative of the distribution of the health profiles expected in a different population on which the application is deployed and meta-tested.

The meta-training data set consists of $N$ data sets $\mset=(Z^M_1,\hdots,Z^M_N), $ where $(Z^M_1,\hdots, Z^M_{\beta N}) \triangleq Z^M_{1:\beta N}$, for some fixed $\beta \in (0,1]$, constitutes the source environment data set, with each $i$th sub-data set $Z^M_i$ being generated independently by first drawing a task $T_i \sim P_{T}$ from the source task distribution $P_T$ and then a task-specific data set $Z^M_i \sim P^M_{Z|T=T_i}$. The sub-data sets $(Z^M_{\beta N+1},\hdots, Z^M_N) \triangleq Z^M_{\beta N+1:N}$ belong to the target environment with each $i$th data set generated independently by first drawing a task $T_i \sim P'_T$ and then task specific data set $Z^M_i \sim P^M_{Z|T=T_i}$. All distributions $P_T,P'_T$ and $\{P_{Z|T=\tau}\}_{\tau \in \mathcal{T}}$ are unknown to the meta-learner.
 Note that $\beta=1$ corresponds to the extreme scenario in which only data from source task environment is available for meta-training.
Considering a randomized meta-learner $U \sim P_{U|\mset}\in \Pscr(\Uscr)$ as in \eqref{eq:meta_learner},
the goal of the meta-learner is to minimize the 
 \textit{transfer meta-generalization loss}
\begin{align}
\Lscr_{g}'(u)=\Ebb_{P'_{T}P^M_{Z|T}} \bigl[L_g(u|Z^M,T) \bigr]
 \label{eq:tfr_metatestloss},
\end{align} evaluated on a new meta-test task $T\sim P'_T$ drawn from the target task distribution $P'_T$ and on the corresponding training data $Z^M$ drawn i.i.d. from the data distribution $P_{Z|T}$. 

In analogy with the weighted average training loss \eqref{eq:trainingloss} used for transfer learning, we propose that the meta-learner aims at minimizing the \textit{weighted average meta-training loss} on the meta-training set $\mset$, which is defined as
\begin{align}
\Lscr_t(u|\mset)= \frac{\alpha}{\beta N} \sum_{i=1}^{\beta N}L_t(u|Z^M_i) +\frac{1-\alpha}{(1-\beta)N}\sum_{i=\beta N+1}^N L_t(u|Z^M_i), \label{eq:tfr_metatrainingloss}
\end{align}
for some hyper-hyperparameter $\alpha \in [0,1]$. We note that this formulation assumes that the meta-learner knows which data comes from the source task environment and which are from the target task environment. We distinguish the transfer meta-generalization loss and the meta-training loss via the subscripts $g$, $t$ of $\Lscr'_g(u)$ and $\Lscr_t(u|\mset)$ respectively, with the superscript $'$ of $\Lscr_g'(u)$ denoting that the generalization loss is evaluated with respect to the target task distribution $P'_T$.
The meta-training loss \eqref{eq:tfr_metatrainingloss} can be computed by the meta-learner based on the meta-training data $\mset$ and it can be used as a criterion to select the hyperparameter $u$ (for a fixed $\alpha$). We refer to the meta-training algorithm that outputs the hyperparameter that minimizes \eqref{eq:tfr_metatrainingloss} as Empirical Meta-Risk Minimization (EMRM). Note that EMRM is deterministic with $P_{U|\mset}=\delta(U-U^{\rm EMRM}(\mset))$ where
\begin{align}
U^{\rm EMRM}(\mset)=\arg \min_{u \in \Uscr} \Lscr_t(u|\mset). \label{eq:EMRM}
\end{align} Here, and hence forth, we take arg min to output any one of the optimal solutions of the problem at hand and we assume that the set of optimal solutions is not empty. In the following sections, we also use loss functions with double subscript. For example, $\Lscr'_{g,t}(u)=\Ebb_{P'_T P^M_{Z|T}}[L_t(u|Z^M)]$, defined in \eqref{eq:decomposition-1}, with subscripts $g,t$ denote that it accounts for the generalization loss (`$g$') at the environment level (with average over $T\sim P'_T$ and $Z^M$), and the empirical training loss (`$t$') at the task level ($L_t(u|Z^M)$).
We conclude this section with the following remark.\\
\begin{remarkc}\label{rem:1}
The transfer meta-learning setting introduced here generalizes conventional learning, transfer learning and meta-learning:
\begin{enumerate}
\item When $\beta=1$, only data from source task environment is available for meta-training. If, in addition, source and target task distributions are equal, \ie, if $P_T=P'_T$, we recover the conventional meta-generalization problem reviewed in Section~\ref{sec:metalearning}.
\item Consider now the special case where source and target task distributions are concentrated around two specific tasks $\tau$ and $\tau'$ respectively, that is, we have $P_T=\delta(T-\tau)$ and $P'_T=\delta(T-\tau')$ for some $\tau,\tau' \in \mathcal{T}$.  With $N=2$, the meta-training set $\mset=(Z^{\beta NM}_{\tau},Z^{(1-\beta)NM}_{\tau'})$ with $Z^{\beta NM}_{\tau} \sim P^{\beta NM}_{Z|T=\tau}$ and $Z^{(1-\beta)NM}_{\tau'}\sim P^{(1-\beta)NM}_{Z|T=\tau'}$ contains samples that are generated i.i.d. from the source data distribution $P_{Z|T=\tau}$ and target data distribution $P_{Z|T=\tau'}$.  Assume that the base learner neglects data from the task to output always the hyperparameter $U$, \ie, $P_{W|Z^M,U}=\delta(W-U)$. Upon fixing $W=U$, we then have the meta-learner $P_{U|\mset}=P_{W|\mset}$. With these choices, the problem of transfer meta-learning  reduces to the conventional transfer learning reviewed in Section~\ref{sec:transferlearning} by mapping the transfer meta-generalization loss $\Lscr_g'(u)$ to the generalization loss $L_g(w|\tau')=L_g(u|\tau')$ and the meta-training loss $\Lscr_t(u|\mset)$ to the training loss $L_t(w|Z^{NM})=L_t(u|Z^{NM})$.
\end{enumerate}
\end{remarkc}\\

\section{Information-Theoretic Analysis of Empirical Meta-Risk Minimization }
In this section, we focus on the information-theoretic analysis of  empirical meta-risk minimization (EMRM), which is defined by the optimization \eqref{eq:EMRM}. To this end, we will first study bounds on the average transfer meta-generalization gap for \textit{any} meta-learner $P_{U|\mset}$, where the average is taken with respect to $P_{\mset}P_{U|\mset}$. Since our goal is to specialize the derived bound to a deterministic algorithm like EMRM, we obtain individual task based bounds \cite{jose2020information}, which yield non-vacuous bounds for deterministic mappings from the space of $\mset$ to $\Uscr$. We then apply the results to analyze the average transfer excess meta-risk for EMRM. We refer to Section~\ref{sec:PAC-Bayesian bound} for PAC-Bayesian bounds and Section~\ref{sec:singledraw} for single-draw bounds on transfer meta-generalization gap. We start with a formal definition of the performance criteria of interest.

 The \textit{transfer meta-generalization gap} is the difference between the transfer meta-generalization loss \eqref{eq:tfr_metatestloss} and the meta-training loss \eqref{eq:tfr_metatrainingloss}.  For any given hyperparameter $u \in \Uscr$, it is defined as
\begin{align}
\Delta \Lscr'(u|\mset)=\Lscr'_g(u)-\Lscr_t(u|\mset)\label{eq:tfr_metagengap}.
\end{align} For a general stochastic meta-learner $P_{U|\mset}$, the \textit{average transfer meta-generalization gap} is obtained as
\begin{align}
\Ebb_{P_{\mset}P_{U|\mset}}[\Delta \Lscr'(U|\mset)]\label{eq:avgtfr_metagengap}
\end{align}
with the expectation taken over the meta-training data set $\mset$ and  hyperparameter $U \sim P_{U|\mset}$. Note that $P_{\mset}$ is the marginal of the product distribution $ \prod_{i=1}^{\beta N}P_{T_i}P^M_{Z|T=T_i} \prod_{i=\beta N+1}^N P'_{T_i}P^M_{Z|T=T_i}$, as described in the previous section. 
The average transfer meta-generalization gap \eqref{eq:avgtfr_metagengap} quantifies how close the meta-training loss is  to the transfer meta-generalization loss, which is the desired, but unknown, meta-learning criterion. If the transfer meta-generalization gap is sufficiently small, the meta-training loss can be taken as a reliable measure of the transfer meta-generalization loss. In this case, one can expect EMRM \eqref{eq:EMRM}, which relies on the minimum of the weighted meta-training loss $\Lscr_t(u|\mset)$, to perform well.

The \textit{average transfer excess meta-risk} evaluates the performance of a meta-training algorithm with respect to the optimal hyperparameter $u^{*}$ that minimizes the transfer meta-generalization loss \eqref{eq:tfr_metatestloss}. For a fixed class of base learners $P_{W|Z^M,u}$, the optimal hyperparameter minimizing \eqref{eq:tfr_metatestloss} is given by
\begin{align}
u^{*}=\arg \min_{u \in \Uscr} \Lscr'_g(u).
\end{align}
The \textit{average transfer excess meta-risk of the EMRM algorithm} is hence computed as
\begin{align}
\Ebb_{P_{\mset}}[\Lscr'_g(U^{\rm EMRM}(\mset))-\Lscr'_g(u^{*})]. \label{eq:transfer_metaexcessrisk}
\end{align}

In the next subsection, we present the technical assumptions underlying the analysis, as well as some exponential inequalities that will play a central role in the derivations. In Section~\ref{sec:EMRM_avgtfrgap}, we obtain upper bounds on the average transfer meta-generalization gap \eqref{eq:avgtfr_metagengap} for any meta-learner, while Section~\ref{sec:EMRM_excessrisk} focuses on bounding the average transfer excess meta-risk \eqref{eq:transfer_metaexcessrisk} for EMRM.
\subsection{Assumptions and Exponential Inequalities}
We start by defining $\sigma^2$-sub-Gaussian random variables.
\begin{definition}
A random variable $X \sim P_X$ with finite mean, \ie, $\Ebb_{P_X}[X]<\infty$, is said to be $\sigma^2$-sub-Gaussian if its moment generating function satisfies
\begin{align}
\Ebb_{P_X}[\exp(\lambda(X-\Ebb_{P_X}[X]))]\leq \exp \biggl(\frac{\lambda^2 \sigma^2}{2} \biggr), \quad \mbox{for all} \hspace{0.2cm} \lambda \in \Real.
\end{align} 
\end{definition}
Moreover,
if $X_i$, $i=1,\hdots, n$ are independent $\sigma^2$-sub-Gaussian random variables, then the average $\sum_{i=1}^n X_i/n$ is $\sigma^2/n$-sub-Gaussian.

Throughout, we denote as $P_U$ the marginal of the joint distribution $P_{\mset}P_{U|\mset}$ induced by the meta-learner. We also use $P_{Z^M}$ to denote the marginal of the joint distribution $P_TP^M_{Z|T}$ of the data under the source environment and, in a similar manner, $P'_{Z^M}$ to denote the marginal of the joint distribution $P'_TP^M_{Z|T}$ of the data under the target environment.
In the rest of this section, we make the following assumptions on the loss function.
\begin{assumption}\label{assum:1}
The environment distributions $P_T, P'_T$, and $\{P_{Z|T=\tau}\}_{\tau \in \mathcal{T}}$, the base learner $P_{W|Z^M,U}$, and the meta-learner $P_{U|\mset}$ satisfy the following assumptions.
\begin{enumerate}
\item [$(a)$] For each task $\tau \in \mathcal{T}$, the loss function $l(W,Z)$ is $\delta_{\tau}^2$-sub-Gaussian when $(W,Z) \sim P_{W|T=\tau} P_{Z|T=\tau}$, where $P_{W|T=\tau}$ is the marginal of the model parameter trained for task $\tau$, which is obtained by marginalizing the joint distribution $P_{U}P^M_{Z|T=\tau}P_{W|Z^M,U}$;
\item[$(b)$]The per-task average training loss $L_t(U|Z^M)$ is $\sigma^2$-sub-Gaussian when $(U,Z^M) \sim P_U P'_{Z^M}$.
\end{enumerate}
\end{assumption}

We note that the sub-Gaussianity properties in Assumption~\ref{assum:1}$(a)$ and Assumption~\ref{assum:1}$(b)$ are with respect to different distributions. As such, satisfying Assumption~\ref{assum:1}$(a)$ does not guarantee sub-Gaussianity in the sense of Assumption~\ref{assum:1}$(b)$. However, if the loss function is bounded, \ie, $l(\cdot,\cdot) \in [a,b]$ for $0\leq a\leq b <\infty$, it can be verified that both of these assumptions hold with $\delta_{\tau}^2= (b-a)^2/4=\sigma^2$ for any $\tau \in \mathcal{T}$.
\begin{definition}
The information density between two discrete or continuous random variables $(A,B) \sim P_{A,B}$ with well-defined joint probability mass or density function $P_{A,B}(a,b)$, and marginals $P_A(a)$ and $P_B(b)$ is the random variable
\begin{align}
\imath(A,B)=\log \frac{ \dP_{A,B}(A,B)}{\dP_A(A) P_B(B)}= \log \frac{\dP_{A|B}(A|B)}{\dP_A(A)}.
\end{align}
\end{definition}
The information density quantifies the evidence for the hypothesis that $A$ is produced from $B$ via the stochastic mechanism $P_{A|B}$ rather than being drawn from the marginal $P_A$. The average of the information density is given by the mutual information (MI)
$I(A;B)=\Ebb_{P_{A,B}}[\imath(A,B)]$.

In the analysis, the information densities $\imath(U,Z^M_i)$ for $i=1,\hdots, N$, and $\imath(W,Z_j|T=\tau)$ for $j=1,\hdots,M$ will play a key role. The information density $\imath(U,Z^M_i)$ is defined for random variables $(U,Z^M_i) \sim P_{U,Z^M_i}$, where $P_{U,Z^M_i}$ is obtained by marginalizing the joint distribution $P_{\mset}P_{U|\mset}$ over the subsets $Z^M_j$ of the meta-training set $\mset$ for all $j \neq i, j=1,\hdots,N$. Similarly, the information density $\imath(W,Z_j|T=\tau)$ is defined for random variables $(W,Z_j) \sim P_{W,Z_j|T=\tau}$, where $P_{W,Z_j|T=\tau}$ is obtained by marginalizing the joint distribution $P_U P_{W|Z^M,U}P^M_{Z|T=\tau}$ over $U$ and over data samples $Z_k$ of the training set $Z^M$ for all $k \neq j$ with $ k=1,\hdots,M$.
The information density $\imath(U,Z^M_i)$ quantifies the evidence for the hyperparameter $U$ to be generated by the meta-learner $P_{U|\mset}$ based on meta-training data that includes the data set $Z^M_i$.
Similarly, the evidence for the model parameter $W$ to be produced by the base learner $P_{W|Z^M}$ (which is the marginal of the joint distribution $P_U P_{W|Z^M,U}$) based on the training set for task $\tau$ that includes the data sample $Z_j$ is captured by the information density $\imath(W,Z_j|T=\tau)$.
All these measures can also be interpreted as the sensitivity of hyperparameter and model parameter to per-task data set $Z^M_i$ (from source or target environment) and data sample $Z_j$ within per-task data set, respectively.
Moreover, the average of these information densities yield the following MI terms
\begin{align}
I(U;Z^M_i)&=
\begin{cases}
\Ebb_{P_{Z^M_i}P_{U|Z^M_i}}[\imath(U,Z^M_i)] & \mbox{for} \hspace{0.2cm} i=1,\hdots, \beta N,\\
\Ebb_{P'_{Z^M_i}P_{U|Z^M_i}}[\imath(U,Z^M_i)] & \mbox{for}  \hspace{0.2cm} i=\beta N+1,\hdots, N,
\end{cases}\non\\
I(W;Z_j|T=\tau)&=\Ebb_{P_{W,Z_j|T=\tau}}[\imath(W,Z_j|T=\tau)]  \hspace{0.2cm} \mbox{for} \hspace{0.2cm} j=1,\hdots,M. \label{eq:MIterms}
\end{align}
 
 \begin{assumption}\label{assum:1a}
The source environment data distribution satisfies $P_{Z^M}(z^M)=0$ almost surely for all $z^M=(z_1,\hdots,z_M) \in \Zscr^M$ such that $P'_{Z^M}(z^M)=0$.
 \end{assumption}


We are now ready to present two important inequalities that will underlie the analysis in the rest of the section. We note that a similar unified approach was presented in \cite{hellstrom2020generalization} to study generalization in conventional learning, and our methodology is inspired by this work. The proofs for these inequalities can be found in Appendix~\ref{app:expinequality_avgtfrgap_ITMI}.
\begin{lemma}\label{lem:expinequality_avgtfrgap_ITMI}
Under Assumption~\ref{assum:1}$(a)$, the following inequality holds 
\begin{align}
\Ebb_{P_{W,Z_j|T=\tau}}\biggl[\exp \biggl( \lambda( l(W,Z_j)- &\Ebb_{P_{W|T=\tau}P_{Z_j|T=\tau}}[l(W,Z_j)] -\frac{\lambda^2\delta_{\tau}^2}{2} -\imath(W,Z_j|T=\tau) \biggr) \biggr] \leq 1, 
\label{eq:expinequality_avgtfr_task_ITMI}
\end{align} 
 for all $j=1,\hdots,M$, $\lambda \in \Real$ and for each task $\tau \in \mathcal{T}$. 
\end{lemma}
\begin{lemma}\label{lem:expinequality_avgtfrgap_ITMI_1}
Under Assumption~\ref{assum:1}$(b)$ and Assumption~\ref{assum:1a}, we have the following inequalities
\begin{align}
\Ebb_{P'_{Z^M_i}P_{U|Z^M_i}}\biggl[\exp \biggl( \lambda( L_t(U|Z^M_i)- \Ebb_{P_{U}P'_{Z^M_i}}[ L_t(U|Z^M_i)] -\frac{\lambda^2\sigma ^2}{2} -\imath(U,Z^M_i)\biggr) \biggr] \leq 1,  \label{eq:expinequality_avgtfr_env_ineq1_ITMI}
\end{align} for $i=\beta N+1,\hdots N$ and
\begin{align}
\Ebb_{P_{Z^M_i}P_{U|Z^M_i}}\biggl[&\exp \biggl( \lambda( L_t(U|Z^M_i)- \Ebb_{P_{U}P'_{Z^M_i}}[ L_t(U|Z^M_i)] -\frac{\lambda^2\sigma ^2}{2} \non \\& -\log \frac{\dP_{Z^M_i}(Z^M_i)}{\dP'_{Z^M_i}(Z^M_i)} -\imath(U,Z^M_i) \biggr) \biggr] \leq 1, \label{eq:expinequality_avgtfr_env_ineq2_ITMI}
\end{align} for $i=1,\hdots, \beta N$, which holds for all $\lambda \in \Real$.
\end{lemma}

Inequalities \eqref{eq:expinequality_avgtfr_task_ITMI}--\eqref{eq:expinequality_avgtfr_env_ineq2_ITMI} relate the per-task training and meta-training loss functions to the corresponding ensemble averages and information densities, and will be instrumental in deriving information theoretic bounds on average transfer meta-generalization gap and average transfer excess meta-risk. 
\subsection{Bounds on the Average Transfer Meta-Generalization Gap}\label{sec:EMRM_avgtfrgap}
In this section, we derive upper bounds on the average transfer meta-generalization gap \eqref{eq:avgtfr_metagengap} for a general meta-learner $P_{U|\mset}$. The results will be specialized to the EMRM meta-learner in Section~\ref{sec:EMRM_excessrisk}.

To start, we decompose the average transfer meta-generalization gap \eqref{eq:avgtfr_metagengap} as
\begin{align}
&\Ebb_{ P_{\mset,U}}[\Delta \Lscr'(U|\mset)]=\Ebb_{ P_{\mset,U}}\bigl[\bigl(\Lscr'_g(U)-\Lscr'_{g,t}(U)\bigr)+\bigl(\Lscr'_{g,t}(U)-\Lscr_t(U|\mset)\bigr) \bigr], \label{eq:decomposition-1}
\end{align}
where we have used the notation $P_{\mset,U}=P_{\mset}P_{U|\mset}$, and  $\Lscr'_{g,t}(u)$ is the average training loss when data is drawn from the distribution $P_{Z|T}$ of a task $T$ sampled from the target task distribution $ P'_T$, \ie
\begin{align}
\Lscr'_{g,t}(u)=\Ebb_{P'_T}\Ebb_{P^M_{Z|T}}[L_t(u|Z^M)]. \label{eq:auxiliaryloss_1}
\end{align} A summary of all definitions for transfer meta-learning can be found in Figure~\ref{fig:variablesrelation}.
\begin{figure}[h!]
\centering
\includegraphics[scale=0.5,clip=true, trim = 0in  0in 0in 0in]{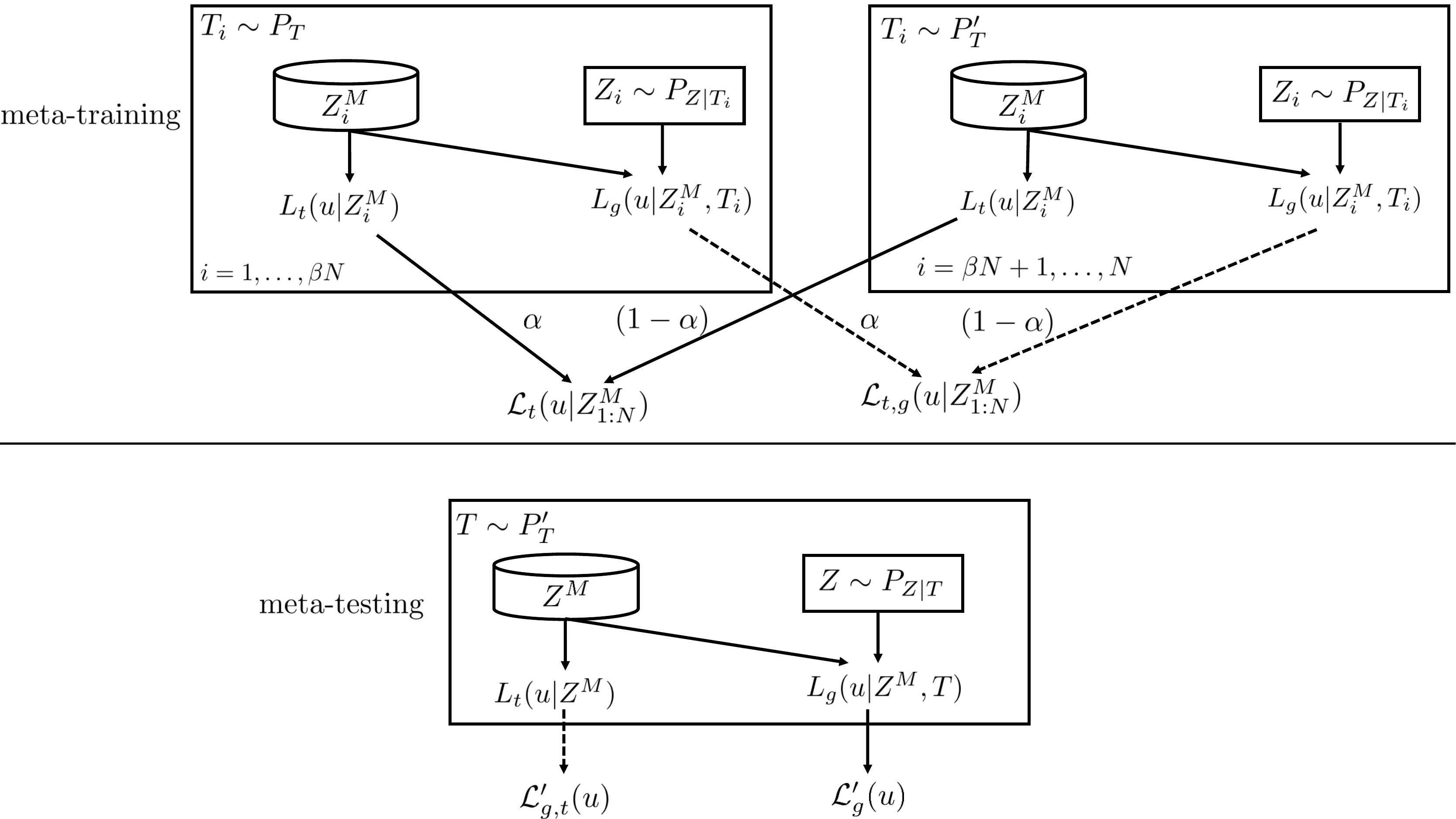} 
\caption{ Illustration of the variables involved in the definition of transfer meta-generalization gap \eqref{eq:tfr_metagengap}. }\label{fig:variablesrelation}
\vspace{-0.5cm}
\end{figure}

The decomposition \eqref{eq:decomposition-1} captures two distinct contributions to the meta-generalization gap in transfer meta-learning. The first difference in \eqref{eq:decomposition-1} accounts for the \textit{within-task generalization gap} that is caused by the observation of a finite number $M$ of data samples for the meta-test task. In contrast,  the second difference accounts for the \textit{environment-level generalization gap} that results from the finite number of observed tasks ($\beta N$ from the source environment and $(1-\beta)N$ from the target environment), as well as from the \textit{meta-environment shift} in task distributions from $P_T$ to $P'_T$. To upper bound the average transfer meta-generalization gap, we proceed by separately bounding the two differences in \eqref{eq:decomposition-1} via the exponential inequalities \eqref{eq:expinequality_avgtfr_task_ITMI}--\eqref{eq:expinequality_avgtfr_env_ineq2_ITMI}.  This results in the following information-theoretic upper bound for transfer meta-learning that extends the individual sample mutual information based approach in \cite{bu2019tightening} for conventional learning.
\begin{theorem} \label{thm:avgtfrgap_bound1_ITMI}
Under Assumption~\ref{assum:1} and Assumption~\ref{assum:1a}, the following upper bound on the average transfer meta-generalization gap holds for $\beta \in (0,1)$
\begin{align}
&|\Ebb_{P_{\mset,U}}[\Delta \Lscr'(U|\mset)]|\non \\ & \leq \frac{\alpha}{\beta N} \sum_{i=1}^{\beta N} \sqrt{2 \sigma^2 \biggl(D(P_{Z^M}||P'_{Z^M})+I(U;Z^M_i) \biggr)}+ \frac{1-\alpha}{(1-\beta) N} \sum_{i=\beta N+1}^{N}\sqrt{2 \sigma^2 I(U;Z^M_i) } \non \\
&+\Ebb_{P'_T}\biggl[\frac{1}{M} \sum_{j=1}^M \sqrt{2 \delta_{T}^2I(W;Z_j|T=\tau)}\biggr] \label{eq:avgtfrgap_bound1_ITMI},
\end{align}
with the MI terms defined in \eqref{eq:MIterms}.
\end{theorem}
\begin{proof}
See Appendix~\ref{app:avgtfrgap_bound1_ITMI}.
\end{proof}

The upper bound \eqref{eq:avgtfrgap_bound1_ITMI} on the average transfer meta-generalization gap is expressed in terms of three distinct contributions $(i)$ \textit{source environment-level generalization gap}: the MI $I(U;Z^M_i)$, for $i=1,\hdots,\beta N$, captures the sensitivity of the meta-learner $U$ to the per-task data $Z^M_i$ of the source environment data set, while the meta-environment shift between the source and target environment per-task data is captured by the KL divergence $D(P_{Z^M}||P'_{Z^M})$ ; $(ii)$ \textit{target environment-level generalization gap}: the MI $I(U;Z^M_i)$, for $i=\beta N+1,\hdots, N$, accounts for the sensitivity of the meta-learner to the per-task data sample $Z^M_i$ from the target task environment; and lastly $(iii)$ \textit{within-task generalization gap}: the MI $I(W;Z_j|T=\tau)$ captures the sensitivity of the base learner to the data sample $Z_j$ of the meta-test task data $Z^M\sim P^M_{Z|T=\tau}$.

As $N$ increases, the dependence of a well-designed meta-learner output on each individual task-data set is expected to decrease, yielding a vanishing MI $I(U;Z^M_i)$. Similarly, with an increase in number $M$ of per-task data samples, the MI $I(W;Z_j|T=\tau)$ is expected to decrease to zero. An interesting observation from \eqref{eq:avgtfrgap_bound1_ITMI} is that, even if these conditions are satisfied, as $N$, $M \rightarrow \infty$, the meta-environment shift between source and target task distributions results in a non-vanishing bound on the transfer meta-generalization gap, which is quantified by the KL divergence $D(P_{Z^M}||P'_{Z^M})$. Futhermore, 
when no data from target environment is available for meta-training, the bound in \eqref{eq:avgtfrgap_bound1_ITMI} can be specialized as follows.
\begin{corollary}
Under Assumption~\ref{assum:1} and Assumption~\ref{assum:1a}, when only data from the source environment is available for meta-training, \ie, when $\beta=1$ and $\alpha=1$, the following upper bound on average transfer meta-generalization gap holds
\begin{align}
&|\Ebb_{P_{\mset,U}}[\Delta \Lscr'(U|\mset)]|\non \\ & \leq \frac{1}{ N} \sum_{i=1}^{ N} \sqrt{2 \sigma^2 \biggl(D(P_{Z^M}||P'_{Z^M})+I(U;Z^M_i) \biggr)}+\Ebb_{P'_T}\biggl[\frac{1}{M} \sum_{j=1}^M \sqrt{2 \delta_{T}^2I(W;Z_j|T=\tau)}\biggr].
\end{align}
\end{corollary}

If, in addition, the source and target task distributions coincide, \ie, if $P_T=P'_T$, the bound \eqref{eq:avgtfrgap_bound1_ITMI} recovers the following result presented in \cite[Cor. 5.8]{jose2020information}.
\begin{corollary}
When the source and task environment data distributions coincide, \ie, when $P_T=P'_T$, for $\beta=1$ and $\alpha=1$, we have the following upper bound on average meta-generalization gap
\begin{align}
&|\Ebb_{P_{\mset,U}}[\Delta \Lscr(U|\mset)]|\non \\ & \leq \frac{1}{ N} \sum_{i=1}^{ N} \sqrt{2 \sigma^2 I(U;Z^M_i)}+\Ebb_{P_T}\biggl[\frac{1}{M} \sum_{j=1}^M \sqrt{2 \delta_{T}^2I(W;Z_j|T=\tau)}\biggr].
\end{align}
\end{corollary}

Finally, the upper bound in \eqref{eq:avgtfrgap_bound1_ITMI} on average transfer meta-generalization gap can be specialized to recover the following upper bound \cite[Cor. 2]{wu2020information} on average generalization gap in conventional transfer learning (see Remark~\ref{rem:1}).
\begin{corollary}\label{cor:avgtfrgap_tfrlearning_ISMIbound}
Consider the setting of Theorem~\ref{thm:avgtfrgap_bound1_ITMI} with $P_T=\delta(T-\tau)$ and $P'_T=\delta(T-\tau')$ for some $\tau,\tau' \in \mathcal{T}$. For $N=2$, let the meta-training set be $\mset=(Z^{\beta \bar{M}}_{\tau},Z^{(1-\beta)\bar{M}}_{\tau'}):=Z^{\bar{M}}$ where $\bar{M}=NM$ and $Z^{\beta \bar{M}}_{\tau} \sim P^{\beta \bar{M}}_{Z|\tau}$ and $Z^{(1-\beta)\bar{M}}_{\tau'}\sim P^{(1-\beta)\bar{M}}_{Z|\tau'}$. Assume that $P_{W|Z^M,U}=\delta(W-U)$ and fix $W=U$. Then, the following upper bound on the average generalization gap for transfer learning holds for $\beta \in (0,1)$
\begin{align}
|\Ebb_{P_{Z^{\bar{M}},W}}[L_g(W|\tau')-L_t(W|Z^{\bar{M}})]|&\leq  \frac{\alpha}{\beta \bar{M}} \sum_{i=1}^{\beta \bar{M}} \sqrt{2 \delta_{\tau'}^2 \biggl(D(P_{Z|\tau}||P_{Z|\tau'})+I(W;Z_i) \biggr)}\non \\&+ \frac{1-\alpha}{(1-\beta) \bar{M}} \sum_{i=\beta \bar{M}+1}^{\bar{M}}\sqrt{2 \delta_{\tau'}^2 I(W;Z_i) }. \label{eq:avgtfrgap_transferlearning}
\end{align}
where the MI $I(W;Z_i)$ is evaluated with respect to the joint distribution $P_{W,Z_i|\tau}$ for $i=1,\hdots,\beta \bar{M}$ and is evaluated with respect to the joint distribution $P_{W,Z_i|\tau'}$ for $i=\beta \bar{M}+1,\hdots, \bar{M}$.
\end{corollary}
\begin{proof}
See Appendix~\ref{app:avgtfrgap_tfrlearning_ISMIbound}.
\end{proof}

Finally, we remark that, as proved in Appendix~\ref{app:newexponentialinequalities}, all the upper bounds in this section, starting from 
\eqref{eq:avgtfrgap_bound1_ITMI}, can be also obtained under the following different assumption analogous to the one considered in the work of Xu and Raginsky \cite{xu2017information}. 
\begin{assumption}\label{assum:newassumption}
For every task $\tau \in \mathcal{T}$, the loss function $l(w,Z)$ is $\delta_{\tau}^2$-sub-Gaussian when $Z \sim P_{Z|T=\tau}$ for all $w \in \Wscr$. Similarly, the per-task average training loss $L_t(u|Z^M)$ is $\sigma^2$-sub-Gaussian when $Z^M \sim P'_{Z^M}$ for all $u \in \Uscr$.
\end{assumption} 
 
 As discussed in \cite{negrea2019information}, this assumption does not imply Assumption~\ref{assum:1}, and vice versa. This is unless the loss function $l(\cdot,\cdot)$ is bounded in the interval $[a,b]$, in which case both these assumptions hold.

\subsection{Bounds on Transfer Excess Meta-Risk of EMRM}\label{sec:EMRM_excessrisk}
In this section, we obtain an upper bound on the average transfer meta-excess risk \eqref{eq:transfer_metaexcessrisk} for the EMRM meta-learner \eqref{eq:EMRM}. We will omit the dependence of $U^{\rm EMRM}$ on $\mset$ to simplify notation. We start by decomposing the average transfer excess meta-risk \eqref{eq:transfer_metaexcessrisk} of EMRM as
\begin{align}
&\Ebb_{P_{\mset}}[\Lscr'_g(U^{\rm EMRM})-\Lscr'_g(u^{*})]\non\\
&=\Ebb_{P_{\mset}}\Bigl[\Bigl(\Lscr'_g(U^{\rm EMRM})-\Lscr_t(U^{\rm EMRM}|\mset)\Bigr)+\Bigl(\Lscr_t(U^{\rm EMRM}|\mset) - \Lscr_t(u^{*}|\mset)\Bigr)\non \\&+\Bigl(\Lscr_t(u^{*}|\mset)- \Lscr'_g(u^{*})]\Bigr) \Bigr].
\label{eq:excessrisk_1}
\end{align} We first observe that  we have the inequality $\Lscr_t(U^{\rm EMRM}|\mset) \leq  \Lscr_t(u^{*}|\mset)$ which is by the definition of EMRM \eqref{eq:EMRM}.  Therefore, from \eqref{eq:excessrisk_1}, the average transfer meta-excess risk is upper bounded by the sum of average transfer meta-generalization gap studied above, which is the first difference in \eqref{eq:excessrisk_1}, and of the average difference $\Ebb_{P_{\mset}}[\Lscr_t(u^{*}|\mset)- \Lscr'_g(u^{*})]$, the last difference in \eqref{eq:excessrisk_1}. Combining a bound on this term with the bound \eqref{eq:avgtfrgap_bound1_ITMI} on the transfer meta-generalization gap yields the following upper bound on the average transfer excess meta-risk.
\begin{theorem}\label{thm:transferrisk_bound}
Under Assumption~\ref{assum:newassumption} and Assumption~\ref{assum:1a}, and for $\beta \in (0,1)$, the following upper bound on the average transfer meta-excess risk holds for the EMRM meta-learner \eqref{eq:EMRM}
\begin{align}
&\Ebb_{P_{\mset}}[\Lscr'_g(U^{\rm EMRM})-\Lscr'_g(u^{*})]\non \\
&\leq \frac{\alpha}{\beta N} \sum_{i=1}^{\beta N} \sqrt{2 \sigma^2 \biggl(D(P_{Z^M}||P'_{Z^M})+I(U^{\rm EMRM};Z^M_i) \biggr)}\non \\&+ \frac{1-\alpha}{(1-\beta) N} \sum_{i=\beta N+1}^{N}\sqrt{2\sigma^2 I(U^{\rm EMRM};Z^M_i) } +\Ebb_{P'_T}\biggl[\frac{1}{M} \sum_{j=1}^M \sqrt{2 \delta_{T}^2I(W;Z_j|T=\tau)}\biggr] \non \\
& + \alpha \sqrt{2 \sigma^2 D(P_{Z^M}||P'_{Z^M})}+\Ebb_{P'_T}\biggl[\frac{1}{M} \sum_{j=1}^M \sqrt{2 \delta_T^2 I(W;Z_j|T=\tau,u^{*}}) \biggr], \label{eq:transferrisk_bound}
\end{align} 
where the MI terms are defined in \eqref{eq:MIterms} with $U=U^{\rm EMRM}$.
\end{theorem}
\begin{proof}
See Appendix~\ref{app:transferrisk_bound}.
\end{proof}

Comparing \eqref{eq:transferrisk_bound} with \eqref{eq:avgtfrgap_bound1_ITMI} reveals that, in addition to the terms contributing to the average transfer meta-generalization gap, the excess meta-risk of EMRM meta-learner also includes the KL divergence between source and target environment per-task data $D(P_{Z^M}||P'_{Z^M})$ and the MI $I(W;Z_j|u^{*},\tau)$. The latter captures the sensitivity of the base learner $P_{W|Z^M,u^{*}}$ under the optimal hyperparameter $u^{*}$ to a training sample $Z_j$ of the meta-test task data $Z^M \sim P'_{Z^M}$. Since $u^{*}$ is unknown in general, once can further upper bound this mutual information by  the supremum value $\sup_{u \in \Uscr} I(W;Z_j|T=\tau,u)$. 

All the bounds obtained in this section depend on the distributions of source and target task environments, namely $P_T$, and $P'_T$, and per-task data distributions $\{P_{Z|T=\tau}\}_{\tau \in \mathcal{T}}$, all of which are generally unknown. In the next section, we obtain high-probability PAC-Bayesian bounds on the transfer meta-generalization gap, which are in general independent of these distributions except for the quantity that captures the meta-environment shift. We further build on this bound to define a novel meta-learner inspired by the principle of information risk minimization \cite{zhang2006information}.

%
\section{Information Risk Minimization for Transfer Meta-Learning}
In this section, we first obtain a novel PAC-Bayesian bound on the transfer meta-generalization gap which holds with high probability over the meta-training set.  Based on the derived bound, we then propose a new meta-training algorithm, termed Information Meta Risk Minimization (IMRM), that is inspired by the principle of information risk minimization \cite{zhang2006information}. This will be compared to EMRM through a numerical example in Section~\ref{sec:example}. 

We first discuss in the next sub-section some technical assumptions that are central to the derivation of PAC-Bayesian bound for transfer meta-learning. We then present the PAC-Bayesian bounds in Section~\ref{sec:PAC-Bayesian bound}, and we introduce IMRM in Section~\ref{sec:IMRM}.
\subsection{Assumptions}
The derivation of the PAC-Bayesian bound relies on slightly different conditions than Assumption~\ref{assum:newassumption}, which are stated next. 
\begin{assumption}\label{assum:2}
The environment distributions $P_T,P'_T$ and $\{P_{Z|T=\tau}\}_{\tau \in \mathcal{T}}$, the base learner $P_{W|Z^M,U}$ and the meta-learner $P_{U|\mset}$ satisfy the following assumptions.
\begin{enumerate}
\item [$(a)$] For each task $\tau \in \mathcal{T}$, the loss function $l(w,Z)$ is $\delta_{\tau}^2$-sub-Gaussian under $Z \sim P_{Z|\tau}$ for all $w \in \Wscr$.
\item[$(b)$] The average per-task generalization loss $L_g(u|T,Z^M)$ in \eqref{eq:avgtestloss} is $\sigma^2$-sub-Gaussian when $(T,Z^M) \sim P'_T P^M_{Z|T}$ for all $u \in \Uscr$.
\end{enumerate}
\end{assumption}
Assumption~\ref{assum:2}$(a)$ on the loss function $l(w,Z)$ is the same as the one considered in Assumption~\ref{assum:newassumption}. In contrast, while Assumption~\ref{assum:2}$(b)$ is on the average per-task generalization loss, Assumption~\ref{assum:newassumption} considers average per-task training loss. This distinction is necessary in order to also bound the task-level generalization gap in high probability. If the loss function is bounded in the interval $[a,b]$, then both Assumption~\ref{assum:newassumption} and Assumption~\ref{assum:2} are satisfied with  $\sigma^2=\delta_{\tau}^2$ for all $\tau \in \mathcal{T}$.

PAC-Bayes bounds depend on arbitrary reference data-independent ``prior" distributions that allow the evaluation of sensitivity measures for base learners \cite{guedj2019primer} and meta-learners \cite{amit2018meta}. Accordingly, in the following sections, we consider a hyper-prior $Q_U \in \Pscr(\Uscr)$ for the  hyperparameter and a family of priors $Q_{W|U=u} \in \Pscr(\Wscr)$ for each $u \in \Uscr$ satisfying the following assumption.
\begin{assumption}\label{assum:2a}
The hyper-prior $Q_U \in \Pscr(\Uscr)$ must satisfy that  $P_{U|\mset=z^M_{1:N}}(u)=0$ almost surely for every $u \in \Uscr$ such that $Q_U(u)=0$, for all $z^M_{1:N} \in \Zscr^{MN}$.
 Similarly, for given $u \in \Uscr$, the prior $Q_{W|U=u} \in \Pscr(\Wscr)$ must satisfy that $P_{W|Z^M=z^M,U=u}(w)=0$ almost surely for every $ w \in \Wscr$ such that $Q_{W|U=u}(w)=0$, for all $z^M \in \Zscr^M$.
 Finally, $P_T(\tau)=0$ almost surely for every $\tau \in \mathcal{T}$ such that $P'_T(\tau)=0$.
\end{assumption}

The derivation of PAC-Bayesian bound is based on novel exponential inequalities that are derived in a similar manner as in the previous section and can be found in Appendix~\ref{app:expinequality_PACBayesian}. In the following, we use $T_{1:N}=(T_1,\hdots,T_N)$ to denote the $N$ selected tasks for generating the meta-training data set $\mset$ 
with $P_{T_{1:N}}=\prod_{i=1}^{\beta N}P_{T_i} \prod_{j=\beta N+1}^N P'_{T_i}$ and $P_{\mset|T_{1:N}}$ denoting the product distribution $\prod_{i=1}^N P^M_{Z|T_i}$.


\subsection{PAC-Bayesian Bound for Transfer Meta-Learning}\label{sec:PAC-Bayesian bound}
In this section, we focus on obtaining PAC-Bayesian bounds of the following form: With probability at least $1-\delta$ over the distribution of meta-training tasks and data $(T_{1:N},\mset) \sim P_{T_{1: N}}P_{\mset|T_{1:N}}$, the transfer meta-generalization gap satisfies
\begin{align}
\Bigl| \Ebb_{P_{U|\mset}}[\Delta \Lscr'(U|\mset)]\Bigr| \leq \epsilon,
\end{align} for $\delta \in (0,1)$.
To start, we define the the empirical weighted average of the per-task test loss of the meta-training set as
\begin{align}
\Lscr_{t,g}(u|\mset,T_{1:N})=\frac{\alpha}{\beta N} \sum_{i=1}^{\beta N}L_g(u|Z^M_i,T_i) +\frac{1-\alpha}{(1-\beta)N}\sum_{i=\beta N+1}^N L_g(u|Z^M_i,T_i),\label{eq:auxiliaryloss_2}
\end{align} where $L_g(u|Z^M_i,T_i)$ is defined in \eqref{eq:avgtestloss}. Then,  the transfer meta-generalization gap can be decomposed as
\begin{align}
&\Ebb_{P_{U|\mset}}\bigl[\Delta \Lscr'(U|\mset)\bigr]\non \\ &= \Ebb_{P_{U|\mset}}\Bigl[\Bigl( \Lscr'_g(U)-\Lscr_{t,g}(U|\mset,T_{1:N}) \Bigr)+ \Bigl(\Lscr_{t,g}(U|\mset,T_{1:N}) -\Lscr_t(U|\mset)\Bigr)\Bigr]. \label{eq:decomposition-2}
\end{align}
In \eqref{eq:decomposition-2}, the first difference accounts for the \textit{environment-level generalization gap} resulting from the observation of a finite number $N$ of meta-training tasks and also from the meta-environment shift between source and target task distributions. The second difference accounts for the \textit{within-task generalization gap} in each subset of the meta-training set $\mset$ arising from observing a finite number $M$ of per-task data samples.
We note that the decomposition in \eqref{eq:decomposition-2} can also be used to obtain an upper bound on the average transfer meta-generalization gap.
 However, the resulting bound
   does not recover the bound in \cite{jose2020information},
or  specialize to the case of conventional transfer learning. We leave a full investigation of this alternate bound to future work.

 As in the bounds on average transfer meta-generalization gap presented in Section~\ref{sec:EMRM_avgtfrgap}, the idea is to separately bound the above two differences in high probability over $(T_{1:N},\mset)\sim P_{T_{1:N}}P_{\mset|T_{1:N}}$ and then combine the results via union bound.
This results in the following PAC-Bayesian bound.

%
\begin{theorem}\label{thm:PAC-Bayesianbound}
For a fixed base learner $P_{W|Z^M,U}$, let $Q_U \in \Pscr(\Uscr)$ be an arbitrary hyper-prior distribution over the space of hyper-parameters and $Q_{W|U=u} \in \Pscr(\Wscr)$ be an arbitrary prior distribution over the space of model parameters for each $u \in \Uscr$ and $\beta \in (0,1)$. Then, under Assumption~\ref{assum:2} and Assumption~\ref{assum:2a}, the following inequality holds uniformly for any meta-learner $P_{U|\mset}$ with probability at least $1-\delta$, $\delta \in (0,1)$, over $(T_{1:N},\mset)\sim P_{T_{1:N}}P_{\mset|T_{1:N}}$
\begin{align}
&\biggl| \Ebb_{P_{U|\mset}}[\Delta \Lscr'(U|\mset)]\biggr| \leq \sqrt{ 2 \sigma^2 \biggl( \frac{ \alpha^2 }{\beta N}+\frac{(1-\alpha)^2 }{(1-\beta) N}\biggr) \biggl( \sum_{i=1}^{\beta N}\log \frac{P_{T}(T_i)}{P'_{T}(T_i)}+D(P_{U|\mset}||Q_U)+ \log \frac{2}{\delta}\biggr)} \non \\
&+\frac{\alpha}{\beta N} \sum_{i=1}^{\beta N} \sqrt{\frac{2 \delta_{T_i}^2}{M}\biggl(D(P_{U|\mset}||Q_U)+ \Ebb_{P_{U|\mset}}[ D(P_{W|U,Z^M_i}||Q_{W|U})]+\log \frac{4\beta N}{\delta}\biggr) } \non \\
&+\frac{1-\alpha}{(1-\beta) N} \sum_{i=\beta N+1}^{ N} \sqrt{\frac{2 \delta_{T_i}^2}{M}\biggl(D(P_{U|\mset}||Q_U)+ \Ebb_{P_{U|\mset}}[ D(P_{W|U,Z^M_i}||Q_{W|U})]+\log \frac{4(1-\beta)N}{\delta}\biggr) }.
\label{eq:PACBayesian_bound1}
\end{align}
\end{theorem}
\begin{proof}
See Appendix~\ref{app:PAC-Bayesianbound}.\end{proof}

The first term in the upper bound \eqref{eq:PACBayesian_bound1} captures the environment-level generalization gap through the log-likelihood ratio $\log (P_{T}(T_i)/P'_{T}(T_i))$, which accounts for the meta-environment shift, and through the KL divergence $D(P_{U|\mset}||Q_U)$. This quantifies the sensitivity of the meta-learner $P_{U|\mset}$ to the meta-training set $\mset$ through its divergence with respect to the data-independent hyper-prior $Q_U$. The second term of \eqref{eq:PACBayesian_bound1} captures the generalization gap within the task data from source environment in terms of the average KL divergence $\Ebb_{P_{U|\mset}}[ D(P_{W|U,Z^M_i}||Q_{W|U})]$ between model posterior and prior distributions together with $D(P_{U|\mset}||Q_U)$, while the last term accounts for the generalization gap within the task data from the target environment. We note that the average KL divergence, $\Ebb_{P_{U|\mset}}[ D(P_{W|U,Z^M_i}||Q_{W|U})]$, quantifies the sensitivity of the base learner $P_{W|Z^M,U}$ to the training set $Z^M$ through its divergence with respect to the data-independent prior $Q_{W|U}$ for a hyperparameter $U \sim P_{U|\mset}$.

The bound in \eqref{eq:PACBayesian_bound1} can be relaxed to obtain the following looser bound that is more amenable to optimization, as we will discuss in the next subsection.
\begin{corollary}\label{cor:PACBayesian_looser}
In the setting of Theorem~\ref{thm:PAC-Bayesianbound}, the following inequality holds with probability at least $1-\delta$ over $(T_{1:N},Z^M_{1:N})\sim P_{T_{1:N}}P_{\mset|T_{1:N}}$ for $\beta \in (0,1)$,
\begin{align}
\Ebb_{P_{U|\mset}}[\Lscr_g'(U)]& \leq \Ebb_{P_{U|\mset}}\biggl[\Lscr_t(U|\mset)+\frac{\alpha}{\beta NM} \sum_{i=1}^{\beta N} D(P_{W|Z^M_i,U}||Q_{W|U})\non\\
&+\frac{1-\alpha}{(1-\beta) NM} \sum_{i=\beta N+1}^{ N} D(P_{W|Z^M_i,U}||Q_{W|U}) \biggr] +\biggl(\frac{1}{N}+\frac{1}{M} \biggr)D(P_{U|\mset}||Q_U)+\Psi, \label{eq:PACBayesian_bound_looser}
\end{align} 
where we have defined the quantity
\begin{align}
\Psi&=\frac{\sigma^2}{2}\biggl(\frac{\alpha^2}{\beta}+\frac{(1-\alpha)^2}{1-\beta} \biggr) +\frac{1}{N} \sum_{i=1}^{\beta N}\log \frac{P_{T}(T_i)}{P'_{T}(T_i)}+\frac{1}{N} \log \frac{2}{\delta} +\frac{\alpha}{\beta N} \sum_{i=1}^{\beta N} \frac{\delta_{T_i}^2}{2} +\frac{\alpha}{M} \log \frac{4\beta N}{\delta}\non \\&+\frac{1-\alpha}{(1-\beta) N} \sum_{i=\beta N +1}^{ N} \frac{\delta_{T_i}^2}{2}+\frac{1-\alpha}{M} \log \frac{4(1-\beta) N}{\delta}.
\end{align}
\end{corollary}
\begin{proof}
To obtain the required bound, we proceed as in the proof of Theorem~\ref{thm:PAC-Bayesianbound}. To bound the first difference of \eqref{eq:decomposition-2}, we use the upper bound in \eqref{eq:envlevl} with $\lambda=-N$. To bound the second difference, we use the upper bound in \eqref{eq:tasklevelineq} with $\lambda=-M$. Combining the resultant bounds via the union bound and rearranging results in \eqref{eq:PACBayesian_bound_looser}.
\end{proof}

\subsection{ Information Meta-Risk Minimization (IMRM) for Transfer Meta-Learning}\label{sec:IMRM}
For fixed base learner $P_{W|Z^M,U}$ and given prior $Q_{W|U}$ and hyper-prior $Q_U$ distributions, the PAC-Bayesian bound in \eqref{eq:PACBayesian_bound_looser} holds for any meta-learner $P_{U|\mset}$. Consequently, following the principle of information risk minimization \cite{zhang2006information}, one can design a meta-learner $P_{U|\mset}$ so as to minimize the upper bound \eqref{eq:PACBayesian_bound_looser} on the transfer meta-generalization loss. As compared to EMRM, this approach accounts  for the transfer meta-generalization gap, and can hence outperform EMRM in terms of meta-generalization performance. The same idea was explored in \cite{rothfuss2020pacoh} for conventional meta-learning, \ie, for the special case when $P_T=P'_T$.

To proceed, we  consider $\beta \in (0,1)$ and denote
\begin{align}
\Lscr(u,\mset)&=\Lscr_t(u|\mset)+\frac{\alpha}{\beta NM} \sum_{i=1}^{\beta N} D(P_{W|Z^M_i,U=u}||Q_{W|U=u})\non \\&+\frac{1-\alpha}{(1-\beta) NM} \sum_{i=\beta N+1}^{ N} D(P_{W|Z^M_i,U=u}||Q_{W|U=u})
\end{align}
as the meta-training loss regularized by the average KL divergence $ D(P_{W|Z^M_i,U=u}||Q_{W|U=u})$ between the base learner output and the prior distribution $Q_{W|U=u}$ over the base learner input data from source and target environments. The IMRM meta-learner is then defined as any algorithm that solves the optimization problem
\begin{align}
P^{{\rm IMRM}}_{U|\mset}=\arg \min_{P_{U|\mset} \in \Pscr(\Uscr)} \biggl(\Ebb_{P_{U|\mset}}[\Lscr(U,\mset)] +\biggl(\frac{1}{N}+\frac{1}{M} \biggr)D(P_{U|\mset}||Q_U)\biggr). \label{eq:optimizing_metalearner}
\end{align}
For fixed $N,M$, $Q_U,Q_{W|U}$ and base learner $P_{W|Z^M,U}$, the IMRM meta-learner can be  expressed as
\begin{align}
P^{{\rm IMRM}}_{U|\mset}(u)\propto Q_U(u)\exp\biggl(-\frac{NM}{N+M}\Lscr(u,\mset) \biggr), \label{eq:Gibbsmetalearner}
\end{align} where the normalization constant is given by $\Ebb_{Q_U}\Bigl[\exp\Bigl(-NM\Lscr(U;\mset)/(N+M) \Bigr) \Bigr]$.

For a given meta-training set, EMRM outputs the single value of the hyperparameter $u \in \Uscr$ that minimizes the meta-training loss \eqref{eq:tfr_metatrainingloss}. In contrast, the IMRM meta-learner  \eqref{eq:Gibbsmetalearner} updates the prior belief $Q_U$ after observing meta-training set, producing a distribution in the hyperparameter space. Given the significance of the meta-learning criterion \eqref{eq:optimizing_metalearner} as an upper bound on the transfer meta-generalization loss, the optimizing distribution \eqref{eq:Gibbsmetalearner} captures the impact of the epistemic uncertainty related to the limited availability of the meta-training data.  In line with this discussion, it can be seen from \eqref{eq:optimizing_metalearner} that as $M, N \rightarrow \infty$ with $M/N$ equal to a constant, the IMRM meta-learner tends to EMRM. 

To implement the proposed IMRM meta-learner, we adopt one of the two approaches. The first, referred to as \textit{IMRM-mode}, selects a single hyperparameter centered at the mode of \eqref{eq:Gibbsmetalearner} as
\begin{align}
U^{{\rm IMRM-mode}}(\mset)= \arg \max_{u \in \Uscr} \hspace{0.2cm} Q_U(u)\exp\biggl(-\frac{NM}{N+M}\Lscr(u;\mset)\biggr).
\end{align} IMRM-mode is akin to Maximum A Posteriori (MAP) inference in conventional machine learning. Alternatively, we obtain one sample from the IMRM meta-learner \eqref{eq:Gibbsmetalearner} for use by the base learner and then average the obtained transfer meta-generalization loss as per definition \eqref{eq:avgtfr_metagengap}. This can be in practice done by using Monte Carlo methods such as Metropolis-Hastings or Langevin dynamics \cite{bishop2006pattern}.
As mentioned, this approach, referred to \textit{IMRM-Gibbs}, reduces to the EMRM in the limit as $M,N \rightarrow \infty $ when $M/N$ is a constant.

\section{Single-Draw Probability Bounds on Transfer Meta-Learning}\label{sec:singledraw}
So far, we have considered the performance of meta-learning procedures defined by a stochastic mapping $P_{U|\mset}$ on average over distributions $P_{U|\mset}$. As discussed in the context of IMRM, this implies that the performance metric of interest is to be evaluated by averaging over realizations of the hyperparameter $U \sim P_{U|\mset}$. It is, however, also of interest to quantify performance guarantees under the assumption that a single draw $U \sim P_{U|\mset}$ is fixed and used throughout. Similar single-draw bounds have been derived  for conventional learning in \cite{hellstrom2020generalization}. With this goal in mind,
in this section, we present novel single-draw probability bounds for transfer meta-learning. The bound takes the following form: With probability at least $1-\delta$, with $\delta \in (0,1)$, over $(T_{1:N},\mset,U)\sim P_{T_{1:N}}P_{\mset|T_{1:N}}P_{U|\mset}$, the transfer meta-generalization gap satisfies the bound
\begin{align}
\bigl | \Delta \Lscr'(U|\mset) \bigr | \leq \epsilon. \label{eq:singledraw_form}
\end{align}
Towards the evaluation of single-draw bounds of this form, we resort again to the decomposition \eqref{eq:decomposition-2} used to derive the PAC-Bayesian bound in Section~\ref{sec:PAC-Bayesian bound}. We use the following  \textit{mismatched information density}
\begin{align}
\jmath(U,\mset)=\log \frac{\dP_{U|\mset}(U|\mset)}{\dQ_{U}(U)},
\end{align} which quantifies the evidence for the hyperparameter $U$ to be generated according to the meta-learner $P_{U|\mset}$ based on meta-training set, rather than being generated according to the hyper-prior distribution $Q_U$.
Considering Assumption~\ref{assum:2} on loss functions and Assumption~\ref{assum:2a} on information densities then yield the following single-draw probability bound for transfer meta-learning.
\begin{theorem}\label{thm:singledraw_bound}
For a fixed base learner $P_{W|Z^M,U}$, let $Q_U \in \Pscr(\Uscr)$ be a hyper-prior distribution over the space of hyperparameters and $Q_{W|U=u} \in \Pscr(\Wscr)$ be a prior distribution over the space of model parameters for each $u \in \Uscr$ and $\beta \in (0,1)$. Then,
under Assumption~\ref{assum:2} and Assumption~\ref{assum:2a}, the following inequality holds uniformly for any meta-learner $P_{U|\mset}$ with probability at least $1-\delta$, $\delta \in (0,1)$, over $(T_{1:N},\mset,U)\sim P_{T_{1:N}}P_{\mset|T_{1:N}}P_{U|\mset}$
\begin{align}
\Bigl| \Delta \Lscr'(U|\mset)\Bigr| &\leq \sqrt{ 2 \sigma^2 \biggl( \frac{ \alpha^2 }{\beta N}+\frac{(1-\alpha)^2 }{(1-\beta) N}\biggr) \biggl( \sum_{i=1}^{\beta N}\log \frac{P_{T}(T_i)}{P'_T(T_i)}+\jmath(U,\mset)+ \log \frac{2}{\delta}\biggr)}\non\\
&+ \frac{\alpha}{\beta N}\sum_{i=1}^{\beta N}\sqrt{\frac{2 \delta_{T_i}^2}{M}\biggl(D(P_{W|Z^M_i,U}||Q_{W|U})+\jmath(U,\mset)+\log \frac{4\beta N}{\delta}\biggr) } \non \\
&+\frac{1-\alpha}{(1-\beta) N}\sum_{i=\beta N+1}^{ N}\sqrt{\frac{2 \delta_{T_i}^2}{M}\biggl(D(P_{W|Z^M_i,U}||Q_{W|U})+\jmath(U,\mset)+\log \frac{4(1-\beta)N}{\delta}\biggr) }. \label{eq:singledraw_bound}
\end{align}
\end{theorem}
\begin{proof}
See Appendix~\ref{app:singledraw_bound}.
\end{proof}

As in the PAC-Bayesian bound \eqref{eq:PACBayesian_bound1}, the upper bound  in \eqref{eq:singledraw_bound} comprises of three contributions: $(i)$ the environment-level generalization gap, which is captured by the meta-environment shift term $\log (P_T(T_i)/P'_T(T_i))$ and by the mismatched information density $\jmath(U,\mset)$, with the latter quantifying the sensitivity of the meta-learner $P_{U|\mset}$ to the meta-training set;  $(ii)$ the generalization within the task drawn from source environment, which is accounted for by the KL divergence $D(P_{W|Z^M_i,U}||Q_{W|U})$ quantifying the sensitivity of the base learner $P_{W|Z^M,U}$ to the training set $Z^M$ through its divergence with respect to the prior distribution $Q_{W|U}$, along with the mismatched information density $\jmath(U,\mset)$, and finally, $(iii)$ the generalization gap within the task data from target environment, which is similarly captured through the KL divergence $D(P_{W|Z^M_i,u}||Q_{W|U})$ and the mismatched information density $\jmath(U,\mset)$.

The bound in \eqref{eq:singledraw_bound} can be specialized to the case of conventional meta-learning as given in the following corollary, which appears also to be a novel result.
\begin{corollary}
Assume that the source and target task distributions coincide, \ie, $P_T=P'_T$, and $\alpha=\beta=1$. Then, under the setting of Theorem~\ref{thm:singledraw_bound}, the following bound holds with probability at least $1-\delta$, $\delta \in (0,1)$, over $(T_{1:N},\mset,U) \sim P_{T_{1:N}}P_{\mset|T_{1:N}}P_{U|\mset}$
\begin{align}
\biggl| \Delta \Lscr(U|\mset)\biggr| &\leq \sqrt{ \frac{2 \sigma^2}{N}  \biggl(\jmath(U,\mset)+ \log \frac{2}{\delta}\biggr)}\non\\
&+ \frac{1}{N}\sum_{i=1}^{ N}\sqrt{\frac{2 \delta_{T_i}^2}{M}\biggl(D(P_{W|Z^M_i,U}||Q_{W|U})+\jmath(U,\mset)+\log \frac{2 N}{\delta}\biggr) }.
\end{align}
\end{corollary}
\section{Example} \label{sec:example}
In this section, we 
consider the problem of estimating the mean of a Bernoulli process based on a few samples. To this end, we adopt a base learner based on biased regularization and meta-learn the bias as the hyperparameter \cite{denevi2020advantage}.
\subsection{Setting}
 The data distribution for each task is given as $P_{Z|T=\tau} \sim \Ber(\tau)$ for a task-specific mean parameter $\tau \in [0,1]$. For meta-training, we sample tasks from the source task distribution $\tau \sim P_{T}$ given by a beta distribution $\Beta(\tau;a,b)$ with shape parameters $a,b>0$, while
the target task distribution $\tau \sim P'_T$ encountered during meta-testing is $\Beta(\tau;a',b')$ with generally different shape parameters $a',b'>0$. We recall that the mean of a random variable $\tau \sim {\rm Beta}(\tau;a,b)$ is given as $R(a,b)=a/(a+b)$ and the variance is $V(a,b)=ab/((a+b)^2(a+b+1))$. For any task $\tau$, the base learner uses training data, distributed  i.i.d. from $\Ber(\tau)$, to determine the model parameter $W$, which is used as a predictor of a new observation $Z \sim \Ber(\tau)$ at test time. The loss function  $l(w,z)=(w-z)^2$ measures the quadratic error between prediction and the test input $z$.

 The base learner adopts a quadratic regularizer with bias given by a hyperparameter $u \in [0,1]$ \cite{denevi2020advantage}, and randomizes its output. Accordingly, the base learner computes
the empirical average $
D_i=\frac{1}{M}\sum_{j=1}^{M} Z^M_{i,j},
$ over the training set, where $Z^M_{i,j}$ denotes the $j$th data sample in the training set of $i$th task. Then, it computes the convex combination
$
R_i(u)=\gamma D_i +(1-\gamma)u,
$
with the hyperparameter $u \in [0,1]$, where $\gamma \in [0,1]$ is a fixed scalar. Finally, it outputs a random model parameter $W$ with mean $R_i(u)$ by drawing $W$ as \begin{align}
P_{W|Z^M_i,U=u}(w)= {\rm Beta}(w;c R_i(u),c(1-R_i(u))),
\label{eq:example_phi}
\end{align}
where $c>0$ is fixed and it determines the variance $V_i(u):=V(cR_i(u),c (1-R_i(u)))$ of the output of the base learner.


The meta-training loss \eqref{eq:tfr_metatrainingloss} can be directly computed as
\begin{align}
\Lscr_t(u|\mset)&=\frac{\alpha}{\beta N } \sum_{i=1}^{\beta N} \Bigl(V_i(u)+R_i(u)^2-2R_i(u)D_i+\sum_{j=1}^M \frac{1}{M} Z_{i,j}^2 \Bigr)\non \\&+\frac{(1-\alpha)}{(1-\beta) N}\sum_{i=\beta N+1}^{N} \Bigl(V_i(u)+R_i(u)^2-2R_i(u)D_i+\sum_{j=1}^M \frac{1}{M} Z_{i,j}^2 \Bigr),
\end{align}
while the transfer meta-generalization loss \eqref{eq:tfr_metatestloss} evaluates as
\begin{align}
\Lscr'_g(u)&=u(1-\gamma) \biggl( \frac{1}{c+1}+u(1-\gamma)\frac{c}{c+1}+2\gamma R'\frac{c}{c+1}-2R'\biggr)+\frac{\gamma R'}{c+1}\non \\&+\frac{\gamma^2c}{c+1}\biggl(\frac{R'}{M}+(V'+R'^2)\biggl(1-\frac{1}{M}\biggr)\biggr)-2\gamma(V'+R'^2)+R',
\end{align}
where $V'=V(a',b')$ is the variance and $R'=R(a',b')$ is the mean of the random variable $\tau \sim P'_T$.

\subsection{Experiments}
For the base learner as described above, we analyze the average  transfer meta-generalization gap $\Ebb_{P_{\mset}P_{U|\mset}}[\Delta \Lscr'(U|\mset)]$ in \eqref{eq:avgtfr_metagengap}  under EMRM \eqref{eq:EMRM} and IMRM \eqref{eq:Gibbsmetalearner}, as well as the average excess meta-risk
$\Ebb_{P_{U}}[\Lscr_g'(U)]-\min_{u \in [0,1]}\Lscr'_g(u)$.  For IMRM, we consider a prior distribution $Q_{W|U=u}(w)  ={\rm \Beta}(w;cu,c-cu) $ in the space of model parameters and a hyper-prior distribution $Q_{U}(u)  ={\rm \Beta}(u;1.8,2.5)$ in the space of hyperparameters.
The prior distribution $Q_{W|U}$ indicates that, in the absence of data, the base learner should select a model parameter with mean equal to the hyperparameter $u$. For the IMRM, we consider both IMRM-mode and IMRM-Gibbs.

\begin{figure}[h!]
 \centering 
   \includegraphics[scale=0.5,trim=1in 1in 1in 1.2in,clip=true]{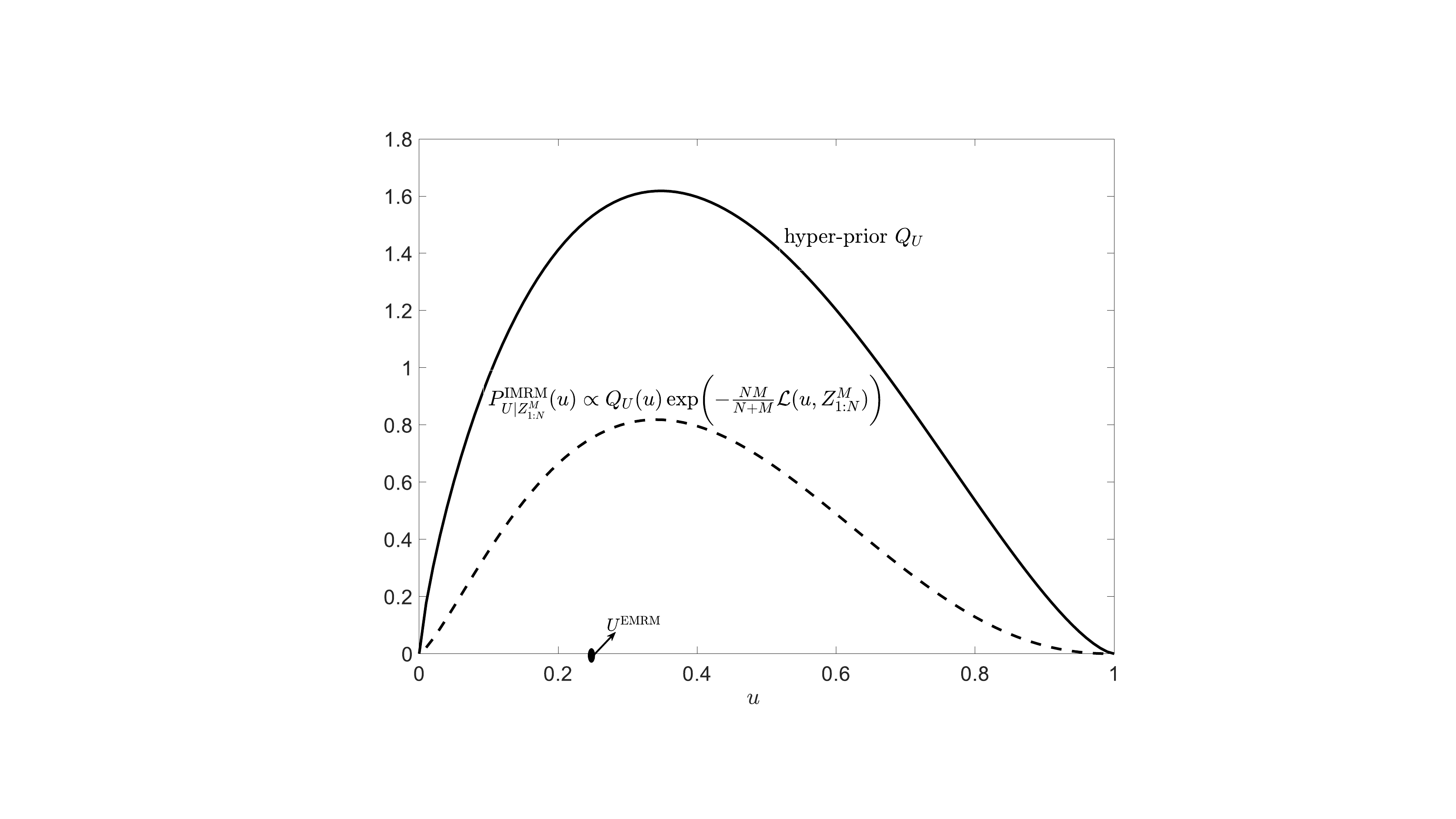} 
   \caption{Hyper-prior distribution $Q_U$, IMRM hyper-posterior $P^{{\rm IMRM}}_{U|\mset}$ in \eqref{eq:Gibbsmetalearner}, and EMRM solution \eqref{eq:tfr_metatrainingloss} for a given meta-training set $\mset$. ($M=10$ , $N=8$, $a=1.5$, $b=7.5$, $a'=4$, $b'=5$, $c=5$, $\alpha=0.1$, $\beta=0.6$, $\gamma=0.55$)  } \label{fig:exp7}
  \end{figure}
To start, in Figure~\ref{fig:exp7}, we illustrate the hyper-prior $Q_U(u)$, the IMRM hyper posterior $P^{{\rm IMRM}}_{U|\mset}$ in \eqref{eq:Gibbsmetalearner}, and the output of EMRM \eqref{eq:EMRM}. It is observed that, for the given values of $M=10$ and $N=8$ and for the given hyper-prior, the IMRM hyper-posterior retains information about the residual uncertainty on the value of the hyperparameter $u$, which is instead reduced to a point estimate based solely on meta-training data by EMRM.

In Figure~\ref{fig:exp5}, we then analyze the performance of EMRM, IMRM-mode and IMRM-Gibbs as a function of increasing values of $M$ and $N$, for a fixed ratio $M/N=0.85$, where $a=1.5,b=7.5$, $a'=4$, $b'=5$, $\gamma=0.55$, $\alpha=0.48$, $\beta=0.48$ and $c=5$. It can be seen that while EMRM yields, by definition, the smallest meta-training loss, IMRM improves the average transfer meta-generalization loss (Figure~\ref{fig:exp5}$(a)$) by decreasing the average transfer meta-generalization gap (Figure~\ref{fig:exp5}$(b)$). 
This gain is more significant for sufficiently small values of $M$ and $N$, since, as $M$ and $N$ increases, IMRM tends to EMRM.
 We also observe that there exists a non-vanishing generalization gap even at high values of $M$ and $N$. As discussed in Section~\ref{sec:EMRM_avgtfrgap}, this is caused by the meta-environment shift from $P_T$ to $P'_T$. Finally, IMRM-mode and IMRM-Gibbs are seen to perform similarly, with the former being generally advantageous in this example. This suggest that the main advantage of IMRM is due to the meta-regularizing effect of the KL term in \eqref{eq:optimizing_metalearner}. In the following two experiments, we adopt IMRM-mode.

\begin{figure}[h!]
 \centering 
%
   \subfloat[]{
   \includegraphics[scale=0.2,trim=3.2in 0.9in 3.2in 1in,clip=true,width=0.5\textwidth]{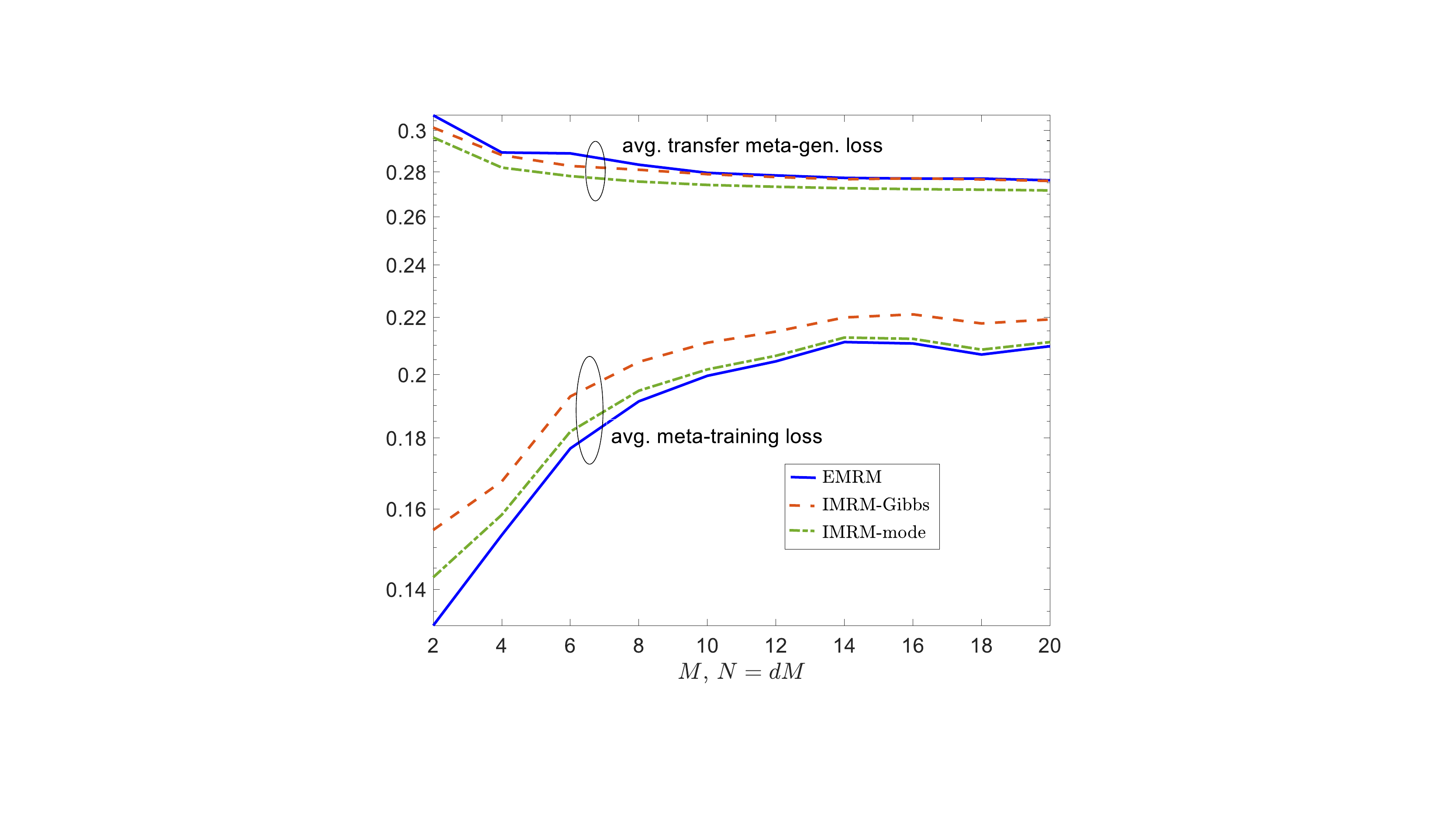}
   \label{fig:exp10_fig1}}
    \subfloat[]{
   \includegraphics[scale=0.18,trim=3.1in 0.9in 3.2in 1in,clip=true,width=0.5\textwidth]{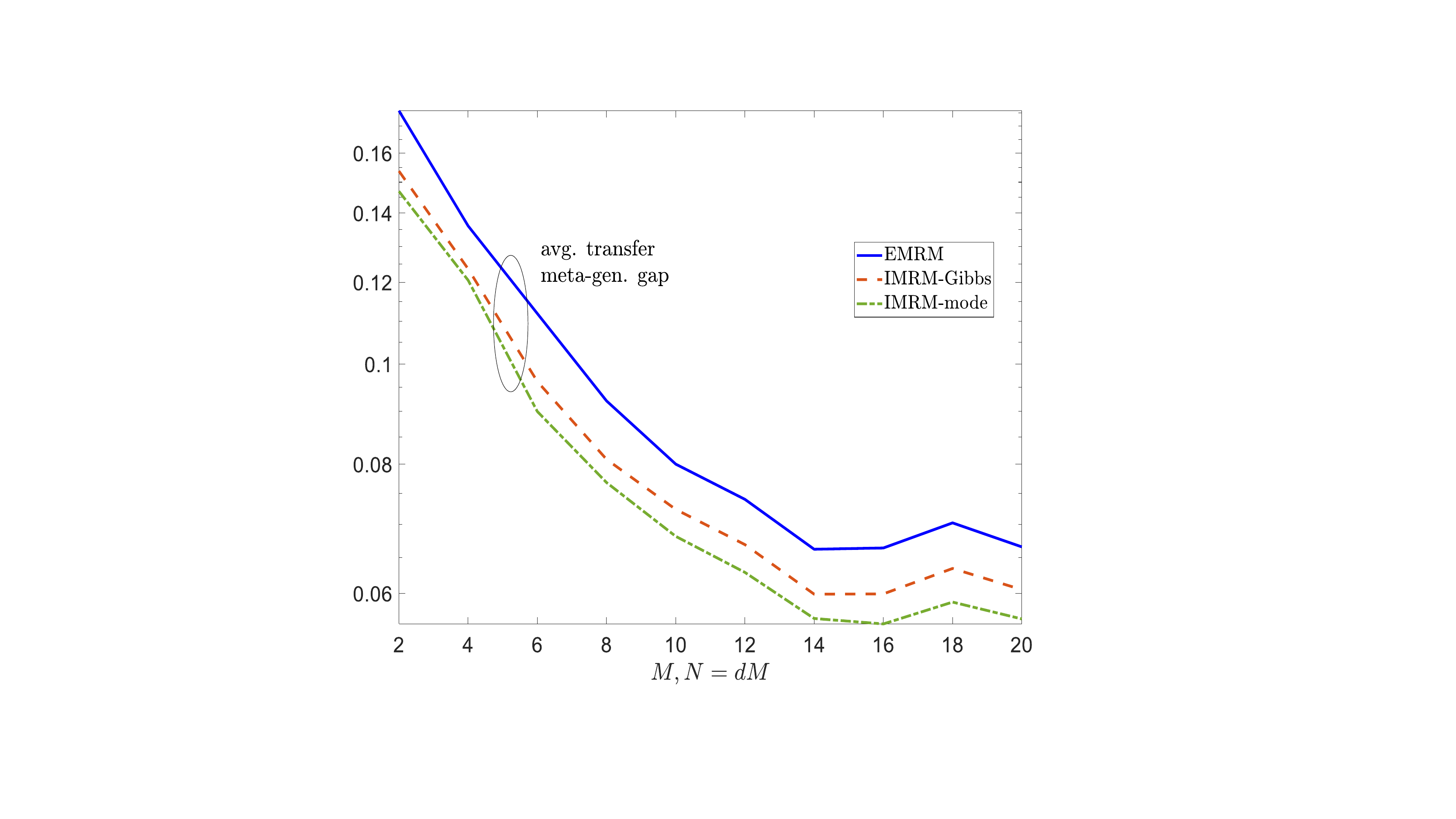}\label{fig:exp10_fig2}} 
    \hspace{0.2cm}
   %
   \caption{Average losses under EMRM, IMRM-mode and IMRM-Gibbs against increasing values of $M$ and $N=dM$ with $d=1/0.85$ for $a=1.5,b=7.5$, $a'=4$, $b'=5$, $\gamma=0.55$, $\alpha=0.48$, $\beta=0.48$ and $c=5$: $(a)$ average transfer meta-generalization loss \eqref{eq:tfr_metatestloss} and average meta-training loss \eqref{eq:tfr_metatrainingloss}, and $(b)$ average transfer meta-generalization gap \eqref{eq:avgtfr_metagengap}.} \label{fig:exp5}
   \end{figure}
   
   
   \begin{figure}[h!]
 \centering 
 \subfloat[]{
   \includegraphics[scale=0.5,trim=3in 0.6in 3in 0.7in,clip=true,width=0.4\textwidth]{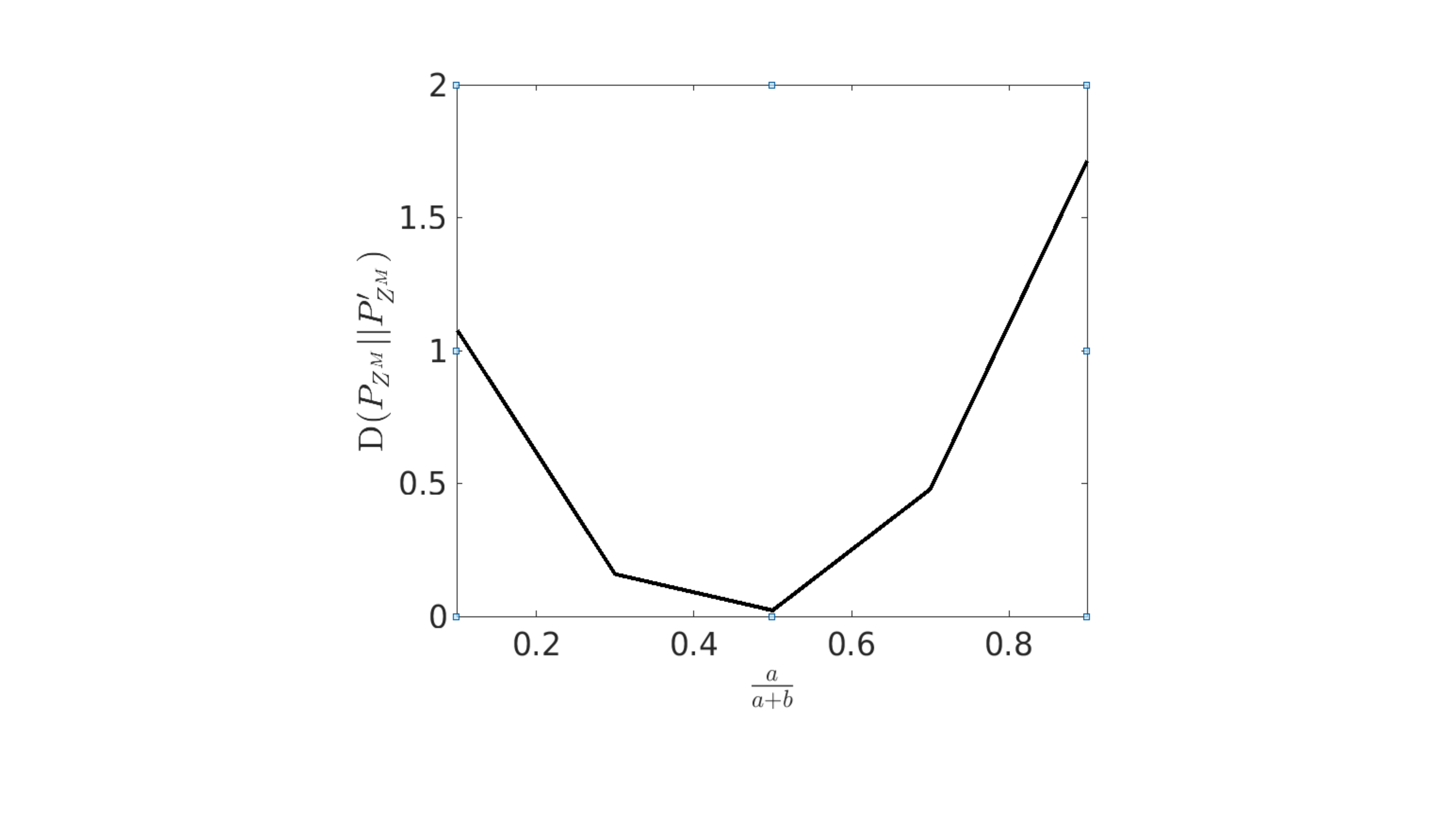}  \label{fig:subfigure2_exp4}}
 \subfloat[]{
   \includegraphics[scale=0.48,trim=2.4in 0in 2.8in 0.1in,clip=true,width=0.6\textwidth]{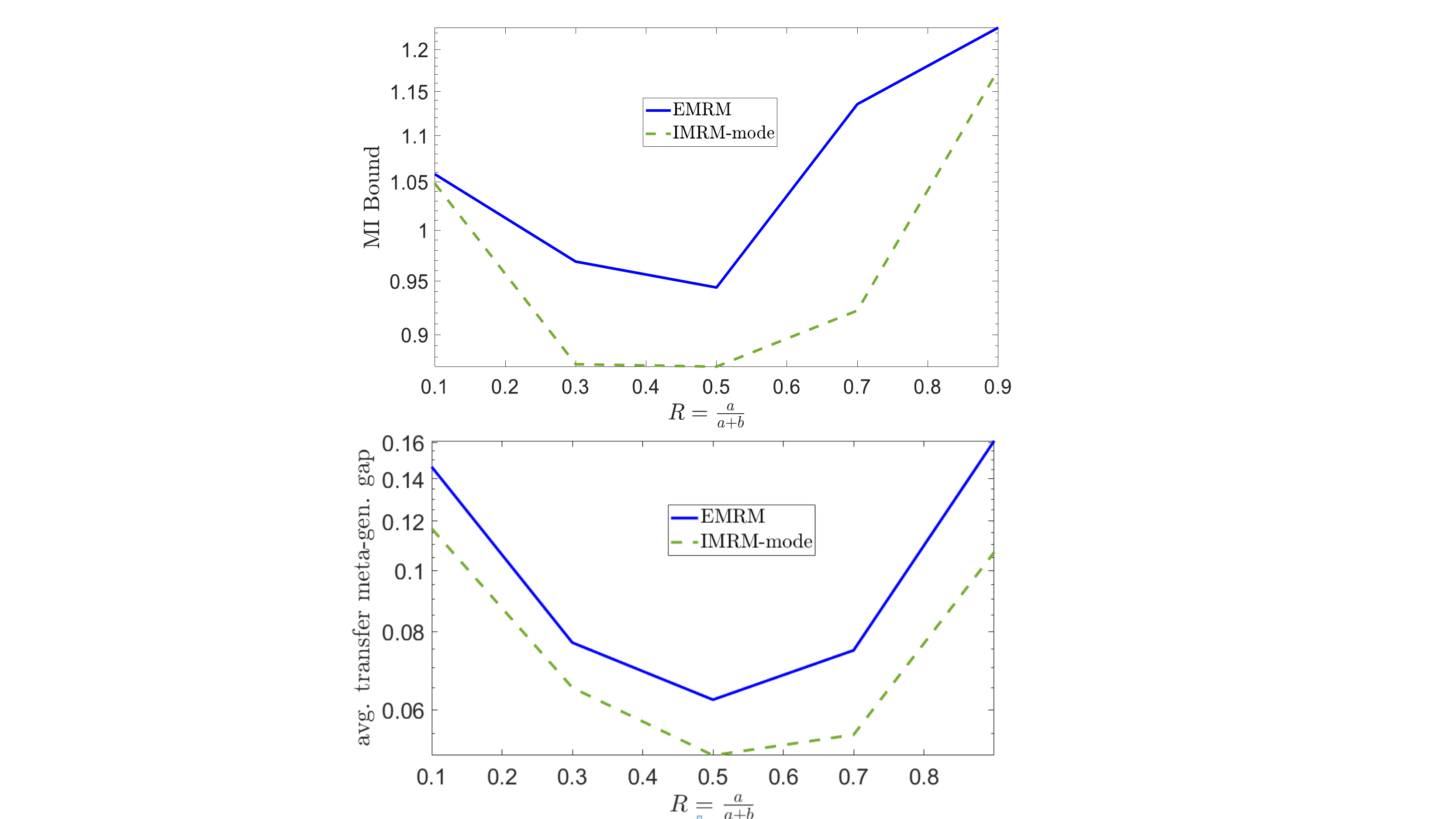}  }
 
   \caption{Impact of meta-environment shift when $P'_T$ is fixed to ${\rm Beta}( a'=4, b'=5)$ and $P_T$ varies as ${\rm Beta} (a=9-b,b=9(1-R))$, $R=a/(a+b)$: $(a)$ the KL divergence between $P_{Z^M}$ and $P'_{Z^M}$; and $(b)$ (top)  MI bound on the average transfer meta-generalization gap \eqref{eq:avgtfrgap_bound1_ITMI}; (bottom) average transfer meta-generalization gap for EMRM and IMRM-mode ($\gamma=0.55$, $\alpha=0.6$, $\beta=0.6$, $N=10$, $M=5$ and $c=5$). } \label{fig:exp4}
   \end{figure}
  Figure~\ref{fig:exp4} studies the impact of the meta-environment shift when the target task distribution $P'_T$ is fixed to ${\rm Beta}( a'=4, b'=5)$ and the source task distribution $P_T$ is given as ${\rm Beta} (a=9-b,b=9(1-R))$ with a varying mean $R=a/(a+b)$. Other parameters are set as $\gamma=0.55$, $\alpha=0.6$, $\beta=0.6$, $N=10$, $m=5$ and $c=5$.
  The analysis in Section~\ref{sec:EMRM_avgtfrgap} revealed that the KL divergence $D(P_{Z^M}||P'_{Z^M})$ between the data distributions under source and target environments is a key quantity in bounding the average transfer meta-generalization gap. The numerical results in the figure confirm that average transfer meta-generalization gap \eqref{eq:avgtfrgap_bound1_ITMI} for EMRM and IMRM-mode also shows a similar trend as the KL divergence as we vary the degree of meta-environment shift: 
  The gap is small when $P_T$ and $P'_T$ are similar in term of KL divergence, and it increases when the divergence grows.
  \begin{figure}[h!]
 \centering 
   \includegraphics[scale=0.5,trim=0.2in 1in 1in 1in,clip=true]{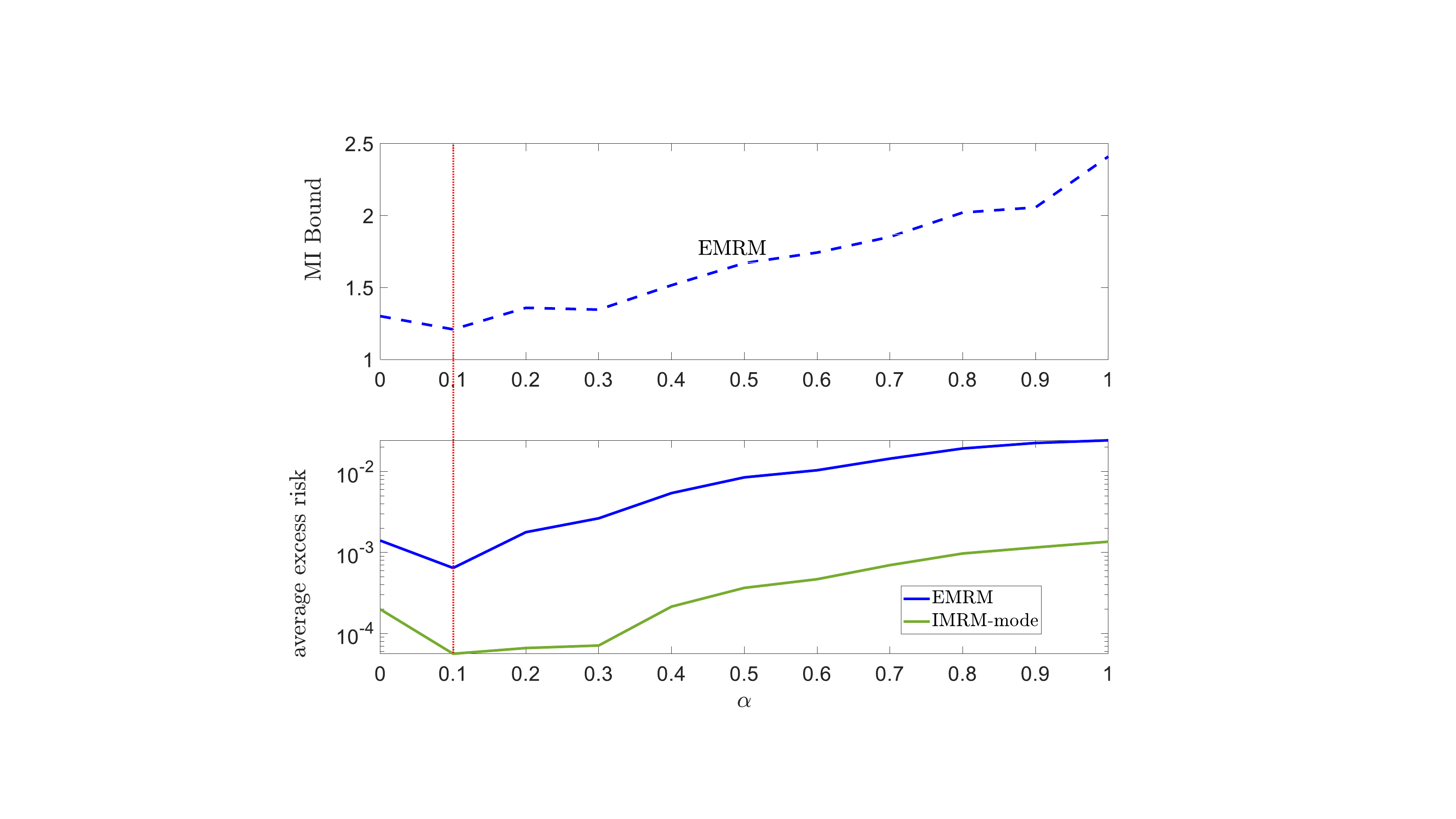}  \label{fig:subfigure1_exp6}
   \caption{Average transfer excess meta-risk as a function of the parameter $\alpha$ used in the definition \eqref{eq:tfr_metatrainingloss} of the weighted meta-training loss: (top) MI-based bound on the average transfer excess meta-risk \eqref{eq:transferrisk_bound} for EMRM; (bottom) average excess transfer meta-risk for EMRM and IMRM-mode  ($a=1.67,b=8.3$, $a'=4.45,b'=5.55$, $\gamma=0.55$, $\beta=0.4$, $ N=23$, $M=15$ and $c=5$). } \label{fig:exp6}
   \end{figure}
   
The average transfer excess meta-risk of EMRM and IMRM-mode are considered in Figure~\ref{fig:exp6} as a function of the parameter $\alpha$ used in the definition \eqref{eq:tfr_metatrainingloss} of the weighted meta-training loss.
The choice of $\alpha$ that minimizes the average transfer excess meta-risk is seen to generally lie somewhere between the extreme points $\alpha=0$, which prescribes the use of only target environment data, or $\alpha=1$, corresponding to the exclusive use of source environment datasets. Furthermore, the analytical bound \eqref{eq:transferrisk_bound} for EMRM (top figure) is  seen to accurately predict the optimal value of $\alpha$ obtained from the actual average transfer excess meta-risk \eqref{eq:transfer_metaexcessrisk} (bottom figure).
 We note that it would also be interesting to derive similar analytical upper bound on the average transfer excess meta-risk for IMRM, by following the methodologies of papers such as \cite{raginsky2017non,kuzborskij2019distribution}.
  
Finally, in Figure~\ref{fig:exp12}, we evaluate the \textit{single-draw} probability bounds obtained in \eqref{eq:singledraw_bound} for IMRM-Gibbs. Note that the single-draw performance of EMRM coincides with its average performance since it is deterministic. 
To illustrate the single-draw scenario, for each meta-training set of $N$ tasks, we generate samples $U$ of the hyperparameter according to $P_{U|\mset}^{\rm IMRM}$. We then compute the transfer meta-generalization gap  $\Delta \Lscr'(u|\mset)$ for each of the generated samples.
In the bottom panel of Figure~\ref{fig:exp12}, we use a box plot to illustrate the obtained empirical distribution of the transfer meta-generalization gap $\Delta \Lscr'(U|\mset)$ with $U \sim P_{U|\mset}^{\rm IMRM}$ for increasing values of $N$ and fixed $M=5$. For each value of $N$, the top of the box represents the $25$th percentile ($\delta=0.25$), the bottom corresponds to the $75$th percentile ($\delta=0.75$) and the centre dash correspond to the median ($\delta=0.5$) of the distribution of $\Delta \Lscr'(U|\mset)$. The two lines outside the box are the ``whiskers'' that indicate the support of the empirical distribution. The information-density based single-draw upper bound \eqref{eq:singledraw_bound}
 is illustrated for comparison in the top panel of Figure~\ref{fig:exp12} for $\delta=0.25,0.5,$ and $0.75$. It can be seen that the bounds exhibit a similar decreasing trend as the empirical transfer meta-generalization gap in the bottom panel.
  \begin{figure}[h!]
 \centering 
   \includegraphics[scale=0.74,trim=2in 1.6in 2.5in 1.8in,clip=true]{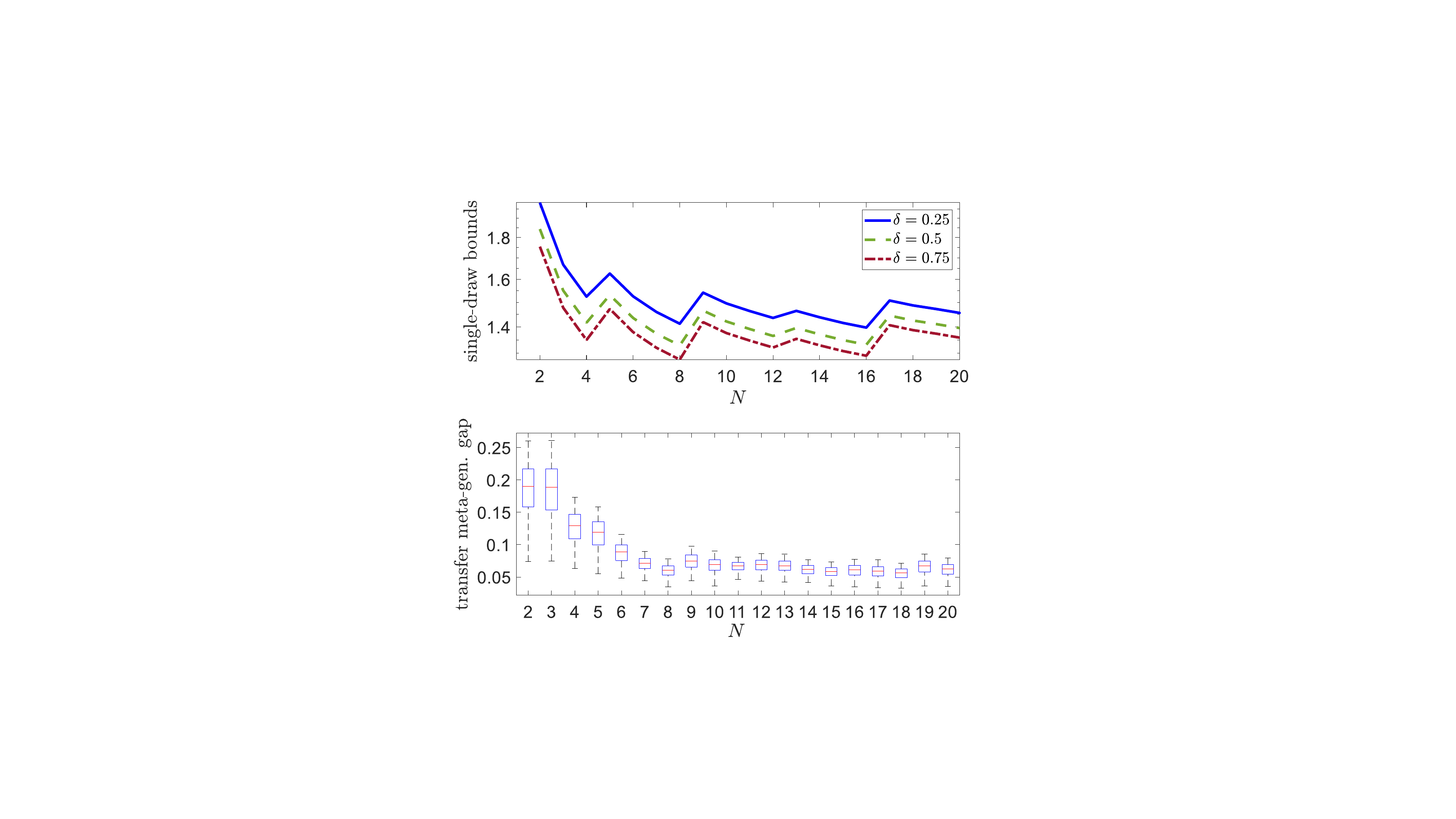}  \label{fig:subfigure1_exp12}
   \caption{Analysis of single-draw probability bounds for IMRM-Gibbs as a function of the number of tasks $N$ for $\delta=0.25,0.5$ and $0.75$. The top panel illustrates the  single-draw bound \eqref{eq:singledraw_bound} on transfer meta-generalization gap, while the bottom panel shows a box plot of the numerical evaluation of transfer meta-generalization gap. The lower quantile ($\delta=0.25$) correspond to the top of the box and the upper quantile $(\delta=0.75)$ correspond to the bottom of the box, while the circled dot in the middle of the box indicates the median $(\delta=0.5)$ ($a=1.5,b=7.5$, $a'=4,b'=5$, $\gamma=0.55$, $\beta=\alpha=0.25$,  $M=5$ and $c=5$). } \label{fig:exp12}
   \end{figure}
\section{Conclusions}
This paper introduced the problem of transfer meta-learning, in which the meta-learner observes data from tasks belonging to a source task environment, while its performance is evaluated on a new meta-test task drawn from the target task environment. We obtained three forms of upper bounds on the transfer meta-generalization gap -- bounds on average generalization gap, high-probability PAC-Bayesian bounds and high-probability single-draw bounds.
These bounds capture the meta-environment shift between source and target task distributions via the KL divergence between source and target data distributions for the average generalization gap bound, and the log-likelihood ratio between the source and target task distributions for the PAC-Bayesian and single-draw bounds. We note that these metrics can be numerically estimated from finite per-task data sets via various parametric or non-parametric methods \cite{sugiyama2012density}. 
Furthermore, we leveraged the derived PAC-Bayesian bound to propose a new meta-learning algorithm for transfer meta-learning, IMRM, which was shown in experiments to outperform an empirical weighted meta-risk minimization algorithm.

Directions for future work include the development of larger-scale experiments for linear and non-linear base learners, the application of the bounding methodologies of \cite{negrea2019information}, \cite{steinke2020reasoning} and the analysis of the excess risk for IMRM by adapting the tools of \cite{raginsky2017non,kuzborskij2019distribution}. It would also be interesting to analyze bounds on transfer meta-generalization gap that capture the meta-environment shift via other statistical divergences like Jensen-Shannon divergences \cite{jose2020informationtheoretic}.
\appendices
\section{Proofs of Lemma~\ref{lem:expinequality_avgtfrgap_ITMI} and Lemma~\ref{lem:expinequality_avgtfrgap_ITMI_1}}\label{app:expinequality_avgtfrgap_ITMI}
Throughout the Appendices, we use the notation $P_{W|\tau}$ to denote the distribution $P_{W|T=\tau}$, $P_{Z|\tau}$ to denote $P_{Z|T=\tau}$ and $P_{W|Z^M,u}$ to denote $P_{W|Z^M,U=u}$.
Under Assumption~\ref{assum:1}$(a)$, the following inequality holds for each task $\tau \in \mathcal{T}$,
\begin{align}
\Ebb_{P_{W|\tau}P_{Z_j|\tau}}\biggl[\exp \biggl( \lambda( l(W,Z_j)- \Ebb_{P_{W|\tau}P_{Z_j|\tau}}[l(W,Z)] -\frac{\lambda^2\delta_{\tau}^2}{2} \biggr) \biggr] \leq 1,
\end{align}
which in turn implies that
\begin{align}
\Ebb_{P_{W|\tau}P_{Z_j|\tau}}\biggl[\Ibb_{\Escr}\exp \biggl( \lambda( l(W,Z)- \Ebb_{P_{W|\tau}P_{Z_j|\tau}}[l(W,Z)] -\frac{\lambda^2\delta_{\tau}^2}{2} \biggr) \biggr] \leq 1,
\end{align}where $\Escr={\rm supp}(P_{W,Z_j|\tau})$. Subsequently, using a change of measure from $P_{W|\tau}P_{Z_j|\tau}$ to $P_{W,Z_j|\tau}$ as in  \cite[Prop. 17.1]{polyanskiy2014lecture} then yield the inequality \eqref{eq:expinequality_avgtfr_task_ITMI}.

Under Assumption~\ref{assum:1}$(b)$, the following inequality holds for $i=1,\hdots, N$,
\begin{align}
\Ebb_{P_{U}P'_{Z^M_i}}\biggl[\exp \biggl( \lambda( L_t(U|Z^M_i)- \Ebb_{P_{U}P'_{Z^M_i}}[ L_t(U|Z^M_i)] -\frac{\lambda^2\sigma ^2}{2} \biggr) \biggr] \leq 1. \label{eq:ineq_proof1}
\end{align} To get to \eqref{eq:expinequality_avgtfr_env_ineq1_ITMI}, we note that \eqref{eq:ineq_proof1} implies the following inequality for $i=\beta N+1,\hdots, N$,
\begin{align}
\Ebb_{P_{U}P'_{Z^M_i}}\biggl[\Ibb_{\Escr_1}\exp \biggl( \lambda( L_t(U|Z^M_i)- \Ebb_{P_{U}P'_{Z^M_i}}[ L_t(U|Z^M_i)] -\frac{\lambda^2\sigma ^2}{2} \biggr) \biggr] \leq 1, 
\end{align} where $\Escr_1={\rm supp}(P_{U|Z^M_i}P'_{Z^M_i})$. Applying change of measure as before from $P'_{Z^M_i}P_{U}$ to $P'_{Z^M_i} P_{U|Z^M_i}$ then yields inequality \eqref{eq:expinequality_avgtfr_env_ineq1_ITMI}.

To get to \eqref{eq:expinequality_avgtfr_env_ineq2_ITMI}, we start from \eqref{eq:ineq_proof1}, which implies for $i=1,\hdots, \beta N$
\begin{align}
\Ebb_{P_{U}P'_{Z^M_i}}\biggl[\Ibb_{\Escr_2}\exp \biggl( \lambda( L_t(U|Z^M_i)- \Ebb_{P_{U}P'_{Z^M_i}}[ L_t(U|Z^M_i)] -\frac{\lambda^2\sigma ^2}{2} \biggr) \biggr] \leq 1, 
\end{align}where $\Escr_2={\rm supp}(P_{Z^M_i})$. Performing change of measure from $P'_{Z^M_i}$ to $ P_{Z^M_i}$ then gives that
\begin{align}
\Ebb_{P_{U}P_{Z^M_i}}\biggl[\exp \biggl( \lambda( L_t(U|Z^M_i)- \Ebb_{P_{U}P'_{Z^M_i}}[ L_t(U|Z^M_i)] -\frac{\lambda^2\sigma ^2}{2} -\log \frac{\dP_{Z^M_i}(Z^M_i)}{\dP'_{Z^M_i}(Z^M_i)} \biggr) \biggr] \leq 1. 
\end{align}
Applying the change of measure again from $ P_{Z^M_i}P_U$ to $P_{Z^M_i}P_{U|Z^M_i}$ then yields \eqref{eq:expinequality_avgtfr_env_ineq2_ITMI}.
\section{Proof of Theorem~\ref{thm:avgtfrgap_bound1_ITMI}}\label{app:avgtfrgap_bound1_ITMI}
To obtain the required upper bound on $|\Ebb_{P_{\mset}P_{U|\mset}}[\Delta \Lscr'(U|\mset)]|$, we leverage the decomposition in \eqref{eq:decomposition-1}. Using triangle inequality, it then follows that
\begin{align}
&|\Ebb_{P_{\mset}P_{U|\mset}}[\Delta \Lscr'(U|\mset)]|\non \\ & \leq |\Ebb_{P_{U}}[\Lscr'_g(U)-\Lscr'_{g,t}(U)| +|\Ebb_{P_{\mset}P_{U|\mset}}[\Lscr'_{g,t}(U)-\Lscr_t(U|\mset)|. \label{eq:decomp_ineq1_ITMI}
\end{align}
The idea is to separately bound the two averages in \eqref{eq:decomp_ineq1_ITMI}. Towards this, we first consider the average difference $|\Ebb_{P_{U}}[\Lscr'_g(U)-\Lscr'_{g,t}(U)|$ which can be equivalently written as
\begin{align}
&|\Ebb_{P'_{T} P^M_{Z|T}}\Ebb_{P_U P_{W|Z^M,U}}[L_g(W|T)-L_{t}(W|Z^M)|\non \\
 &\leq \Ebb_{P'_{T}}|\Ebb_{ P^M_{Z|T}}\Ebb_{P_{W|Z^M}}[L_g(W|T)-L_{t}(W|Z^M)|  \\
 &\leq \Ebb_{P'_{T}}\biggl[\frac{1}{M} \sum_{j=1}^M \biggl| \Ebb_{P_{W|T}P_{Z_j|T}}[l(W,Z_j)]-\Ebb_{P_{W,Z_j|T}}[l(W,Z_j)]\biggr| \biggr].\label{eq:decomp_ineq2_ITMI}
\end{align}We now bound the difference $\Ebb_{P_{W|T}P_{Z_j|T}}[l(W,Z_j)]-\Ebb_{P_{W,Z_j|T}}[l(W,Z_j)]$ using \eqref{eq:expinequality_avgtfr_task_ITMI}. For $T=\tau$, applying Jensen's inequality on \eqref{eq:expinequality_avgtfr_task_ITMI} and  taking log on both sides of the resultant inequality gives  that
\begin{align}
\lambda \biggl(\Ebb_{P_{W,Z_j|T=\tau}}[l(W,Z_j)]-\Ebb_{P_{W|T=\tau}P_{Z_j|T=\tau}}[l(W,Z)] \biggr)\leq \frac{\lambda^2 \delta_{\tau}^2}{2}+ I(W;Z_j|T=\tau). \label{eq:point1}
\end{align}
Choosing $\lambda=\sqrt{2 I(W;Z_j|T=\tau)}/\delta_{\tau}$ then yields that
\begin{align}
[|\Ebb_{P_{W,Z_j|T=\tau}}[l(W,Z_j)]-\Ebb_{P_{W|T=\tau}P_{Z_j|T=\tau}}[l(W,Z_j)] | \leq \sqrt{2 \delta_{\tau}^2I(W;Z_j|T=\tau)}.
\end{align}
Substituting back in \eqref{eq:decomp_ineq2_ITMI}, averaging over $T$, then yields the following upper bound
\begin{align}
\Ebb_{P'_{T}}|\Ebb_{P^M_{Z|T}}\Ebb_{P_{W|Z^M}}[L_g(W|T)-L_{t}(W|Z^M)]| \leq \Ebb_{P'_T}\biggl[\frac{1}{M} \sum_{j=1}^M \sqrt{2 \delta_{T}^2I(W;Z_j|T=\tau)}\biggr]. \label{eq:b}
\end{align}

We now bound the second average difference in \eqref{eq:decomp_ineq1_ITMI} using the
 the exponential inequalities \eqref{eq:expinequality_avgtfr_env_ineq1_ITMI}--\eqref{eq:expinequality_avgtfr_env_ineq2_ITMI}. Towards this, we denote by $P_{Z^M_{1:\beta N}}$  the marginal of the joint distribution $\prod_{i=1}^{\beta N}P_{T_i}P^M_{Z|T_i}$ and by $P'_{Z^M_{\beta N+1:N}}$ the marginal of the joint distribution $\prod_{i=\beta N+1}^{ N}P'_{T_i}P^M_{Z|T_i} $. We will also use \[\Lscr_t(u|Z^M_{1:\beta N})=\frac{1}{\beta N}\sum_{i=1}^{\beta N}L_t(u|Z^M_i)\] for the  the meta-training loss on task data from source environment and \[\Lscr_t(u|Z^M_{\beta N+1:N})=\frac{1}{(1-\beta)N}\sum_{i=\beta N+1}^{ N}L_t(u|Z^M_i)\] for the meta-training loss on task data from target environment. Then, the second average difference in \eqref{eq:decomp_ineq1_ITMI} can be equivalently written as
\begin{align}
&|\Ebb_{P_{\mset,U}}[\Lscr'_{g,t}(U)-\Lscr_t(U|\mset)|\non\\
&=\biggl|\Ebb_{P_{Z^M_{1:\beta N}}P'_{Z^M_{\beta N+1:N}}P_{U|\mset}}\biggl[\alpha\biggl(\Lscr'_{g,t}(U)-\Lscr_t(U|Z^M_{1:\beta N})\biggr)+(1-\alpha)\biggl(\Lscr'_{g,t}(U)-\Lscr_t(U|Z^M_{\beta N+1:N})\biggr)\biggr]\biggr|\non\\
& \leq \alpha |\Ebb_{P_{Z^M_{1:\beta N}}P_{U|Z^M_{1:\beta N}}}[\Lscr'_{g,t}(U)-\Lscr_t(U|Z^M_{1:\beta N}) ]| \non \\&+(1-\alpha)|\Ebb_{P'_{Z^M_{\beta N+1:N}}P_{U|Z^M_{\beta N+1:N}}}[\Lscr'_{g,t}(U)-\Lscr_{t}'(U|Z^M_{\beta N+1:N}) ]| \non \\
&=  \alpha \biggl| \frac{1}{\beta N} \sum_{i=1}^{\beta N} \biggl(\Ebb_{P_{U}P'_{Z^M_i}}[L_t(U|Z^M_i)]-\Ebb_{P_{Z^M_i}P_{U|Z^M_i}}[L_t(U|Z^M_i) ]\biggr) \biggr| \non \\&+(1-\alpha)\biggl| \frac{1}{(1-\beta) N} \sum_{i=\beta N+1}^{N} \biggl(\Ebb_{P_{U}P'_{Z^M_i}}[L_t(U|Z^M_i)]-\Ebb_{P'_{Z^M_i}P_{U|Z^M_i}}[L_t(U|Z^M_i) ]\biggr) \biggr| \label{eq:env1}.
\end{align}
We now proceed to use the exponential inequalities in \eqref{eq:expinequality_avgtfr_env_ineq1_ITMI} and \eqref{eq:expinequality_avgtfr_env_ineq2_ITMI} to bound the two terms in \eqref{eq:env1}. To bound the first difference, we use \eqref{eq:expinequality_avgtfr_env_ineq2_ITMI}. Applying Jensen's inequality and taking log on both sides of the resulting inequality yields 
\begin{align}
\lambda \biggl(\Ebb_{P_{Z^M_i}P_{U|Z^M_i}}[L_t(U|Z^M_i) ]-\Ebb_{P_{U}P'_{Z^M_i}}[L_t(U|Z^M_i)]\biggr) \leq \frac{\lambda^2 \sigma^2}{2}+ D(P_{Z^M}||P'_{Z^M})+I(U;Z^M_i). \label{eq:point2}
\end{align}
Further, choosing $\lambda=\sqrt{2(D(P_{Z^M}||P'_{Z^M})+I(U;Z^M_i))}/\sigma$ then gives that
\begin{align}
|\Ebb_{P_{Z^M_i}P_{U|Z^M_i}}[L_t(U|Z^M_i) ]-\Ebb_{P_{U}P'_{Z^M_i}}[L_t(U|Z^M_i)]| &\leq \sqrt{2 \sigma^2 \biggl(D(P_{Z^M}||P'_{Z^M})+I(U;Z^M_i) \biggr)}. \label{eq:b1}
\end{align}
In a similar way, the second difference in \eqref{eq:env1} can be bounded by using \eqref{eq:expinequality_avgtfr_env_ineq1_ITMI}. Applying Jensen's inequality, taking log on both sides, and finally choosing  $\lambda=\sqrt{2I(U;Z^M_i)}/\sigma$ then yields
\begin{align}
|\Ebb_{P_{U}P'_{Z^M_i}}[L_t(U|Z^M)]-\Ebb_{P'_{Z^M_i}P_{U|Z^M_i}}[L_t(U|Z^M_i) ]| \leq \sqrt{2 \sigma^2 I(U;Z^M_i)} \label{eq:b2}
\end{align}
 Combining \eqref{eq:b1} and \eqref{eq:b2} in \eqref{eq:env1} and using it in \eqref{eq:decomp_ineq1_ITMI} together with \eqref{eq:b} gives the upper bound in \eqref{eq:avgtfrgap_bound1_ITMI}.
\section{Proof of Corollary~\ref{cor:avgtfrgap_tfrlearning_ISMIbound}}\label{app:avgtfrgap_tfrlearning_ISMIbound}
The bound \eqref{eq:avgtfrgap_transferlearning} follows by specializing the bound \eqref{eq:avgtfrgap_bound1_ITMI} to the setting considered here. Towards this, we first note that the meta-training set $\mset=Z^{\bar{M}}=(Z_1,\hdots Z_{\bar{M}})$ with its $i$th sub-set $Z^M_i$ corresponding to the data sample $Z_i$, where  $Z_i\sim P_{Z|\tau}$ for $i=1,\hdots,\beta \bar{M}$ and $Z_i\sim P_{Z|\tau'}$ for $i=\beta \bar{M}+1,\hdots,\bar{M}$. Thus, there are $\beta N M =\beta \bar{M}$ data samples from the source task environment and $(1-\beta)\bar{M}$ samples from the target task environment. Using $P_{W|Z^M,U}=\delta(W-U)$ and $U=W$, we then have $\Lscr'_g(u)=L_g(w|\tau')$, and $\Lscr_t(u|\mset)=L_t(w|Z^{\bar{M}})$ with $L_t(u|Z^M_i)=l(w,Z_i)$. Consequently, we have $I(U;Z^M_i)=I(W;Z_i)$ and the KL divergence $D(P_{Z^M}||P'_{Z^M})=D(P_{Z|\tau}||P_{Z|\tau'})$.  It can also be verified that $P_{W,Z_j|\tau'}=P_{W|\tau'}P_{Z_j|\tau'}=P_{W}P_{Z_j|\tau'}$ whereby the MI $I(W;Z_j|\tau')=0$ in \eqref{eq:avgtfrgap_bound1_ITMI}. Further, since $L_t(u|z^M)=l(w,z)$, Assumption~\ref{assum:1} then implies that $\sigma^2=\delta_{\tau'}^2$.
 Using all these expressions in \eqref{eq:avgtfrgap_bound1_ITMI} yields the bound \eqref{eq:avgtfrgap_transferlearning}.
 \section{Exponential Inequalities Based on Assumption~\ref{assum:newassumption}} \label{app:newexponentialinequalities}
 \begin{lemma}\label{lem:expinequality_avgtfrgap_ITMI_new}
Under Assumption~\ref{assum:newassumption} and Assumption~\ref{assum:1a}, the following inequality holds for $j=1,\hdots,M$,
\begin{align}
\Ebb_{P_{W,Z_j|\tau}}\biggl[\exp \biggl( \lambda( l(W,Z_j)- \Ebb_{P_{Z_j|\tau}}[l(W,Z_j)]-\imath(W,Z_j|T=\tau) -\frac{\lambda^2\delta_{\tau}^2}{2} \biggr) \biggr] \leq 1,\label{eq:exp_task_newassumption}
\end{align}
 for all $\lambda \in \Real$ and for each task $\tau \in \mathcal{T}$. Moreover, we have the following inequality for $i=1,\hdots,\beta N$
 \begin{align}
\Ebb_{P_{Z^M_i}P_{U|Z^M_i}}\biggl[\exp \biggl( \lambda( L_t(U|Z^M_i)- \Ebb_{P'_{Z^M_i}}[ L_t(U|Z^M_i)]-\log \frac{\dP_{Z^M_i}(Z^M_i)}{\dP'_{Z^M_i}(Z^M_i)}-\imath(U,Z^M_i)-\frac{\lambda^2\sigma ^2}{2} \biggr) \biggr] \leq 1, \label{eq:exp_env1_newassumption}
\end{align}
and for $i=\beta N+1,\hdots, N$, we have
\begin{align}
\Ebb_{P'_{Z^M_i}P_{U|Z^M_i}}\biggl[\exp \biggl( \lambda( L_t(U|Z^M_i)- \Ebb_{P'_{Z^M_i}}[ L_t(U|Z^M_i)]-\imath(U,Z^M_i)-\frac{\lambda^2\sigma ^2}{2} \biggr) \biggr] \leq 1, \label{eq:exp_env2_newassumption}
\end{align}
which holds for all $\lambda \in \Real$.
\end{lemma}
\begin{proof}
 Under Assumption~\ref{assum:newassumption}, the following inequality holds for each task $\tau \in \mathcal{T}$ and for all $w \in \Wscr$ and $\lambda \in \Real$,
\begin{align}
\Ebb_{P_{Z_j|\tau}}\biggl[\exp \biggl( \lambda( l(w,Z_j)- \Ebb_{P_{Z_j|\tau}}[l(w,Z)] -\frac{\lambda^2\delta_{\tau}^2}{2} \biggr) \biggr] \leq 1. \label{eq:tasklevel_exp_assum}
\end{align} Now, averaging both sides with respect to $W \sim P_{W|\tau}$, where $P_{W|\tau}$ is obtained by marginalizing $P_{W|Z^M,U}P_U P_{Z^M|T=\tau}$, we get that
 \begin{align}
\Ebb_{P_{W|\tau}P_{Z_j|\tau}}\biggl[\exp \biggl( \lambda( l(W,Z_j)- \Ebb_{P_{Z_j|\tau}}[l(W,Z)] -\frac{\lambda^2\delta_{\tau}^2}{2} \biggr) \biggr] \leq 1.
\end{align}
Performing change of measure from $P_{Z_j|\tau}P_{W|\tau}$ to $P_{W,Z_j|\tau}$ similar to Appendix~\ref{app:expinequality_avgtfrgap_ITMI} gets us to the exponential inequality in \eqref{eq:exp_task_newassumption}.

Similarly, for obtaining environment-level exponential inequalities, we have from Assumption~\ref{assum:newassumption} the following inequality
\begin{align}
\Ebb_{P'_{Z^M_i}}\biggl[\exp \biggl( \lambda( L_t(u|Z^M_i)- \Ebb_{P'_{Z^M_i}}[ L_t(u|Z^M_i)] -\frac{\lambda^2\sigma ^2}{2} \biggr) \biggr] \leq 1, \label{eq:ineq_proof12}
\end{align} for $i=1,\hdots, N$, which holds for all $u \in \Uscr$ and $\lambda \in \Real$. Now, to get to \eqref{eq:exp_env1_newassumption}, average both sides with respect to $U \sim P_U$, and change measure from $P'_{Z^M_i}$ to $ P_{Z^M_i}$. This results in the following for $i=1,\hdots,\beta N$
\begin{align}
\Ebb_{P_{Z^M_i}P_{U}}\biggl[\exp \biggl( \lambda( L_t(U|Z^M_i)- \Ebb_{P'_{Z^M_i}}[ L_t(U|Z^M_i)] -\log \frac{\dP_{Z^M_i}(Z^M_i)}{\dP'_{Z^M_i}(Z^M_i)}-\frac{\lambda^2\sigma ^2}{2} \biggr) \biggr] \leq 1. 
\end{align}
Performing a second change of measure from $P_{Z^M_i}P_{U}$ to $P_{Z^M_i}P_{U|Z^M_i}$ then yields the exponential inequality in \eqref{eq:exp_env1_newassumption}.
For $i=\beta N+1,\hdots,N$, we obtain \eqref{eq:exp_env2_newassumption} from \eqref{eq:ineq_proof12} by first averaging over $P_U$, then performing a change of measure from $P'_{Z^M_i}P_U$ to $P'_{Z^M_i}P_{U|Z^M_i}$. 
\end{proof}
%

To see how the exponential inequalities in Lemma~\ref{lem:expinequality_avgtfrgap_ITMI_new} yields the upper bound in Theorem~\ref{thm:avgtfrgap_bound1_ITMI}, we proceed as in the proof of Theorem~\ref{thm:avgtfrgap_bound1_ITMI} in Appendix~\ref{app:avgtfrgap_bound1_ITMI}. To bound the difference in expectation in \eqref{eq:decomp_ineq2_ITMI}, we use the exponential inequality \eqref{eq:exp_task_newassumption}. Note that applying Jensen's inequality on \eqref{eq:exp_task_newassumption} results in the inequality \eqref{eq:point1}.
Similarly, to bound the enviroment-level generalization gap in \eqref{eq:env1}, we use the exponential inequalities \eqref{eq:exp_env1_newassumption} and \eqref{eq:exp_env2_newassumption} and apply Jensen's inequality. In particular, applying Jensen's inequality on \eqref{eq:exp_env1_newassumption} leads to \eqref{eq:point2}. The required bound is then obtained by proceeding as in Appendix~\ref{app:avgtfrgap_bound1_ITMI}.
\section{Proof of Theorem~\ref{thm:transferrisk_bound}}\label{app:transferrisk_bound}
For obtaining an upper bound on the average transfer meta-excess risk, we bound the average transfer generalization gap, the first difference in \eqref{eq:excessrisk_1}, by \eqref{eq:avgtfrgap_bound1_ITMI}.

We now bound the second difference in \eqref{eq:excessrisk_1}. This can be equivalently written as
\begin{align}
&\Ebb_{P_{\mset}}[\Lscr_t(u^{*}|\mset)- \Lscr'_g(u^{*})]\\
&=\Ebb_{P_{\mset}}[\Lscr_t(u^{*}|\mset)-\Lscr'_{g,t}(u^{*})+\Lscr'_{g,t}(u^{*})- \Lscr'_g(u^{*})]\\
&=\alpha\Ebb_{P_{Z^M_{1:\beta N}}}[\Lscr_t(u^{*}|Z^M_{1:\beta N})-\Lscr'_{g,t}(u^{*})]+(1-\alpha)\Ebb_{P'_{Z^M_{\beta N+1:N}}}[\Lscr_t(u^{*}|Z^M_{\beta N+1:N})-\Lscr'_{g,t}(u^{*})]\non \\&+\Lscr'_{g,t}(u^{*})- \Lscr'_g(u^{*})\non \\
&=\alpha\Ebb_{P_{Z^M_{1:\beta N}}}[\Lscr_t(u^{*}|Z^M_{1:\beta N})-\Lscr'_{g,t}(u^{*})]+\Lscr'_{g,t}(u^{*})- \Lscr'_g(u^{*})
\label{eq:eq11}
\end{align}
where the last equality follows since $\Ebb_{P'_{Z^M_{\beta N+1:N}}}[\Lscr_t(u^{*}|Z^M_{\beta N+1:N})]=\Lscr'_{g,t}(u^{*})$.
We now separately bound the two differences in \eqref{eq:eq11}.

To bound the first difference in \eqref{eq:eq11}, note that
$$\Ebb_{P_{Z^M_{1:\beta N}}}[\Lscr_t(u^{*}|Z^M_{1:\beta N})-\Lscr'_{g,t}(u^{*})]=\Ebb_{P_{Z^M}}[L_t(u^{*}|Z^M)]-\Ebb_{P'_{Z^M}}[L_t(u^{*}|Z^M)]. $$ To bound this term, we resort to the inequality \eqref{eq:ineq_proof12} which is a consequence of Assumption~\ref{assum:newassumption} (note that we can ignore the subscript $i$ in the current context), and fix $u=u^{*}$. Applying change of measure from $P'_{Z^M}$ to $P_{Z^M}$ then yields the following inequality,
\begin{align}
\Ebb_{P_{Z^M}}\biggl[\exp \biggl( \lambda( L_t(u^{*}|Z^M)- \Ebb_{P'_{Z^M}}[ L_t(u^{*}|Z^M)] -\log \frac{\dP_{Z^M}(Z^M)}{\dP'_{Z^M}(Z^M)}-\frac{\lambda^2\sigma ^2}{2} \biggr) \biggr] \leq 1, \label{eq:ineq_proof123}
\end{align}which holds for all $\lambda \in \Real$. Applying Jensen's inequality and choosing $\lambda=\sqrt{2D(P_{Z^M}||P'_{Z^M})}/\sigma$ then gives that
\begin{align}
\Ebb_{P_{Z^M}}[L_t(u^{*}|Z^M)]-\Ebb_{P'_{Z^M}}[L_t(u^{*}|Z^M)] \leq \sqrt{2 \sigma^2 D(P_{Z^M}||P'_{Z^M})}. \label{eq:ub}
\end{align}

We now bound the second difference in \eqref{eq:eq11}. Towards this, note that the following set of relations hold,
\begin{align}
\Lscr'_{g,t}(u^{*})-\Lscr'_g(u^{*})
&= \Ebb_{P'_T} \Ebb_{P^M_{Z|T}}\Ebb_{P_{W|Z^M,u^{*}}}[L_t(W|Z^M)-L_g(W|T)] \non \\
&=\Ebb_{P'_T}\biggl[ \frac{1}{M}\sum_{j=1}^M \biggl(\Ebb_{P_{W,Z_j|u^{*},T=\tau}}[l(W,Z_j)]-\Ebb_{P_{W|u^{*},T=\tau} P_{Z_j|T=\tau}}[l(W,Z_j)]\biggr)\biggr]. \label{eq:eq2}
\end{align}
 To bound the difference $\Ebb_{P_{W,Z_j|u^{*},T=\tau}}[l(W,Z_j)]-\Ebb_{P_{W|u^{*},T=\tau} P_{Z_j|T=\tau}}[l(W,Z_j)]$, we slightly modify the exponential inequality \eqref{eq:exp_task_newassumption} in Lemma~\ref{lem:expinequality_avgtfrgap_ITMI_new}. Towards this, we average the inequality \eqref{eq:tasklevel_exp_assum} with respect to $W \sim P_{W|\tau,u^{*}}$, where $P_{W|\tau,u^{*}}$ is the marginal of the joint $P_{W|Z^M,u^{*}}P_{Z^M|\tau}$, and subsequently perform a change of measure from $P_{Z_j|\tau}P_{W|\tau,u^{*}}$ to $P_{W,Z_j|\tau,u^{*}}$. This results in the following modified form of \eqref{eq:exp_task_newassumption}
 \begin{align}
\Ebb_{P_{W,Z_j|\tau,u^{*}}}\biggl[\exp \biggl( \lambda( l(W,Z_j)- \Ebb_{P_{Z_j|\tau}}[l(W,Z_j)]-\imath(W,Z_j|T=\tau,u^{*}) -\frac{\lambda^2\delta_{\tau}^2}{2} \biggr) \biggr] \leq 1.\label{eq:exp_task_newassumption_modified}
\end{align}
Now, applying Jensen's inequality, and choosing $\lambda=\sqrt{2I(W;Z_j|T=\tau,u^{*})}/\delta_{\tau}$ gives that
\begin{align}
\Ebb_{P_{W,Z_j|\tau,u^{*}}}[ l(W,Z_j)]- \Ebb_{P_{W|\tau,u^{*}}P_{Z_j|\tau}}[l(W,Z_j)] \leq \sqrt{2 \delta_{\tau}^2 I(W;Z_j|T=\tau,u^{*})}.
\end{align}
Substituting this in \eqref{eq:eq2}, and using the resulting inequality together with \eqref{eq:ub} in \eqref{eq:eq11} yields the required bound.
\section{ Exponential Inequalities for PAC-Bayesian and Single-Draw Probability Bounds}\label{app:expinequality_PACBayesian}
We now present two exponential inequalities that are crucial to the derivation of high-probability PAC-Bayesian and high-probability single-draw bounds. Towards this, we first define the following \textit{mismatched information densities}
\begin{align}
\jmath(U,\mset)=\log \frac{\dP_{U|\mset}(U|\mset)}{\dQ_{U}(U)}, \quad \jmath(W,Z^M|U)=\log \frac{\dP_{W|Z^M,U}(W|Z^M,U)}{\dQ_{W|U}(W|U)}
\end{align} where  $Q_U \in \Pscr(\Uscr)$ represents an arbitrary data-independent hyper-prior over the space of hyperparameters, and $Q_{W|U=u} \in \Pscr(\Wscr)$ represents a class of arbitrary data-independent priors over the space of model parameters for each $u\in \Uscr$. The mismatched information density $\jmath(U,\mset)$ quantifies the evidence for the hyperparameter $U$ to be generated according to the meta-learner $P_{U|\mset}$ based on meta-training set, rather than being generated according to the hyper-prior distribution $Q_U$. Similarly, the density $\jmath(W,Z^M|U)$ quantifies the evidence of the model parameter $W$ being generated by the base learner $P_{W|Z^M,U}$ based on the training set $Z^M$, rather than being generated according to the prior.

We denote $Z^M_{1:N /i}:=(Z^M_1,\hdots,Z^M_{i-1},Z^M_{i+1},\hdots,Z^M_N)$, for $i=1,\hdots,N$, to be the meta-training set without the $i$th subset and is distributed according to $P_{Z^M_{1:N/i}}$ which is obtained by marginalizing $P_{\mset}$.
\begin{lemma}\label{lem:expinequality_PB_task}
Under Assumption~\ref{assum:2}$(a)$ and Assumption~\ref{assum:2a}, the following 
exponential inequality holds for the $i$th sub-set, $Z^M_i \sim P^M_{Z|T=T_i}$, of the meta-training set $\mset=(Z^M_i,Z^M_{1:N/i})$ for $i=1,\hdots,N$,
\begin{align}
\Ebb_{P_{Z^M_{1:N/i}}}\Ebb_{P^M_{Z|T=T_i}P_{U|\mset}P_{W|Z^M_i,U}}\biggl[&\exp \biggl(\lambda (L_t(W|Z^M_i)-L_g(W|T_i))-\frac{\lambda^2 \delta_{T_i}^2}{2M}\non \\&-\jmath({W},Z^M_i|U)-\jmath({U},\mset) \biggr) \biggr] \leq 1. \label{eq:expinequality_PB_task2}
\end{align} 
\end{lemma}
\begin{proof}
From Assumption~\ref{assum:2}$(a)$, we have that for task $T=T_i$, 
 $L_t(w|Z^M_i)$ is the average of $M$ independent $\delta_{T_i}^2$-sub-Gaussian random variables $l(w,Z_i)$. It is then easy to see that $L_t(w|Z^M_i)$ is $\delta_{T_i}^2/M$-sub-Gaussian under $Z^M_i \sim P_{Z|T_i}^M$ for all $w \in \Wscr$. This can be equivalently expressed as
\begin{align}
\Ebb_{P^M_{Z|T=T_i}}\biggl[\exp \biggl(\lambda (L_t(w|Z^M_i)-L_g(w|T_i))-\frac{\lambda^2 \delta_{T_i}^2}{2M} \biggr) \biggr] \leq 1
\end{align}
which holds for all $w \in \Wscr$ and $\lambda \in \Real$. Averaging both sides with respect to $Z^M_{1:N /i}$ gives that
\begin{align}
\Ebb_{P_{Z^M_{1:N/i}}}\Ebb_{P^M_{Z|T=T_i}}\biggl[\exp \biggl(\lambda (L_t(w|Z^M_i)-L_g(w|T_i))-\frac{\lambda^2 \delta_{T_i}^2}{2M} \biggr) \biggr] \leq 1 \label{eq:ineq11}
\end{align} for all $w \in \Wscr$.
To get to the inequality \eqref{eq:expinequality_PB_task2}, we consider \eqref{eq:ineq11} as a function of both model parameter $w$ and hyperparameter $u$. Subsequently, average both sides of inequality \eqref{eq:ineq11} with respect to $Q_{W,U}=Q_UQ_{W|U} \in \Pscr(\Wscr \times \Uscr)$. We now follow the approach of \cite[Prop. 17.1]{polyanskiy2014lecture} and apply a change of measure as detailed below. Towards this, we first note that average over $Q_{W,U}$ on \eqref{eq:ineq11} implies the following inequality
\begin{align}
\Ebb_{P_{Z^M_{1:N/i}}}\Ebb_{P^M_{Z|T=T_i}}\Ebb_{Q_{W,U}}\biggl[\Ibb_{\Escr(Z^M_i,Z^M_{1:N/i})}\exp \biggl(\lambda (L_t(W|Z^M_i)-L_g(W|T_i))-\frac{\lambda^2 \delta_{T_i}^2}{2M} \biggr) \biggr] \leq 1, \label{ineq:12}
\end{align}where $\Escr(z^M_i,z^M_{1:N/i})={\rm supp}(P_{W,U|z^M_i,z^M_{1:N/i}})$ and $P_{W,U|Z^M_i,Z^M_{1:N/i}}=P_{U|\mset}P_{W|U,Z^M_i}$.
It is then easy to see that for $Z^M_i=z^M_i$, $Z^M_{1:N/i}=z^M_{1:N/i}$, the following relation holds
\begin{align}
&\Ebb_{Q_{W,U}}\biggl[\Ibb_{\Escr(z^M_i,z^M_{1:N/i})}\exp \biggl(\lambda (L_t(W|z^M_i)-L_g(W|T_i))-\frac{\lambda^2 \delta_{T_i}^2}{2M} \biggr) \biggr]\\
&=\Ebb_{P_{W,U|z^M_i,z^M_{1:N/i}}}\biggl[\exp \biggl(\lambda (L_t(W|z^M_i)-L_g(W|T_i))-\frac{\lambda^2 \delta_{T_i}^2}{2M} -\log \frac{\dP_{W,U|z^M_i,z^M_{1:N/i}}(W,U)}{\dQ_{W,U}(W,U)} \biggr) \biggr].
\end{align}
Using this in \eqref{ineq:12} and averaging over $Z^M_i,Z^M_{1:N/i}$ then yields inequality \eqref{eq:expinequality_PB_task2} with $$\log \frac{\dP_{W,U|Z^M_i,Z^M_{1:N/i}}(W,U|Z^M_i,Z^M_{1:N/i})}{\dQ_{W,U}(W,U)}=\jmath({W},Z^M_i|U)+\jmath({U},\mset).$$ 
\end{proof}
Inequality \eqref{eq:expinequality_PB_task2} relates the per-task training loss to the per-task generalization loss and the information densities $\jmath({W},Z^M_i|U)$, $\jmath({U},\mset)$. We will use this to bound the contribution of within-task generalization gap to transfer meta-generalization gap. 
Assumption~\ref{assum:2}$(b)$ then provides the following exponential inequality on the difference $\Lscr_{t,g}(u|\mset,T_{1:N})-\Lscr'_g(u)$.
%
\begin{lemma}\label{lem:expinequality_PB_env}
Under Assumption~\ref{assum:2}$(b)$ and Assumption~\ref{assum:2a}, the following exponential inequality  holds
\begin{align}
\Ebb_{P_{T_{1: N}}P_{\mset|T_{1: N}}}\Ebb_{P_{U|\mset}}&\biggl[\exp \biggl(\lambda  \biggl(\Lscr_{t,g}(U|T_{1: N},Z^M_{1: N})-\Lscr_g'(U)\biggr)- \frac{\lambda^2 \alpha^2 \sigma^2}{2\beta N} \non \\&- \sum_{i=1}^{\beta N}\log \frac{P_{T}(T_i)}{P'_{T}(T_i)} -\frac{\lambda^2 (1-\alpha)^2 \sigma^2}{2(1-\beta) N}-\jmath({U},\mset)\biggr) \biggr] \leq 1 \label{eq:expinequality_PB_env1}.
\end{align}
\end{lemma}
\begin{proof}
In the following, we denote $T_{1:\beta N}:=(T_1,\hdots, T_{\beta N})$, $T_{\beta N+1:N}:=(T_{\beta N+1}, \hdots,T_N)$, the empirical average per-task test loss of the source environment data set as \begin{align}
\Lscr_{t,g}(u|T_{1:\beta N},Z^M_{1:\beta N})&=\frac{1}{\beta N} \sum_{i=1}^{\beta N}L_g(u|Z^M_i,T_i),\non
\end{align} and the empirical average per-task test loss of the target environment data set as
\begin{align}
\Lscr_{t,g}(u|T_{\beta N+1:N},Z^M_{\beta N+1:N})&=\frac{1}{(1-\beta) N} \sum_{i=\beta N+1}^{N}L_g(u|Z^M_i,T_i). \non
\end{align}

From Assumption~\ref{assum:2}$(b)$, we get that
$\Lscr_{t,g}(u|T_{1:\beta N},Z^M_{1:\beta N})$ is the average of i.i.d. $\sigma^2$-sub-Gaussian random variables under $(T_i,Z^M_i) \sim P'_{T_i}P^M_{Z|T_i}$. Consequently, it is $\sigma^2/\beta N$-sub-Gaussian when $(T_{1:\beta N},Z^M_{1:\beta N})\sim P'_{T_{1:\beta N}}P_{Z^M_{1:\beta N}|T_{1:\beta N}}$ for all $u \in \Uscr$. Note here that we use $P'_{T_{1:\beta N}}P_{Z^M_{1:\beta N}|T_{1:\beta N}}$ to denote the product distribution $\prod_{i=1}^{\beta N} P'_{T_i}P^M_{Z|T_i}$.
Similarly,  $\Lscr_{t,g}(u|T_{\beta N+1:N},Z^M_{\beta N+1:N})$ is $\sigma^2/(1-\beta)N$-sub-Gaussian under $(T_{\beta N+1:N},Z^M_{\beta N+1:N})\sim P'_{T_{\beta N+1:N}}P_{Z^M_{\beta N+1:N}|T_{\beta N+1:N}}$ for all $u \in \Uscr$. Here, $P'_{T_{\beta N+1:N}}P_{Z^M_{\beta N+1:N}|T_{\beta N+1:N}}$ denotes the product distribution $\prod_{i=\beta N+1}^{N} P'_{T_i}P^M_{Z|T_i}$.
Denoting $P'_{T_{1:N}}=\prod_{i=1}^N P'_{T_i}$,
 the following set of relations then follow from the sub-Gaussianity assumptions discussed above, and holds for all $u \in \Uscr$ and $\lambda \in \Real$:
\begin{align}
&\Ebb_{P'_{T_{1:N}}P_{\mset|T_{1:N}}}\biggl[\exp \biggl(\lambda (\Lscr_{t,g}(u|T_{1:N},\mset)-\Lscr_g'(u)) \biggr) \biggr]\\
&=\Ebb_{P'_{T_{1:\beta N}}P_{Z^M_{1:\beta N}|T_{1:\beta N}}}\biggl[\exp \biggl(\lambda \alpha \biggl(\Lscr_{t,g}(u|T_{1:\beta N},Z^M_{1:\beta N})-\Lscr_g'(u)\biggr) \biggr) \biggr] \times \non \\& \qquad \qquad  \Ebb_{P'_{T_{\beta N+1:N}}P_{Z^M_{\beta N+1:N}|T_{\beta N+1:N}}}\biggl[\exp \biggl(\lambda (1-\alpha) \biggl(\Lscr_{t,g}(u|T_{\beta N+1:N},Z^M_{\beta N+1:N})-\Lscr_g'(u)\biggr) \biggr) \biggr] \non \\
& \leq \exp \biggl(\frac{\lambda^2 \alpha^2 \sigma^2}{2\beta N} \biggr) \exp \biggl(\frac{\lambda^2 (1-\alpha)^2 \sigma^2}{2(1-\beta) N} \biggr). \label{eq:separateterms}
\end{align}
This in turn implies that
\begin{align}
\Ebb_{P'_{T_{\beta N+1:N}}P_{Z^M_{\beta N+1:N}|T_{\beta N+1:N}}}\Ebb_{P'_{T_{1:\beta N}}P_{Z^M_{1:\beta N}|T_{1:\beta N}}}&\biggl[\Ibb_{\Escr}\exp \biggl(\lambda (\Lscr_{t,g}(u|T_{1:N},\mset)-\Lscr_g'(u))-\frac{\lambda^2 \alpha^2 \sigma^2}{2\beta N} \non \\&-\frac{\lambda^2 (1-\alpha)^2 \sigma^2}{2(1-\beta) N} \biggr) \biggr] \leq 1
\end{align}
where $\Escr={\rm \supp}(P_{T_{1:\beta N}}P_{Z^M_{1:\beta N}|T_{1:\beta N}})$. Applying change of measure from $P'_{T_{1:\beta N}}P_{Z^M_{1:\beta N}|T_{1:\beta N}}$ to $P_{T_{1:\beta N}}P_{Z^M_{1:\beta N}|T_{1:\beta N}}$, then yields 
%
\begin{align}
\Ebb_{P_{T_{1:\beta N}}P'_{T_{\beta N+1:N}}P_{\mset|T_{1: N}}}&\biggl[\exp \biggl(\lambda  \biggl(\Lscr_{t,g}(u|T_{1: N},Z^M_{1: N})-\Lscr_g'(u)\biggr)- \frac{\lambda^2 \alpha^2 \sigma^2}{2\beta N} \non \\&-\log \frac{\dP_{T_{1:\beta N}}(T_{1:\beta N})}{\dP'_{T_{1:\beta N}}(T_{1:\beta N})} -\frac{\lambda^2 (1-\alpha)^2 \sigma^2}{2(1-\beta) N}\biggr) \biggr] \leq 1,
\end{align} which holds for all $u \in \Uscr$. 
Average both sides of the inequality with respect to $Q_U \in \Pscr(\Uscr)$. The resultant inequality implies the following
\begin{align}
\Ebb_{P_{T_{1:\beta N}}P'_{T_{\beta N+1:N}}P_{\mset|T_{1: N}}}\Ebb_{Q_U}&\biggl[\Ibb_{\Escr(\mset)}\exp \biggl(\lambda  \biggl(\Lscr_{t,g}(U|T_{1: N},Z^M_{1: N})-\Lscr_g'(u)\biggr)- \frac{\lambda^2 \alpha^2 \sigma^2}{2\beta N} \non \\&-\log \frac{P_{T_{1:\beta N}}(T_{1:\beta N})}{P'_{T_{1:\beta N}}(T_{1:\beta N})} -\frac{\lambda^2 (1-\alpha)^2 \sigma^2}{2(1-\beta) N}\biggr) \biggr] \leq 1,
\end{align} where $\Escr(z^M_{1:N})={\rm supp}(P_{U|\mset=z^M_{1:N}})$.
 Applying change of measure from $Q_U $ to $P_{U|\mset}$ together with $\log \Bigl(\dP_{T_{1:\beta N}}(T_{1:\beta N})/\dP'_{T_{1:\beta N}}(T_{1:\beta N}) \Bigr) =\sum_{i=1}^{\beta N} \log (\dP_T(T_i)/\dP'_T(T_i))$ then gives the required inequality \eqref{eq:expinequality_PB_env1}.
\end{proof}
The inequality \eqref{eq:expinequality_PB_env1} relates the difference between weighted average per-task test loss and transfer meta-generalization loss, $\Lscr_{t,g}(U|T_{1: N},Z^M_{1: N})-\Lscr'_g(U)$, to the mismatched information density $\jmath({U},\mset)$ and the log-likelihood ratio $\log (\dP_T(T_i)/\dP'_T(T_i))$, that captures the meta-environment shift in task distributions.
\section{Proof of Theorem~\ref{thm:PAC-Bayesianbound}}\label{app:PAC-Bayesianbound}
To obtain the required PAC-Bayesian bound, we use the decomposition \eqref{eq:decomposition-2}. The idea is to separately bound the two differences in \eqref{eq:decomposition-2} in high probability over $(T_{1:N},\mset)$, and subsequently combine the bounds via union bound.

To start, we bound the first difference in \eqref{eq:decomposition-2}. Towards this, we resort to the exponential inequality \eqref{eq:expinequality_PB_env1}. Applying Jensen's inequality with respect to just $P_{U|\mset}$ on \eqref{eq:expinequality_PB_env1} results in
\begin{align}
\Ebb_{P_{T_{1: N}}P_{\mset|T_{1: N}}}&\biggl[\exp \biggl(\lambda  \Ebb_{P_{U|\mset}}\biggl[\Lscr_{t,g}(U|T_{1: N},Z^M_{1: N})-\Lscr_g'(U)\biggr]- \frac{\lambda^2 \alpha^2 \sigma^2}{2\beta N} \non \\&-\sum_{i=1}^{\beta N}\log \frac{\dP_{T}(T_i)}{\dP'_{T}(T_i)} -\frac{\lambda^2 (1-\alpha)^2 \sigma^2}{2(1-\beta) N}-D(P_{U|\mset}||Q_U)\biggr) \biggr] \leq 1.
\end{align}
Take $V=\exp \bigl(\lambda  \Ebb_{P_{U|\mset}}\bigl[\Lscr_{t,g}(U|T_{1: N},Z^M_{1: N})-\Lscr_g'(U)\bigr]- \lambda^2 \alpha^2 \sigma^2/(2\beta N)-\sum_{i=1}^{\beta N}\log (\dP_{T}(T_i)/\dP'_{T}(T_i)) -\lambda^2 (1-\alpha)^2 \sigma^2/(2(1-\beta) N)-D(P_{U|\mset}||Q_U)\bigr)$. Applying Markov's inequality of the form $\Pbb[V \geq \frac{1}{\delta_0}]\leq \delta_0 \Ebb[V]\leq \delta_0$ then gives that with probability at least $1-\delta_0$ over $(\mset,T_{1:N})\sim P_{T_{1:N}}P_{\mset|T_{1:N}}$ we have $V \leq \frac{1}{\delta_0}$. Taking logarithm on both sides of the inequality then results in
\begin{align}
 \lambda \biggl( \Ebb_{P_{U|\mset}}\biggl[\Lscr_{t,g}(U|T_{1: N},Z^M_{1: N})&-\Lscr_g'(U)\biggr]\biggr)  \leq \frac{\lambda^2 \alpha^2 \sigma^2}{2\beta N}+\frac{\lambda^2 (1-\alpha)^2 \sigma^2}{2(1-\beta) N} \non \\&+ \sum_{i=1}^{\beta N}\log \frac{\dP_{T}(T_i)}{\dP'_{T}(T_i)}+D(P_{U|\mset}||Q_U)+ \log \frac{1}{\delta_0} \label{eq:envlevl}
\end{align}
which when optimized over $\lambda$ with the choice $$\lambda=\sqrt{\frac{\sum_{i=1}^{\beta N}\log \frac{\dP_{T}(T_i)}{\dP'_{T}(T_i)}+D(P_{U|\mset}||Q_U)+ \log \frac{1}{\delta_0}}{\frac{ \alpha^2 \sigma^2}{2\beta N}+\frac{ (1-\alpha)^2 \sigma^2}{2(1-\beta) N}} } $$ yields
\begin{align}
&\biggl|\Ebb_{P_{U|\mset}}\biggl[\Lscr_{t,g}(u|T_{1: N},Z^M_{1: N})-\Lscr_g'(u)\biggr]\biggr| \non \\& \leq \sqrt{ 2 \sigma^2 \biggl( \frac{ \alpha^2 }{\beta N}+\frac{(1-\alpha)^2 }{(1-\beta) N}\biggr) \biggl( \sum_{i=1}^{\beta N}\log \frac{\dP_{T}(T_i)}{\dP'_{T}(T_i)}+D(P_{U|\mset}||Q_U)+ \log \frac{1}{\delta_0}\biggr)}. \label{eq:envlevl_opt}
\end{align}

We now bound the second difference in \eqref{eq:decomposition-2}, which can be equivalently written as
\begin{align}
&\Ebb_{P_{U|\mset}}\biggl[ \Lscr_{t,g}(U|\mset,T_{1:N}) -\Lscr_t(U|\mset)\biggr] \non\\
&=\Ebb_{P_{U|\mset}}\biggl[ \frac{\alpha}{\beta N}\sum_{i=1}^{\beta N} (L_g(U|Z^M_i,T_i)-L_t(U|Z^M_i)) + \frac{1-\alpha}{(1-\beta)N}\sum_{i=\beta N+1}^{N} (L_g(U|Z^M_i,T_i)-L_t(U|Z^M_i))\biggr]. \label{eq:1111}
\end{align}
The idea then is to bound each of the terms
$
\Ebb_{P_{U|\mset}}[L_g(U|Z^M_i,T_i)-L_t(U|Z^M_i)]
$ separately  with probability at least $(1-\delta_i)$ over $(Z^M_{1:N/i},T_i,Z^M_i)\sim P_{Z^M_{1:N/i}}P_{T_i}P^M_{Z|T_i}$ for $i=1,\hdots, \beta N$ and over  $(Z^M_{1:N/i},T_i,Z^M_i)\sim P_{Z^M_{1:N/i}}P'_{T_i}P^M_{Z|T_i}$ for $i=\beta N+1,\hdots, N$. Towards this, we resort to the exponential inequality \eqref{eq:expinequality_PB_task2} and apply Jensen's inequality with respect to $P_{U|\mset}P_{W|Z^M_i,U}$. This results in
\begin{align}
\Ebb_{P_{Z^M_{1:N/i}}}\Ebb_{P^M_{Z|T_i}}&\biggl[\underbrace{\exp  \biggl(\lambda \Ebb_{P_{U|\mset}}[L_t(U|Z^M_i)-L_g(U|Z^M_i,T_i)]-\frac{\lambda^2 \delta_{T_i}^2}{2M} -D(P_{W,U|\mset,Z^M_i}||Q_{W,U})  \biggr)}_{V}\biggr]\non \\& \leq 1,
\end{align} which holds for $i=1,\hdots,N$. Note that the above inequality holds even after averaging both sides of the inequality with respect to $P_{T_i}$ (or $P'_{T_i}$).
Applying Markov's inequality of the form $\Pbb[V \geq \frac{1}{\delta_i}]\leq \beta_0 \Ebb[V]\leq \delta_i$ then gives that with probability at least $1-\delta_i$ over $(Z^M_{1:N/i},T_i,Z^M_i)\sim P_{Z^M_{1:N/i}}P_{T_i}P^M_{Z|T_i}$ for $i=1,\hdots, \beta N$ and over  $(Z^M_{1:N/i},T_i,Z^M_i)\sim P_{Z^M_{1:N/i}}P'_{T_i}P^M_{Z|T_i}$ for $i=\beta N+1,\hdots, N$
\begin{align}
 &\lambda \biggl(\Ebb_{P_{U|\mset}}[L_t(U|Z^M_i)-L_g(U|Z^M_i,T_i)]\biggr)\non \\&\leq \frac{\lambda^2 \delta_{T_i}^2}{2M} +D(P_{U|\mset}||Q_U)+ \Ebb_{P_{U|\mset}}[ D(P_{W|U,Z^M_i}||Q_{W|U})]+\log \frac{1}{\delta_i}. \label{eq:tasklevelineq}
\end{align}
Now, choosing $$\lambda=\sqrt{\frac{D(P_{U|\mset}||Q_U)+ \Ebb_{P_{U|\mset}}[ D(P_{W|U,Z^M_i}||Q_{W|U})]+\log \frac{1}{\delta_i}}{\frac{ \delta_{T_i}^2}{2M}}}$$ then results in 
\begin{align}
&\biggl| \Ebb_{P_{U|\mset}}[L_t(U|Z^M_i)-L_g(U|Z^M_i,T_i)]\biggr|\non 
\\
&= \sqrt{\frac{2 \delta_{T_i}^2}{M}\biggl(D(P_{U|\mset}||Q_U)+ \Ebb_{P_{U|\mset}}[ D(P_{W|U,Z^M_i}||Q_{W|U})]+\log \frac{1}{\delta_i}\biggr) }
\label{eq:tasklevelineq_opt}.
\end{align}
Choosing $\delta_0=\frac{\delta}{2}$ and $\delta_i=\frac{\delta}{4\beta N}$ for $i=1,\hdots, \beta N$ and $\delta_i=\frac{\delta}{4(1-\beta) N}$ for $i=\beta N+1,\hdots,N$, and combining the bounds \eqref{eq:envlevl_opt} and \eqref{eq:tasklevelineq_opt} in \eqref{eq:1111} via union bound then yields the bound \eqref{eq:PACBayesian_bound1}.
\section{Proof of Theorem~\ref{thm:singledraw_bound}}\label{app:singledraw_bound}
To obtain the required single-draw bound, we use the decomposition \eqref{eq:decomposition-2}.
We start by bounding the first difference in \eqref{eq:decomposition-2} without the expectation over meta-training algorithm.  Towards this, we resort to the exponential inequality \eqref{eq:expinequality_PB_env1}. 
Take $$V=\exp \biggl(\lambda  \bigl(\Lscr_{t,g}(U|T_{1: N},Z^M_{1: N})-\Lscr_g'(U)\bigr)- \frac{\lambda^2 \alpha^2 \sigma^2}{2\beta N}- \sum_{i=1}^{\beta N}\log \frac{\dP_{T}(T_i)}{\dP'_{T}(T_i)} -\frac{\lambda^2 (1-\alpha)^2 \sigma^2}{2(1-\beta) N}-\jmath(U,\mset)\biggr).$$ Applying Markov's inequality of the form $\Pbb[V \geq \frac{1}{\delta_0}]\leq \delta_0 \Ebb[V]\leq \delta_0$ then gives that with probability at least $1-\delta_0$ over $(T_{1:N},\mset,U)\sim P_{T_{1:N}}P_{\mset|T_{1:N}}P_{U|\mset}$ we have $V \leq \frac{1}{\delta_0}$. Taking logarithm on both sides of the inequality then results in
\begin{align}
\lambda \biggl(\Lscr_{t,g}(U|T_{1: N},Z^M_{1: N})-\Lscr_g'(U)\biggr) & \leq \frac{\lambda^2 \alpha^2 \sigma^2}{2\beta N}+\frac{\lambda^2 (1-\alpha)^2 \sigma^2}{2(1-\beta) N} \non \\&+  \sum_{i=1}^{\beta N}\log \frac{\dP_{T}(T_i)}{\dP'_{T}(T_i)}+\jmath(U,\mset)+ \log \frac{1}{\delta_0} \label{eq:envlevl_sd}
\end{align}
which when optimized over $\lambda$ yields
\begin{align}
&\biggl|\Lscr_{t,g}(u|T_{1: N},Z^M_{1: N})-\Lscr_g'(u)\biggr| \non \\& \leq \sqrt{ 2 \sigma^2 \biggl( \frac{ \alpha^2 }{\beta N}+\frac{(1-\alpha)^2 }{(1-\beta) N}\biggr) \biggl(  \sum_{i=1}^{\beta N}\log \frac{\dP_{T}(T_i)}{\dP'_{T}(T_i)}+\jmath(U,\mset)+ \log \frac{1}{\delta_0}\biggr)}. \label{eq:envlevl_opt_sd}
\end{align}

We now bound the second difference in \eqref{eq:decomposition-2}, which can be equivalently written as
\begin{align}
&\Lscr_{t,g}(U|\mset,T_{1:N}) -\Lscr_t(U|\mset) \non\\
&= \frac{\alpha}{\beta N}\sum_{i=1}^{\beta N} (L_g(U|Z^M_i,T_i)-L_t(U|Z^M_i)) + \frac{1-\alpha}{(1-\beta)N}\sum_{i=\beta N+1}^{N} (L_g(U|Z^M_i,T_i)-L_t(U|Z^M_i)). \label{eq:1111_sd}
\end{align}
We now bound each of the terms
$
L_g(U|Z^M_i,T_i)-L_t(U|Z^M_i)
$ separately with probability at least $(1-\delta_i)$ over $(Z^M_{1:N/i},T_i,Z^M_i,U)\sim P_{Z^M_{1:N/i}}P_{T_i}P^M_{Z|T_i}P_{U|\mset}$ for $i=1,\hdots,\beta N$ and over $(Z^M_{1:N/i},T_i,Z^M_i,U)\sim P_{Z^M_{1:N/i}}P'_{T_i}P^M_{Z|T_i}P_{U|\mset}$ for $i=\beta N+1,\hdots,N$.  Towards this, we resort to the exponential inequality \eqref{eq:expinequality_PB_task2} and apply Jensen's inequality with respect to $P_{W|Z^M_i,U}$. This results in
\begin{align}
\Ebb_{P_{Z^M_{1:N/i}}}\Ebb_{P^M_{Z|T_i}P_{U|\mset}}&\biggl[\underbrace{\exp  \biggl(\lambda (L_t(U|Z^M_i)-L_g(U|Z^M_i,T_i)]-\frac{\lambda^2 \delta_{T_i}^2}{2M} -D(P_{W|U,Z^M_i}||Q_{W|U})-\jmath(U,\mset)}_{V}\biggr]\non \\& \leq 1,
\end{align} which holds for $i=1,\hdots,N$. Note that the above inequality holds even after averaging both sides of the inequality with respect to $P_{T_i}$ (or $P'_{T_i}$).
Applying Markov's inequality of the form $\Pbb[V \geq \frac{1}{\delta_i}]\leq \beta_0 \Ebb[V]\leq \delta_i$ then gives that with probability at least $(1-\delta_i)$ over $(Z^M_{1:N/i},T_i,Z^M_i,U)\sim P_{Z^M_{1:N/i}}P_{T_i}P^M_{Z|T_i}P_{U|\mset}$ for $i=1,\hdots,\beta N$ and over $(Z^M_{1:N/i},T_i,Z^M_i,U)\sim P_{Z^M_{1:N/i}}P'_{T_i}P^M_{Z|T_i}P_{U|\mset}$ for $i=\beta N+1,\hdots,N$, the following inequality holds,
\begin{align}
\lambda(L_t(U|Z^M_i)-L_g(U|Z^M_i,T_i))&\leq \frac{\lambda^2 \delta_{T_i}^2}{2M} +D(P_{W|U,Z^M_i}||Q_{W|U})+\jmath(U,\mset)+\log \frac{1}{\delta_i} \label{eq:tasklevelineq_sd}.
\end{align}
Optimizing over $\lambda$ then results in 
\begin{align}
&\biggl| L_t(U|Z^M_i)-L_g(U|Z^M_i,T_i)\biggr|\non \\&\leq \sqrt{\frac{2 \delta_{T_i}^2}{M}\biggl(D(P_{W|U,Z^M_i}||Q_{W|U})+\jmath(U,\mset)+\log \frac{1}{\delta_i}\biggr) } 
\label{eq:tasklevelineq_opt_sd}.
\end{align}
Choosing $\delta_0=\frac{\delta}{2}$ and $\delta_i=\frac{\delta}{4\beta N}$ for $i=1,\hdots, \beta N$ and $\delta_i=\frac{\delta}{4(1-\beta) N}$ for $i=\beta N+1,\hdots,N$, and combining the bounds \eqref{eq:envlevl_opt_sd} and \eqref{eq:tasklevelineq_opt_sd} via union bound then yields the bound \eqref{eq:singledraw_bound}.

\bibliographystyle{IEEEtran}
\bibliography{ref}

\begin{thebibliography}{10}
\providecommand{\url}[1]{#1}
\csname url@samestyle\endcsname
\providecommand{\newblock}{\relax}
\providecommand{\bibinfo}[2]{#2}
\providecommand{\BIBentrySTDinterwordspacing}{\spaceskip=0pt\relax}
\providecommand{\BIBentryALTinterwordstretchfactor}{4}
\providecommand{\BIBentryALTinterwordspacing}{\spaceskip=\fontdimen2\font plus
\BIBentryALTinterwordstretchfactor\fontdimen3\font minus
  \fontdimen4\font\relax}
\providecommand{\BIBforeignlanguage}[2]{{%
\expandafter\ifx\csname l@#1\endcsname\relax
\typeout{** WARNING: IEEEtran.bst: No hyphenation pattern has been}%
\typeout{** loaded for the language `#1'. Using the pattern for}%
\typeout{** the default language instead.}%
\else
\language=\csname l@#1\endcsname
\fi
#2}}
\providecommand{\BIBdecl}{\relax}
\BIBdecl

\bibitem{schmidhuber1987evolutionary}
J.~Schmidhuber, ``Evolutionary {P}rinciples in {S}elf-{R}eferential {L}earning,
  or {O}n {L}earning {H}ow to {L}earn: {T}he {M}eta-meta-... {H}ook,'' Ph.D.
  dissertation, Technische Universit{\"a}t M{\"u}nchen, 1987.

\bibitem{thrun1996learning}
S.~Thrun, ``Is {L}earning the {N}-th {T}hing {A}ny {E}asier than {L}earning the
  {F}irst?'' in \emph{Proc. of Adv. in Neural Inf. Processing Sys. (NIPS)},
  Dec. 1996, pp. 640--646.

\bibitem{thrun1998learning}
S.~Thrun and L.~Pratt, ``Learning to {L}earn: {I}ntroduction and {O}verview,''
  in \emph{Learning to Learn}.\hskip 1em plus 0.5em minus 0.4em\relax Springer,
  1998, pp. 3--17.

\bibitem{baxter2000model}
J.~Baxter, ``A {M}odel of {I}nductive {B}ias {L}earning,'' \emph{Journal of
  Artificial Intelligence Research}, vol.~12, pp. 149--198, March 2000.

\bibitem{collins2020taskrobust}
L.~Collins, A.~Mokhtari, and S.~Shakkottai, ``{T}ask-{R}obust
  {M}odel-{A}gnostic {M}eta-{L}earning,'' \emph{arXiv preprint 2002.04766},
  2020.

\bibitem{ben2007analysis}
S.~Ben-David, J.~Blitzer, K.~Crammer, and F.~Pereira, ``Analysis of
  {R}epresentations for {D}omain {A}daptation,'' in \emph{Advances in {N}eural
  {I}nformation {P}rocessing {S}ystems}, 2007, pp. 137--144.

\bibitem{bickel2007discriminative}
S.~Bickel, M.~Br{\"u}ckner, and T.~Scheffer, ``Discriminative {L}earning for
  {D}iffering {T}raining and {T}est {D}istributions,'' in \emph{Proceedings of
  the ICML}, 2007, pp. 81--88.

\bibitem{blitzer2006domain}
J.~Blitzer, R.~McDonald, and F.~Pereira, ``Domain {A}daptation with
  {S}tructural {C}orrespondence {L}earning,'' in \emph{Proceedings of the 2006
  conference on empirical methods in natural language processing}, 2006, pp.
  120--128.

\bibitem{hellstrom2020generalization}
F.~Hellstr{\"o}m and G.~Durisi, ``Generalization {B}ounds via {I}nformation
  {D}ensity and {C}onditional {I}nformation {D}ensity,'' \emph{arXiv preprint
  arXiv:2005.08044}, 2020.

\bibitem{russo2016controlling}
D.~Russo and J.~Zou, ``Controlling {B}ias in {A}daptive {D}ata {A}nalysis
  {U}sing {I}nformation {T}heory,'' in \emph{Proc. of Artificial Intelligence
  and Statistics (AISTATS)}, May 2016, pp. 1232--1240.

\bibitem{xu2017information}
A.~Xu and M.~Raginsky, ``Information-{T}heoretic {A}nalysis of {G}eneralization
  {C}apability of {L}earning {A}lgorithms,'' in \emph{Proc. of Adv. in Neural
  Inf. Processing Sys. (NIPS)}, Dec. 2017, pp. 2524--2533.

\bibitem{bu2019tightening}
Y.~Bu, S.~Zou, and V.~V. Veeravalli, ``Tightening {M}utual {I}nformation
  {B}ased {B}ounds on {G}eneralization {E}rror,'' in \emph{Proc. of IEEE Int.
  Symp. Inf. Theory (ISIT)}, July 2019, pp. 587--591.

\bibitem{negrea2019information}
J.~Negrea, M.~Haghifam, G.~K. Dziugaite, A.~Khisti, and D.~M. Roy,
  ``Information-{T}heoretic {G}eneralization {B}ounds for {SGLD} via
  {D}ata-{D}ependent {E}stimates,'' in \emph{Proc. of Adv. Neural Inf.
  Processing Sys. (NIPS)}, Dec 2019, pp. 11\,013--11\,023.

\bibitem{steinke2020reasoning}
T.~Steinke and L.~Zakynthinou, ``{R}easoning {A}bout {G}eneralization via
  {C}onditional {M}utual {I}nformation,'' vol. 125, pp. 3437--3452, 09--12 Jul
  2020.

\bibitem{wu2020information}
X.~Wu, J.~H. Manton, U.~Aickelin, and J.~Zhu, ``Information-{T}heoretic
  {A}nalysis for {T}ransfer {L}earning,'' \emph{arXiv preprint
  arXiv:2005.08697}, 2020.

\bibitem{jose2020information}
S.~T. Jose and O.~Simeone, ``Information-{T}heoretic {G}eneralization {B}ounds
  for {M}eta-{L}earning and {A}pplications,'' \emph{arXiv preprint
  arXiv:2005.04372}, 2020.

\bibitem{vapnik2015uniform}
V.~N. Vapnik and A.~Y. Chervonenkis, ``On the {U}niform {C}onvergence of
  {R}elative {F}requencies of {E}vents to {T}heir {P}robabilities,'' in
  \emph{Theory of Probability and its Applications}.\hskip 1em plus 0.5em minus
  0.4em\relax SIAM, May 1971, vol.~16, no.~2, pp. 264--280.

\bibitem{koltchinskii2000rademacher}
V.~Koltchinskii and D.~Panchenko, ``Rademacher {P}rocesses and {B}ounding the
  {R}isk of {F}unction {L}earning,'' in \emph{High {D}imensional {P}robability
  II}.\hskip 1em plus 0.5em minus 0.4em\relax Springer, 2000, vol.~47, pp.
  443--457.

\bibitem{mcallester1999pac}
D.~A. McAllester, ``{PAC}-{B}ayesian {M}odel {A}veraging,'' in \emph{Proc. of
  Annual Conf. Computational Learning Theory (COLT)}, July 1999, pp. 164--170.

\bibitem{dziugaite2020role}
G.~K. Dziugaite, K.~Hsu, W.~Gharbieh, and D.~M. Roy, ``On the {R}ole of {D}ata
  in {PAC-B}ayes {B}ounds,'' 2020.

\bibitem{seeger2002pac}
M.~Seeger, ``{PAC}-{B}ayesian {G}eneralization {E}rror {B}ounds for {G}aussian
  {P}rocess {C}lassification,'' \emph{Journal of Machine Learning Research},
  vol.~3, pp. 233--269, Oct 2002.

\bibitem{mcallester2003pac}
D.~A. McAllester, ``{PAC-B}ayesian {S}tochastic {M}odel {S}election,''
  \emph{Machine Learning}, vol.~51, no.~1, pp. 5--21, 2003.

\bibitem{maurer2004note}
A.~Maurer, ``A {N}ote on the {PAC}-{B}ayesian {T}heorem,'' \emph{arXiv preprint
  cs/0411099}, 2004.

\bibitem{alquier2016properties}
P.~Alquier, J.~Ridgway, and N.~Chopin, ``On the {P}roperties of {V}ariational
  {A}pproximations of {G}ibbs {P}osteriors,'' \emph{The Journal of Machine
  Learning Research}, vol.~17, no.~1, pp. 8374--8414, Dec 2016.

\bibitem{pentina2014pac}
A.~Pentina and C.~Lampert, ``A {PAC}-{B}ayesian {B}ound for {L}ifelong
  {L}earning,'' in \emph{Proc. of Int. Conf. on Machine Learning (ICML)}, June
  2014, pp. 991--999.

\bibitem{amit2018meta}
R.~Amit and R.~Meir, ``Meta-{L}earning by {A}djusting {P}riors {B}ased on
  {E}xtended {PAC}-{B}ayes {T}heory,'' in \emph{Proc. of Int. Conf. Machine
  Learning (ICML)}, Jul 2018, pp. 205--214.

\bibitem{rothfuss2020pacoh}
J.~Rothfuss, V.~Fortuin, and A.~Krause, ``{PACOH}: {B}ayes-{O}ptimal
  {M}eta-{L}earning with {PAC}-{G}uarantees,'' \emph{arXiv preprint
  arXiv:2002.05551}, 2020.

\bibitem{germain2017pac}
P.~Germain, A.~Habrard, F.~Laviolette, and E.~Morvant, ``{PAC-B}ayes and
  {D}omain {A}daptation,'' \emph{arXiv preprint arXiv:1707.05712}, 2017.

\bibitem{zhang2006information}
T.~Zhang, ``Information-{T}heoretic {U}pper and {L}ower {B}ounds for
  {S}tatistical {E}stimation,'' \emph{IEEE Transactions on Information Theory},
  vol.~52, no.~4, pp. 1307--1321, 2006.

\bibitem{raginsky2016information}
M.~Raginsky, A.~Rakhlin, M.~Tsao, Y.~Wu, and A.~Xu, ``Information-{T}heoretic
  {A}nalysis of {S}tability and {B}ias of {L}earning {A}lgorithms,'' in
  \emph{Proc. of IEEE Inf. Theory Workshop (ITW)}, Sep. 2016, pp. 26--30.

\bibitem{bassily2016algorithmic}
R.~Bassily, K.~Nissim, A.~Smith, T.~Steinke, U.~Stemmer, and J.~Ullman,
  ``Algorithmic {S}tability for {A}daptive {D}ata {A}nalysis,'' in \emph{Proc.
  of ACM Symp. Theory of Computing (STOC)}, June 2016, pp. 1046--1059.

\bibitem{esposito2019generalization}
A.~R. Esposito, M.~Gastpar, and I.~Issa, ``Generalization {E}rror {B}ounds via
  {R}{\'e}nyi-, $f$-{D}ivergences and {M}aximal leakage,'' \emph{arXiv preprint
  arXiv:1912.01439}, 2019.

\bibitem{blitzer2008learning}
J.~Blitzer, K.~Crammer, A.~Kulesza, F.~Pereira, and J.~Wortman, ``Learning
  {B}ounds for {D}omain {A}daptation,'' in \emph{Advances in {N}eural
  {I}nformation {P}rocessing {S}ystems}, 2008, pp. 129--136.

\bibitem{mansour2009domain}
Y.~Mansour, M.~Mohri, and A.~Rostamizadeh, ``{D}omain {A}daptation: {L}earning
  {B}ounds and {A}lgorithms,'' \emph{arXiv preprint arXiv:0902.3430}, 2009.

\bibitem{zhang2012generalization}
C.~Zhang, L.~Zhang, and J.~Ye, ``Generalization {B}ounds for {D}omain
  {A}daptation,'' in \emph{Advances in {N}eural {I}nformation {P}rocessing
  {S}ystems}, 2012, pp. 3320--3328.

\bibitem{ben2010theory}
S.~Ben-David, J.~Blitzer, K.~Crammer, A.~Kulesza, F.~Pereira, and J.~W.
  Vaughan, ``A {T}heory of {L}earning from {D}ifferent {D}omains,''
  \emph{Machine {L}earning}, vol.~79, no. 1-2, pp. 151--175, 2010.

\bibitem{mansour2012multiple}
Y.~Mansour, M.~Mohri, and A.~Rostamizadeh, ``Multiple {S}ource {A}daptation and
  the {R}{\'e}nyi {D}ivergence,'' \emph{arXiv preprint arXiv:1205.2628}, 2012.

\bibitem{germain2013pac}
P.~Germain, A.~Habrard, F.~Laviolette, and E.~Morvant, ``A {PAC-B}ayesian
  {A}pproach for {D}omain {A}daptation with {S}pecialization to {L}inear
  {C}lassifiers,'' in \emph{International {C}onference on {M}achine
  {L}earning}, 2013, pp. 738--746.

\bibitem{hoffman2018algorithms}
J.~Hoffman, M.~Mohri, and N.~Zhang, ``{A}lgorithms and {T}heory for
  {M}ultiple-{S}ource {A}daptation,'' in \emph{{A}dvances in {N}eural
  {I}nformation {P}rocessing {S}ystems}, 2018, pp. 8246--8256.

\bibitem{redko2017theoretical}
I.~Redko, A.~Habrard, and M.~Sebban, ``{T}heoretical {A}nalysis of {D}omain
  {A}daptation with {O}ptimal {T}ransport,'' in \emph{{J}oint {E}uropean
  {C}onference on {M}achine {L}earning and {K}nowledge {D}iscovery in
  {D}atabases}.\hskip 1em plus 0.5em minus 0.4em\relax Springer, 2017, pp.
  737--753.

\bibitem{finn2017model}
C.~Finn, P.~Abbeel, and S.~Levine, ``Model-{A}gnostic {M}eta-{L}earning for
  {F}ast {A}daptation of {D}eep {N}etworks,'' in \emph{Proc. of Int. Conf.
  Machine Learning-Volume 70}, Aug. 2017, pp. 1126--1135.

\bibitem{guedj2019primer}
B.~Guedj, ``A {P}rimer on {PAC-B}ayesian {L}earning,'' \emph{arXiv preprint
  arXiv:1901.05353}, 2019.

\bibitem{bishop2006pattern}
C.~M. Bishop, \emph{Pattern {R}ecognition and {M}achine {L}earning}.\hskip 1em
  plus 0.5em minus 0.4em\relax Springer, 2006.

\bibitem{denevi2020advantage}
G.~Denevi, M.~Pontil, and C.~Ciliberto, ``The {A}dvantage of {C}onditional
  {M}eta-{L}earning for {B}iased {R}egularization and {F}ine-{T}uning,''
  \emph{arXiv preprint arXiv:2008.10857}, 2020.

\bibitem{raginsky2017non}
M.~Raginsky, A.~Rakhlin, and M.~Telgarsky, ``Non-{C}onvex {L}earning via
  {S}tochastic {G}radient {L}angevin {D}ynamics: {A} {N}onasymptotic
  {A}nalysis,'' \emph{arXiv preprint arXiv:1702.03849}, 2017.

\bibitem{kuzborskij2019distribution}
I.~Kuzborskij, N.~Cesa-Bianchi, and C.~Szepesv{\'a}ri,
  ``Distribution-{D}ependent {A}nalysis of {Gibbs-ERM P}rinciple,'' \emph{arXiv
  preprint arXiv:1902.01846}, 2019.

\bibitem{sugiyama2012density}
M.~Sugiyama, T.~Suzuki, and T.~Kanamori, \emph{Density {R}atio {E}stimation in
  {M}achine {L}earning}.\hskip 1em plus 0.5em minus 0.4em\relax Cambridge
  {U}niversity {P}ress, 2012.

\bibitem{jose2020informationtheoretic}
S.~T. Jose and O.~Simeone, ``{I}nformation-{T}heoretic {B}ounds on {T}ransfer
  {G}eneralization {G}ap {B}ased on {J}ensen-{S}hannon {D}ivergence,''
  \emph{arXiv preprint 2010.09484}, 2020.

\bibitem{polyanskiy2014lecture}
Y.~Polyanskiy and Y.~Wu, ``Lecture {N}otes on {I}nformation {T}heory,''
  \emph{Lecture Notes for ECE563 (UIUC) and}, vol.~6, no. 2012-2016, p.~7,
  2014.

\end{thebibliography}
\end{document}